\documentclass{article}

\usepackage[preprint]{neurips_2025}

\usepackage[utf8]{inputenc}
\usepackage[T1]{fontenc}
\usepackage{microtype}
\usepackage[table]{xcolor}
\usepackage{graphicx}
\usepackage{booktabs}
\usepackage{arydshln}
\usepackage{tabularx} 
\newcommand{\methodname}{ROVA}

\definecolor{basegray}{RGB}{240,240,240}  
\definecolor{medorange}{RGB}{255,165,0}
\definecolor{badred}{RGB}{255,0,0}
\definecolor{goodgreen}{RGB}{0,128,0}
\usepackage{amsmath,amssymb,amsthm}
\usepackage{mathtools}
\usepackage{mathrsfs}
\usepackage{bm}
\usepackage{bbm}
\usepackage{nicefrac}

\newtheorem{proposition}{Proposition}
\usepackage[table]{xcolor}

\definecolor{navy}{RGB}{0,0,65}
\definecolor{headerblue}{RGB}{220,230,245}
\definecolor{frameblue}{RGB}{70,130,180}
\definecolor{lightblue}{RGB}{219,234,254}
\definecolor{rowgray}{RGB}{245,245,245}
\definecolor{bestcell}{RGB}{230,245,230}
\definecolor{highlightgray}{RGB}{230,230,230}
\definecolor{baselinegray}{RGB}{215,215,215}
\definecolor{badred}{RGB}{200,50,50}
\definecolor{negred}{RGB}{180,0,0}
\definecolor{medorange}{RGB}{210,140,30}
\definecolor{goodgreen}{RGB}{30,150,30}
\definecolor{improveblue}{RGB}{40,100,200}
\definecolor{classifygold}{HTML}{FFF3CD}
\definecolor{mycitecolor}{HTML}{6666DD}

\usepackage[pagebackref,breaklinks,colorlinks]{hyperref}
\hypersetup{
  colorlinks = true,
  citecolor  = mycitecolor,
  linkcolor  = red,
  urlcolor   = magenta,
}

\usepackage[capitalize,noabbrev]{cleveref}

\usepackage{booktabs}
\usepackage{tabularx}
\usepackage{multirow}
\usepackage{array}
\usepackage{arydshln}
\usepackage{colortbl}

\usepackage{graphicx}
\usepackage{subcaption}
\usepackage{wrapfig}
\definecolor{basegray}{RGB}{250,250,250}
\definecolor{oursgray}{HTML}{EAF2FB}
\usepackage{enumitem}
\usepackage{url}
\usepackage{soul}
\usepackage{tcolorbox}
\usepackage{titletoc}
\usepackage{tocloft}

\usepackage{algorithm}
\usepackage{algorithmic}

\usepackage{listings}
\definecolor{darkgreen}{rgb}{0,0.5,0}
\definecolor{azureblue}{rgb}{0,0.5,1}
\definecolor{darkgreen}{rgb}{1,0,0}
\definecolor{color1}{HTML}{006EB8}
\definecolor{color2}{HTML}{009B55}

\usepackage{pifont}

\Crefname{section}{Section}{Sections}
\Crefname{table}{Table}{Tables}
\crefname{section}{Sec.}{Secs.}
\crefname{table}{Tab.}{Tabs.}
\crefname{figure}{Fig.}{Figs.}
\crefname{appendix}{Sec.}{Secs.}

\newcommand{\methodnamefull}{RObust Video Alignment (ROVA)}
\newcommand{\benchname}{PVRBench}
\newcommand{\benchnamefull}{Perturbed Video Reasoning Benchmark}

\title{Are Video Reasoning Models Ready to Go Outside?}
\author{%
  Yangfan He \\
  NTU Singapore\\
  \texttt{yhe873232@gmail.com} \\
  \And
  Changgyu Boo \\
  Korea University \\
  \texttt{2019150348@korea.ac.kr} \\
  \AND
  Jaehong Yoon\thanks{Corresponding author} \\
  NTU Singapore \\
  \texttt{jaehong.yoon@ntu.edu.sg} \\
}
\begin{document}

\maketitle
\begin{abstract}
In real-world deployment, vision-language models often encounter disturbances such as weather, occlusion, and camera motion. Under such conditions, their understanding and reasoning degrade substantially, revealing a gap between clean, controlled (i.e., unperturbed) evaluation settings and real-world robustness.
To address this limitation, we propose \methodname{}, a novel training framework that improves robustness by modeling a robustness-aware consistency reward under spatio-temporal corruptions. 
\methodname{} introduces a difficulty-aware online training strategy that prioritizes informative samples based on the model’s evolving capability.
Specifically, it continuously re-estimates sample difficulty via self-reflective evaluation, enabling adaptive training with a robustness-aware consistency reward.
We also introduce \benchname{}, a new benchmark that injects real-world perturbations into embodied video datasets to assess both accuracy and reasoning quality under realistic disturbances.
We evaluate \methodname{} and baselines on \benchname{}, UrbanVideo, and VisBench, where open-source and proprietary models suffer up to 35\% and 28\% drops in accuracy and reasoning under realistic perturbations. \methodname{} effectively mitigates performance degradation, boosting relative accuracy by at least 24\% and reasoning by over 9\% compared with baseline models (QWen2.5/3-VL, InternVL2.5, Embodied-R). These gains transfer to clean standard benchmarks, yielding consistent improvements.

\begin{center}
\vspace{0.1in}
Project Page: \url{https://robust-video-reason.github.io/}
\end{center}

\end{abstract}


\begin{figure*}[t!]
    \centering
    \includegraphics[width=\textwidth]{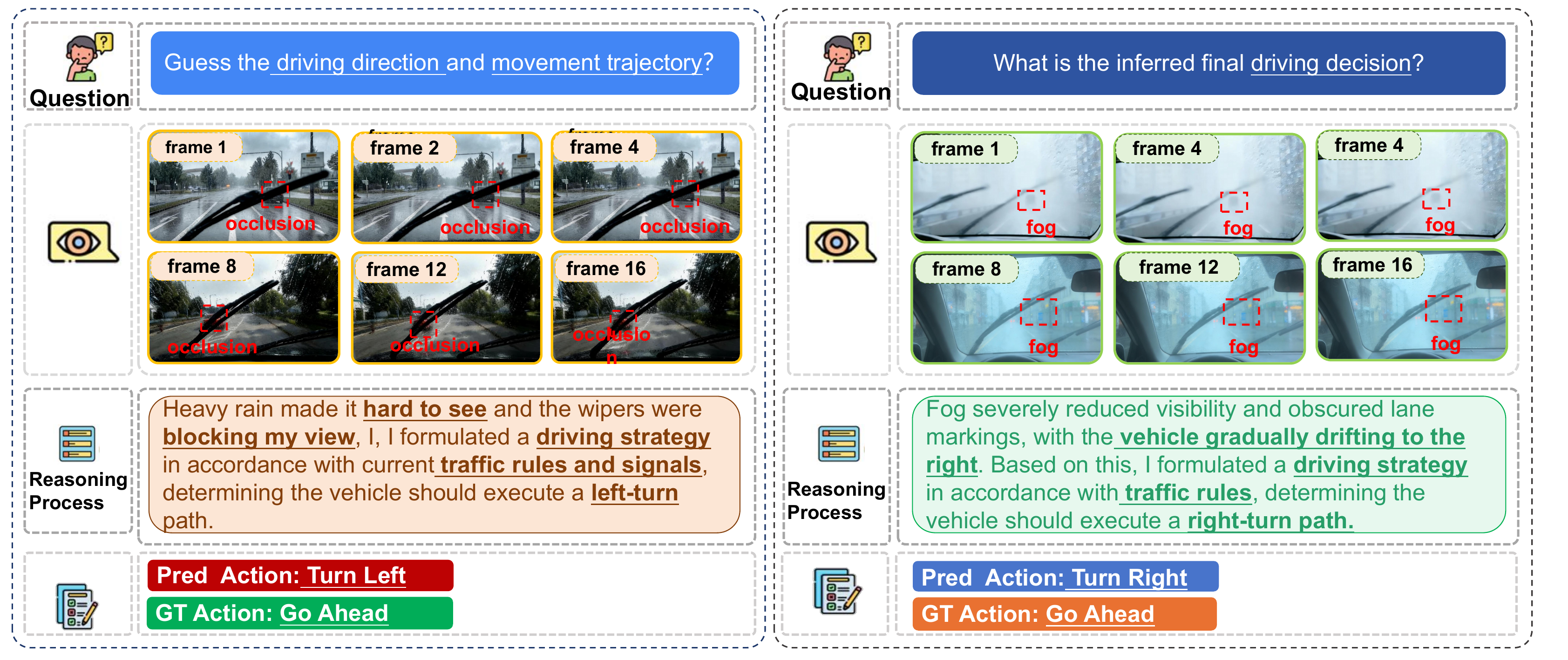}
    \caption{Failure cases of Qwen2.5-VL under two representative perturbations: (a) occlusion (left) and (b) adverse weather (right). The model incorrectly predicts Turn Left'' under occlusion and Turn Right'' under fog, despite the ground-truth being ``Go Ahead'' in both cases, demonstrating how realistic perturbations mislead reasoning and motivating the need for robustness-aware training.}
    \label{fig:motivation}
    \vspace{-0.65cm}
\end{figure*}
\section{Introduction}
Vision-language models (VLMs)~\citep{zhang2023video,maaz2024video,Shu2024Video,yuan2025tarsier2advancinglargevisionlanguage,li2025videochatr1enhancingspatiotemporalperception,yu2025crema,clark2026molmo2} have rapidly advanced video understanding and reasoning, allowing systems to interpret complex scenes and perform temporally grounded inference. These capabilities support many real-world applications, yet a key question remains: \textit{are current VLMs robust enough to operate reliably beyond clean, controlled conditions?} In practice, these models frequently face challenging video streams, corrupted by adverse weather (e.g., rain, fog, snow), dynamic occlusions (e.g., pedestrians, vehicles, vegetation), abrupt illumination changes (e.g., glare, shadows, low light), and camera motion induced by vibration or viewpoint shifts. Such perturbations are common in the real world, yet these models severely degrade perception and lead to brittle or unreliable reasoning~(\cref{fig:motivation}). For instance, under conditions such as video occlusion or adverse weather, baseline models may incorrectly output “Turn Left” or “Turn Right” rather than the ground-truth “Going Ahead.”
This gap between benchmark assumptions and real-world conditions highlights the need for training frameworks that promote reliable generalization under realistic variability and uncertainty.
A few prior studies~\citep{mao2022understanding,zhou2024revisitingadversarialrobustnessvision,zhang2024benchmarkinglargemultimodalmodels} have explored improving the robustness of VLMs through generic data augmentation, random frame masking, zero-shot, or adversarial training. However, these methods typically treat robustness as a single objective, overlooking that different perturbations induce distinct failure modes. Consequently, they struggle to address structured, semantically meaningful corruptions common in real-world environments, since perturbation-specific failure behaviors are not explicitly modeled.


To address this challenge, we propose \textit{\methodnamefull{}}, a novel training approach for robust vision reasoning under realistic visual disturbances. We first apply corruption-based augmentation to generate perturbed videos. \methodname{} then measures divergence in reasoning coherence and answer quality between clean and corrupted videos as a proxy for corruption-induced difficulty. Moderately difficult instances are used for training, while overly easy samples are discarded and excessively difficult ones are stored in a temporal memory buffer for later revisiting.
Unlike curriculum learning, which follows a fixed, easy-to-hard schedule, this \textit{self-reflective evaluation} estimates the difficulty and informativeness of each video–query instance based on the model’s current capability, enabling an adaptive curriculum that prioritizes informative samples while deferring overly difficult ones through memory replay. 
Next, we introduce a dual-branch alignment objective that enforces output consistency between paired clean and perturbed inputs. This robustness-aware consistency alignment is guided by reward modeling over reasoning and answer consistency, and optimized using group relative policy optimization.
Specifically, we enforce output consistency between paired clean and perturbed video inputs through reward-guided optimization that evaluates both reasoning and answer consistency, trained via group relative policy optimization~\citep{shao2024deepseekmath}.

\begin{table*}[t!]
\centering
\large  
\definecolor{cmark}{HTML}{2E8B57}
\definecolor{xmark}{HTML}{CD5C5C}
\definecolor{oursrow}{HTML}{EAF2FB}
\newcommand{\yes}{\textcolor{cmark}{\ding{51}}}
\newcommand{\no}{\textcolor{xmark}{\ding{55}}}
\caption{Comparison of \benchname{} with existing video understanding benchmarks.
\protect\textit{\#Types} counts perturbation subtypes. 
\protect\textit{\#Cat.} counts scene or class categories.
\textit{{Syn}thetic}, \textit{{Spat}ial}, and \textit{{Temp}oral} indicate artificially generated, spatially grounded, and temporally consistent perturbations, respectively.
\benchname{} covers 27 tasks covering \textit{\textbf{in}door}, \textit{\textbf{out}door}, and \textit{\textbf{emb}odied AI} scenarios.
$^\ddagger$: An image-level benchmark for reference.
}
\label{tab:benchmark_comparison}
\resizebox{\textwidth}{!}{%
\setlength{\tabcolsep}{4pt}        
\renewcommand{\arraystretch}{1.2}     
\begin{tabular}{l cc ccccc cccc}
\toprule
\multirow{2}{*}{\textbf{Benchmark}} & \multicolumn{2}{c}{\textbf{Scale}} & \multicolumn{5}{c}{\textbf{Perturbation Properties}} & \multicolumn{4}{c}{\textbf{Scene Coverage}} \\
\cmidrule(lr){2-3} \cmidrule(lr){4-8} \cmidrule(lr){9-12}
& \textbf{\#Videos} & \textbf{\#QAs} & \textbf{Synthetic} & \textbf{Real} & \textbf{Spatial} & \textbf{Temporal} & \textbf{\#Types} & \textbf{Ind.} & \textbf{Out.} & \textbf{Emb.} & \textbf{\#Cat.} \\
\midrule
ImageNet-C$^\ddagger$~\citep{xie2020self}   & 50K  & 50K  & \yes & \no & \no & \no & 19        & \yes & \yes & \no  & 1K \\
MVBench~\citep{li2024mvbench}               & 4K   & 4K   & \no  & \no & \no & \no & 0         & \yes & \yes & \no  & 20 \\
Video-MME~\citep{fu2025video}               & 900  & 2.7K & \no  & \no & \no & \no & 0         & \yes & \yes & \no  & 30 \\
ALFRED~\citep{shridhar2020alfred}           & 8K   & 25K  & \no  & \no & \no & \no & 0         & \yes & \no  & \yes & 7  \\
Ego4D~\citep{grauman2022ego4d}              & 3.7K & 3.8M & \no  & \no & \no & \no & 0         & \yes & \yes & \yes & 5  \\
VisBench~\citep{yang2025thinking}           & 500  & 3K   & \no  & \no & \no & \no & 0         & \yes & \no  & \yes & 11 \\
UrbanVideo~\citep{zhao2025urbanvideo}       & 1.5K & 6K   & \no  & \no & \no & \no & 0         & \no  & \yes & \yes & 16 \\
\midrule
\rowcolor{oursrow}
\textbf{\benchname{} (Ours)} & \textbf{9K} & \textbf{52K} & \yes & \no & \yes & \yes & \textbf{12} & \yes & \yes & \yes & \textbf{27} \\
\bottomrule
\end{tabular}%
}
\label{tab : benchmark}
\vspace{-0.2in}
\end{table*}



We further introduce \textit{\benchnamefull{} (\benchname{})}, for evaluating the robustness of video reasoning under diverse realistic perturbations.
Unlike existing benchmarks, including VisBench~\citep{yang2025thinking} and UrbanVideo~\citep{zhao2025urbanvideo}, which primarily evaluate models on curated environments, \benchname{} systematically injects perturbations from 12 corruption styles associated with lighting, camera motion, occlusion, and weather (\cref{tab:benchmark_comparison}), across 27 scene categories. Notably, all perturbations are spatially aware and temporally coherent, capturing realistic video disturbances. We observe that performant proprietary models (GPT-4o~\citep{hurst2024gpt} / Gemini-3-Pro~\citep{team2023gemini}) suffer 11–17\% and 10–14\% drops in accuracy and reasoning, and open-source models degrade by up to 35\% and 26\%, respectively, highlighting robustness gaps in VLMs under realistic conditions.

\methodname{} consistently outperforms proprietary and open-source models on \benchname{}, UrbanVideo, and VisBench across all perturbation types in both answer accuracy and reasoning quality. Specifically, \methodname{} surpasses the strongest open-source baselines of comparable size, Embodied-R, by 17\%, while larger variants (13B/72B) match or exceed leading proprietary models such as Gemini-3-Pro and GPT-4o. Notably, these improvements extend to clean videos, demonstrating enhanced generalizability and stronger performance on clean data. Furthermore, \methodname{} achieves higher reasoning quality, with improved consistency and belief scores, reflecting more stable, confident reasoning under visual corruption.
\section{Related Work}
\textbf{Robust Training for Multimodal Models.}
Several works~\citep {mao2022understanding,zhao2023evaluatingadversarialrobustnesslarge,sheng2025rtptimprovingadversarialrobustness,oh2025understanding,agarwal-etal-2025-mvtamperbench,schiappa2023robustnessanalysisvideolanguagemodels} have explored robustness to distribution shifts and adversarial inputs through data augmentation~\citep{duan2023improvevideorepresentationtemporal}, test-time adaptation~\citep{zhao2024testtimeadaptationclipreward}, and transfer-based strategies~\citep{Tong2025zero,cai2025clapisolatingcontentstyle}. However, these approaches primarily address generic perturbations or optimization efficiency, rather than the structured, semantically grounded disturbances encountered in real-world video settings.
In video reasoning, recent methods~\citep{zhou2025reagent,wang2025timer1posttraininglargevision,chen2025datasetsrecipesvideotemporal,wang2025video} improve efficiency via adaptive frame sampling or data filtering, but they do not explicitly model realistic corruption patterns~\citep{Zeng2024benchmarking,Yang2025robench} that alter scene visibility and temporal coherence. As a result, robustness is treated as incidental resilience rather than being explicitly modeled during optimization.
In contrast, \methodname{} incorporates structured and semantically grounded perturbations that reflect realistic environmental disturbances. 
The proposed architecture and training objectives enforce representation consistency between clean and 
perturbed videos, progressively strengthening disturbance-aware reasoning.

\textbf{Robust Video Reasoning in Real-World Environments.}
Recent advances in video–language models~\citep{zhang2023video,nguyen-etal-2024-video,yuan2025tarsier2advancinglargevisionlanguage,yu2025crema,clark2026molmo2} have substantially improved temporal reasoning and long-horizon embodied planning~\citep{chen2025exploringembodiedmultimodallarge,azzolini2025cosmos,zhang2025embodied,zhao2025embodied,yu2026and,yeo2026worldmm}.
However, most existing benchmarks evaluate models under nearly clean visual conditions~\citep{maaz2024video}, implicitly assuming stable lighting, unobstructed views, and smooth camera movement. Although robustness is sometimes measured via synthetic textual perturbations~\citep{wu2025pay}, such evaluations do not capture structured, semantically grounded visual disturbances encountered in real-world environments. Consequently, no standardized benchmark systematically integrates realistic disturbances into embodied video reasoning, leaving a gap between benchmarks and deployment conditions.
In contrast, we introduce \benchname{} that integrates semantically meaningful perturbations into temporally coherent reasoning tasks. Rather than treating corruption as incidental noise, we ask models to reliably reason about scene content, even in the presence of disturbances.
\begin{figure*}[t!]
\centering
\includegraphics[width=\linewidth]{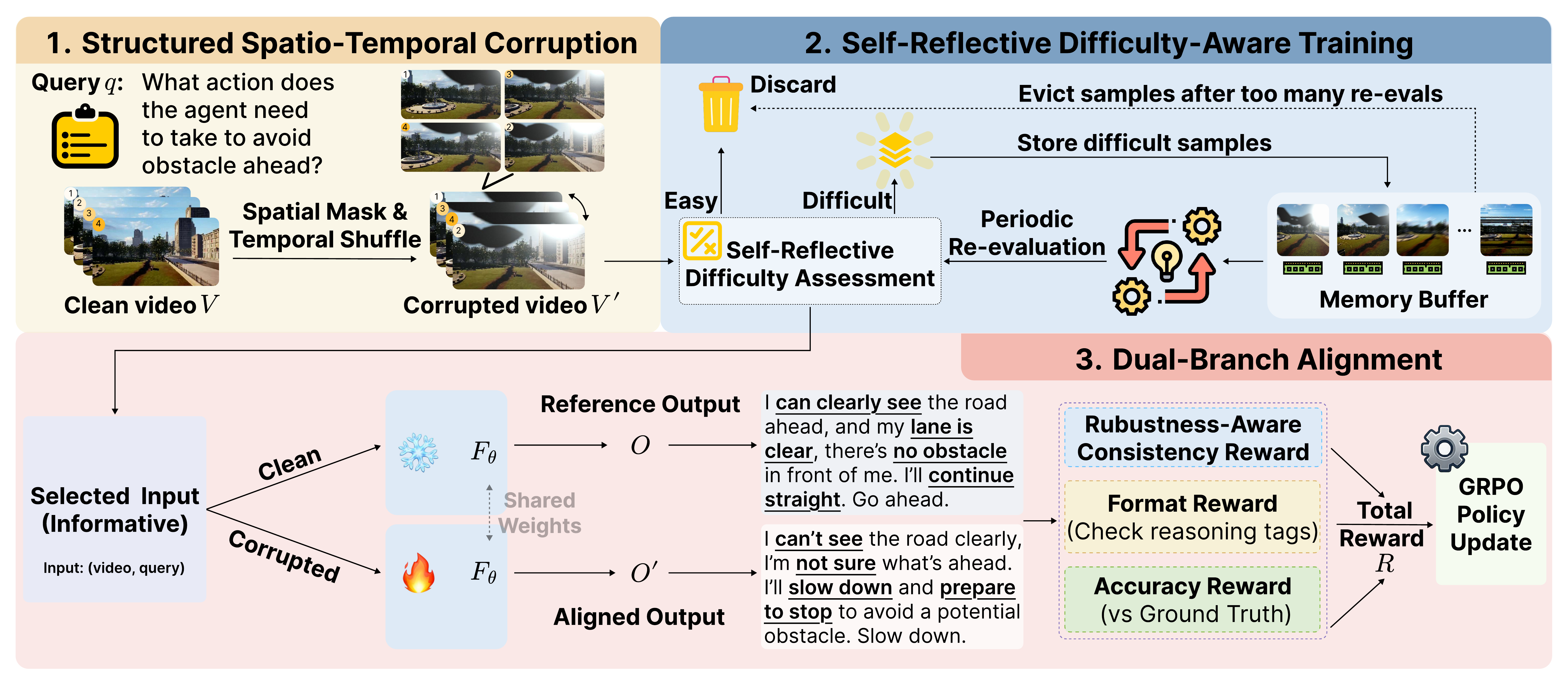}
\caption{\textbf{Overview of \methodname{}}: (1) structured spatio-temporal corruption that generates realistic perturbations, (2) self-reflective evaluation with difficulty-aware online training that adaptively prioritizes informative samples, and (3) dual-branch alignment reward modeling that enforces output consistency between clean and perturbed inputs.}
\vspace{-0.15in}
\label{fig:concept}
\end{figure*}

\section{Training Robust Video Reasoning Models with \methodname{}}
\label{sec : method_VRM}
As illustrated in \cref{fig:concept}, \methodname{}, a novel training approach for robust video reasoning under real-world perturbations, comprises three stages: we first generate corruption-augmented video-query pairs via dynamic, physically plausible perturbations (\cref{sec : corruption}). Next, a difficulty-aware curriculum performs self-reflective evaluation to selectively curate informative training samples conditioned on the model’s evolving capability (\cref{sec : Data Curation}) . 
Finally, dual-branch alignment enforces consistency between clean and perturbed videos via reasoning-aware rewards and group relative policy optimization (GRPO) (\cref{sec : dual-branch}).  

\subsection{Learning with Structured Spatio-Temporal Corruption} \label{sec : corruption}
We first design a structured spatio-temporal corruption pipeline that models four realistic disturbances, including \textit{weather, lighting, occlusion, and camera motion}, using style-specific, cross-frame coherent masks for spatial perturbations and temporal shuffling to disrupt temporal order.
Unlike generic augmentations that apply independent pixel or frame perturbations (e.g., random masking, color jittering)~\citep{xie2020self}, we explicitly model perturbation styles with spatial grounding and temporal coherence, yielding structured spatio-temporal disturbances.
Each video is then paired with its corrupted counterpart in a dual-branch alignment framework to optimize output consistency.
Through this design, the model learns perturbation-invariant representations for robust real-world generalization.

Let a video sequence be denoted as $V = \{f_1, f_2, \dots, f_T\}$, where $f_t \in \mathbb{R}^{H \times W \times C}$ denotes the $t$-th frame of height $H$, width $W$, and $C$ channels.

\textbf{Temporal Corruption.} To disrupt temporal coherence, we randomly permute the frame sequence. A permutation $\pi: \{1, \dots, T\} \to \{1, \dots, T\}$ is sampled uniformly at random, and the temporally shuffled video is defined as
\begin{equation}
V_{\mathrm{temp}} = \{f_{\pi(1)}, f_{\pi(2)}, \dots, f_{\pi(T)}\},
\label{eq:temp_shuffle}
\end{equation}
which completely scrambles temporal order while preserving all frame content.

\noindent\textbf{Spatial Corruption.} Rather than coarse block-wise masking that risks removing critical cues, we apply fine-grained masks across four perturbation styles $m \in \mathcal{P} = \{\mathrm{weather},\, \mathrm{lighting},\, \mathrm{camera},\, \mathrm{occlusion}\}$. 
For each frame $f_t$, the mask $P_t^{(m)} = B_t^{(m)} \odot C_t^{(m)}$ fuses a binary map $B_t^{(m)} \in \{0,1\}^{H \times W}$, where $1$/$0$ denotes corrupted/clean pixels, with layouts driven by depth awareness or stochastic sampling, and a continuous modulation map $C_t^{(m)} \in [0,1]^{H \times W}$ encoding per-pixel effect intensity (e.g., rain strength, shadow depth, blur kernel; see~\cref{appendix:perturbation_system}.) The corrupted frame is computed as
$f_t^{\mathrm{masked}} = f_t \odot P_t^{(m)}$,
where $\odot$ denotes element-wise multiplication.

\noindent\textbf{Spatio-Temporal Corruption.} For each video, a perturbation style $m \in \mathcal{P}$ is uniformly sampled to generate the corrupted frame sequence:
\begin{equation}
V' = \left\{ f_{\pi(t)} \odot P^{(m)}_t \right\}_{t=1}^{T},
\label{eq:combined_corruption}
\end{equation}
where $P^{(m)}_t$ denotes the smooth, style-specific mask associated with style $m$. By jointly introducing temporal order disruption and spatially realistic, continuous masking, our approach promotes perturbation-invariant representation learning while preserving essential visual semantics.

\subsection{Self-Reflective Difficulty-Aware Training}\label{sec : Data Curation}
{
Introducing structured visual corruptions exposes the model to a broader spectrum of reasoning difficulty than training on clean videos alone. While clean inputs typically lie within a narrow difficulty range, corrupted versions vary widely in severity, expanding the diversity of learning signals during training.
Crucially, training is most effective on samples that are neither too easy nor excessively difficult~\citep{wang2025video} under the model’s current capacity, as these instances provide the most informative learning signals and support stable optimization. Rather than uniformly sampling across the expanded difficulty range, we therefore prioritize appropriately challenging examples through a self-reflective, difficulty-aware strategy that implicitly forms an online curriculum. By continuously focusing on corrupted samples that provide meaningful learning signals, the model enables to promote robust and reliable reasoning under realistic visual disturbances.}

To this end, we propose a self-reflective, difficulty-aware training pipeline that implicitly builds an adaptive curriculum in an online manner. Formally, let $F_\theta$ denote a learnable VLM parameterized by $\theta$. 
We assume that training video–text pairs arrive sequentially, and let $\theta_i$ denote the model parameters at training iteration $i$.
At each iteration, \methodname{} performs two internal steps:
\textit{1) self-reflective evaluation}, where $F_{\theta_i}$ estimates the usefulness of incoming samples for training under its current state; and \textit{2) difficulty-aware selective training}, where model updates are performed using only a subset of samples selected according to the proposed policy.

\textbf{Self-Reflective Evaluation.} 
At iteration $i$, the model $F$ evaluates each masked video $V'_i$ and produces a difficulty label $d \in \{\emph{easy}, \emph{difficult}, \emph{informative}\}$ and a confidence score $c \in [0,1]$, defined as,
\begin{equation}
d, c = F_{\theta_i}(q_i, V'_i, S_e),
\label{eq:self_reflection}
\end{equation}
where $q_i$ denotes the input query and $S_e$ denotes the evaluation prompt (See \cref{tab:prompt_judge}). Specifically, $d$ is obtained by prompting $F_{\theta_i}$ with $S_e$ to compare its responses on clean and corrupted inputs: if the model answers correctly and consistently, the sample is labeled $d=\emph{easy}$; if responses diverge substantially or are incorrect, it is labeled $d=\emph{difficulty}$; otherwise, the sample is labeled $d=\emph{informative}$, indicating moderate uncertainty that is most beneficial for training. The confidence score $c$ is derived from the model's output token probabilities.
Unlike traditional curriculum learning with a fixed schedule, our prompt-based sample-level evaluation dynamically estimates the model’s current capability and prioritizes informative samples to stabilize the effective training distribution.
Based on $d$ and $c$, we design the following data selection policy: 

(i) high-confidence easy samples ($d=\emph{easy},\, c>\tau$, where $\tau$ is a confidence threshold) are considered as sufficiently learned and filtered out, {enabling the model to prioritize disturbance-sensitive samples that provide strong learning signals}. 
(ii) difficult samples ($d=\emph{difficult}$) are stored in a temporal memory buffer $\mathcal{M}$ {for deferred training and periodically re-evaluated. While potentially informative, they may yield weak or unstable learning signals under the current model state, and are revisited once the model has sufficiently improved.}
(iii) informative samples ($d=\emph{informative}$) as well as low-confidence easy samples ($d=\emph{easy},\, c<=\tau$) are treated as high-information instances and prioritized for immediate training.

{\textbf{Difficulty Re-evaluation and Deferred Training with Memory.}} As the model improves over time, samples that were previously too difficult to learn from may later provide meaningful training signals. To leverage this evolving capability, we introduce a memory-based deferred training mechanism that periodically re-evaluates difficult instances.
Formally, when newly arriving data are evaluated as \textit{difficult}, 
it is stored in a temporal memory buffer $\mathcal{M}$ as:
    \begin{equation}
    \mathcal{M} \leftarrow \mathcal{M} \cup \{(q, \tilde{V}, k = 0)\},
    \label{eq:memory_insert}
    \end{equation}
    where $\tilde{V}$ encodes the mask metadata, including perturbation style, parameters, and spatial-temporal regions.
    This design allows the corrupted video $V'$ to be regenerated on demand during re-evaluation, avoiding the need to store full video data.
    During training, instances in $\mathcal{M}$ are periodically re-evaluated under the updated model. The counter $k$ tracks the number of re-assessments performed for each sample.
    For each entry $(q_n, \tilde{V}_n, k_n) \in \mathcal{M}$, the current model $F$ periodically re-assesses its difficulty using the current parameter $\theta_i$:
    \begin{equation}
    d', c' = F(q_n, \tilde{V}_n, S_e; \theta_i), \quad
    k_n \leftarrow k_n + 1.
    \label{eq:difficulty}
    \end{equation}
    Here, $d'$ and $c'$ denote the updated difficulty level and confidence score, respectively. 
    Entries reclassified as \textit{informative} are immediately used for training, whereas those labeled \textit{easy} are removed from the memory buffer. Entries that remain difficult are retained in $\mathcal{M}$ with their re-evaluation counter incremented.

    The confidence score $c'$ serves as an auxiliary diagnostic signal for self-monitoring and stability analysis, but is not used directly for memory retention decisions to avoid sensitivity to noisy confidence estimates.
    As training progresses, samples that were previously difficult may transition to informative or easy categories, allowing the curriculum to adapt to the model’s evolving capability. However, repeated re-evaluation can lead to unbounded memory growth, {particularly when samples remain persistently difficult or heavily corrupted, yielding little effective learning signal.} To prevent this, we impose a maximum re-evaluation threshold and evict entries exceeding it:
    \begin{equation}
    \mathcal{M} \leftarrow \mathcal{M} \setminus \{(q, \tilde{V}, k) \mid k > K_{\max}\}.
    \label{eq : memory}
    \end{equation}
    Overall, the proposed self-reflective, difficulty-aware training framework establishes a closed-loop mechanism that dynamically adjusts the training data distribution to the model’s evolving capability. 
    By prioritizing samples based on estimated difficulty and confidence, the framework selects instances that yield effective learning signals under corrupted conditions while filtering low-utility ones.
    Although periodic re-evaluation incurs modest computational overhead, this cost is negligible relative to the high per-sample cost of reinforcement learning on videos. 
    In addition, selectively discarding uninformative instances leads to substantial gains in training efficiency (See~\cref{tab:training_efficiency}). 
    
\subsection{Dual-Branch Alignment Optimization} \label{sec : dual-branch}
\methodname{} trains the model through a dual-branch alignment mechanism that aligns representations from clean and partially perturbed video inputs.
The training objective enforces consistency between two branches using the proposed reward modeling combined with GRPO~\citep{shao2024deepseekmath}. Here, the clean video branch serves as a fixed anchor with gradients detached, while the perturbed branch is optimized to align its outputs with those of the clean branch.
Given a group of $G$ paired samples, the clean branch produces reference outputs $\{o_j\}_{j=1}^G$ and the perturbed branch generates aligned outputs $\{\tilde{o}_j\}_{j=1}^G$. Each pair $(o_j, \tilde{o}_j)$ corresponds to the same video query under clean and perturbed visual conditions. 
    We treat a  $F_\theta$ as a policy that generates reasoning outputs conditioned on video inputs:
    \begin{equation}
    \begin{split}
    J(\theta) = \mathbb{E}_{(q, V)\sim \mathcal{D},\; 
    \{o_j\}_{j=1}^G \sim F_{\theta_\text{old}}(O|q, V)}&\frac{1}{G} \sum_{j=1}^G \big[\min\big( r_j A_j, \\
    \text{clip}( r_j, 1-\epsilon, &1+\epsilon ) A_j \big)
    - \beta D_{\text{KL}}\big( F_\theta \| F_{\text{ref}} \big) \big],
    \end{split}
    \label{eq:grpo_objective}
    \end{equation}
    where $r_j = F_\theta(o_j|q) / F_{\theta_\text{old}}(o_j|q)$, $\epsilon$ 
    and $\beta$ are hyperparameters, and $D_{\text{KL}}\big( F_\theta \| F_{\text{ref}} \big)$ denotes the KL-divergence penalty term. The advantage $A_j$ corresponding to output $o_j$ is calculated from the associated reward set $\{r_1, r_2, \dots, r_G\}$:
    \begin{equation}
    A_j = \frac{r_j - \text{mean}\big(\{r_1, r_2, \dots, r_G\}\big)}{\text{std}\big(\{r_1, r_2, \dots, r_G\}\big)}.
    \label{eq:advantage}
    \end{equation}
    \textbf{Format Reward.} The model is required to generate an output $o_j$ consisting of an embodied reasoning process $p_j$ followed by a final answer $a_j$, enclosed within \texttt{<think></think>} and \texttt{<answer></answer>} tags, respectively. Compliance with this format is verified via a regular expression, producing the format reward $r^\text{F}_j$.
    \begin{equation}
    r^\text{F}_j =
    \begin{cases}
    1, & \text{if the format is correct;} \\
    0, & \text{if the format is incorrect.}
    \end{cases}
    \label{eq:format_reward}
    \end{equation}
    \textbf{Accuracy Reward.}
    The accuracy reward $r^\text{Acc}_j$ evaluates whether the extracted answer $o_j$ is semantically consistent with the ground truth $g$. Multiple-choice questions typically have a unique and precise answer that can be directly compared once the response follows the required format.
    \begin{equation}
    r^\text{Acc}_j =
    \begin{cases}
    1, & o_j = g; \\
    0, & o_j \neq g.
    \end{cases}
    \label{eq:accuracy_reward}
    \end{equation}
\textbf{Alignment Reward.} For each output pair $(o_j, \tilde{o}_j)$, the alignment reward is decomposed into reasoning and answer components: 
    $r^A_j = r^\text{align, r}_j + r^\text{align, a}_j,$
    where $r^\text{align, r}_j = \alpha_r \cdot {Sim}^\text{r}(o_j, \tilde{o}_j)$ and $r^\text{align,a}_j = \alpha_a \cdot {Sim}^\text{a}(o_j, \tilde{o}_j)$. 
    Here, $\alpha_r$ and $\alpha_a$ weight the respective contributions, with ${Sim}^\text{r}$ and ${Sim}^\text{a}$ to measure semantic consistency in the reasoning process and answer segment (see \cref{tab:prompt_answer,tab:prompt_reasoning}). The total reward combines alignment with three rewards:
    $R_j = r^\text{F}_j + r^\text{Acc}_j + r^A_j.$
    
With the proposed dual-branch alignment framework, the model is optimized via GRPO using a combined reward signal with robustness-aware consistency, encouraging stable reasoning and answer predictions across clean and perturbed video inputs, thereby improving robustness and generalization.

\section{Evaluating Video Reasoning under Various Realistic Disturbances}
\textbf{Motivation.} 
Existing video reasoning benchmarks, including MVBench~\citep{li2024mvbench}, Video-MME~\citep{fu2025video}, ALFRED~\citep{shridhar2020alfred}, Ego4D~\citep{grauman2022ego4d}, and UrbanVideo~\citep{zhao2025urbanvideo}, evaluate models primarily under clean visual conditions (\cref{tab:benchmark_comparison}). {In contrast, real-world deployment often exposes VLMs to adverse weather, dynamic occlusions, abrupt illumination changes, and camera instability.
As shown in \cref{tab : benchmark}, such perturbations can degrade both accuracy and reasoning quality by 12 to 35\%. Although ImageNet-C~\citep{xie2020self} introduced the evaluation of corruption robustness for image classification, no existing benchmark systematically measures how temporally coherent and spatially grounded visual perturbations affect reasoning over videos. This leaves a critical blind spot: we lack the tools to diagnose whether failures under visual corruption arise from perceptual errors, reasoning fragility, or both.}

\begin{wrapfigure}[18]{r}{0.45\linewidth} \centering \vspace{-0.15in} \includegraphics[width=0.96\linewidth]{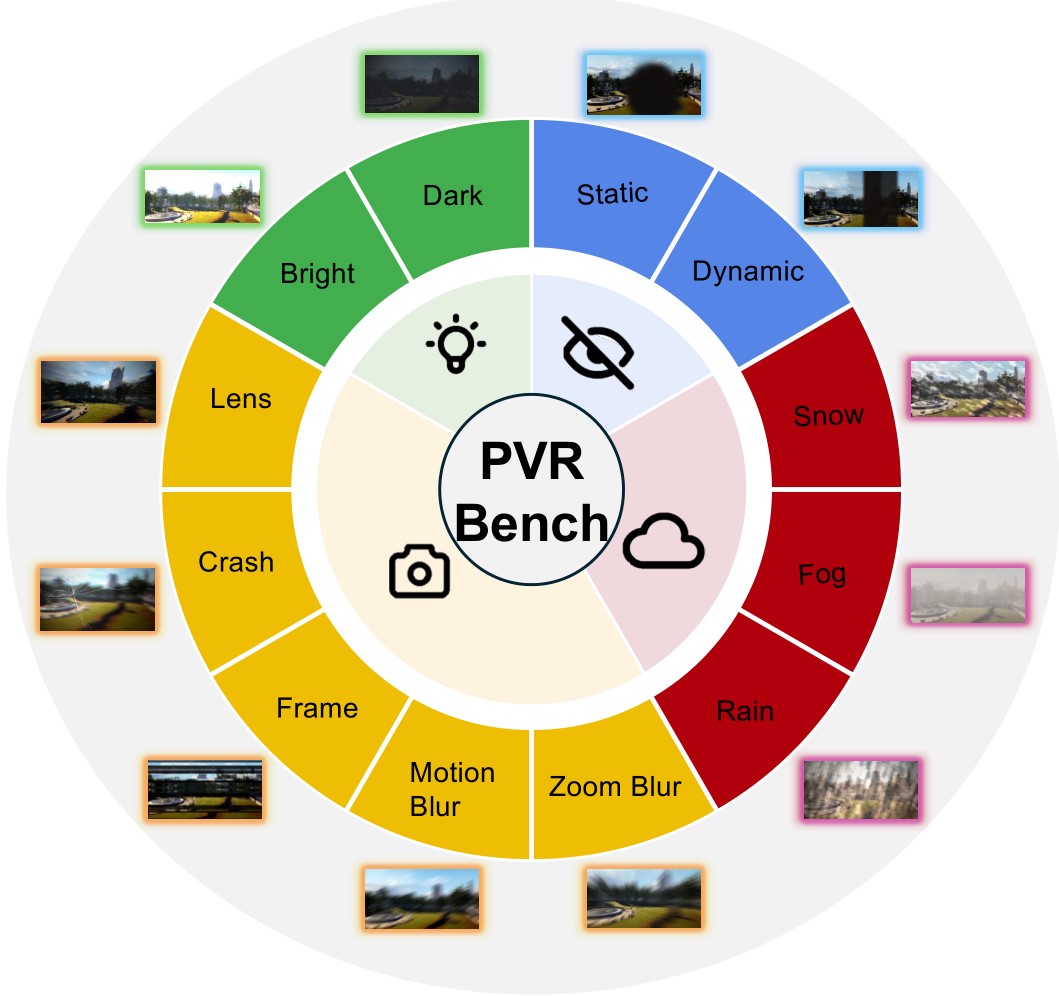} \vspace{-3pt} \caption{Overview of the perturbation types in \benchname{}.} \label{fig:discard_rate} \end{wrapfigure}
\noindent\textbf{Construction.} To close this gap, we introduce \textit{\benchnamefull{} (\benchname{})}, designed to evaluate the robustness of video reasoning models under structured, real-world visual variations beyond simple pixel-level corruption.
Our focus is on \emph{reasoning reliability}, defined as the ability to maintain coherent and logically consistent inference chains grounded in correct visual observations and valid causal steps despite degraded video input.
\benchname{} integrates four categories of realistic, video-specific disturbances: lighting (dusk, night, overexposure, shadow), camera motion (translation, zoom, rotation), occlusion (static, dynamic), and weather (fog, rain, snow). Each disturbance is applied with spatial awareness (e.g., depth-conditioned occlusion placement and scene-adapted weather rendering) and temporal coherence across frames. The benchmark comprises over 9K videos and 51K question-answer pairs spanning diverse indoor, outdoor, and embodied scenarios, with 27 task coverage from~\cite{zhao2025urbanvideo,yang2025thinking}, which exercise 
a broad spectrum of video reasoning capabilities.

{\noindent\textbf{Perturbation Injection.} At its core, we generate \emph{video-specific masks} (\cref{eq:combined_corruption}) that contain semantically coherent perturbations conditioned on each video's content, including depth layout, object locations, and motion patterns. These perturbations are contextually adapted to scene semantics; for instance, weather appears as windshield rain refraction in driving scenes, while occlusions are placed at plausible foreground locations.
For benchmark evaluation, we adopt a static protocol in which masks are pre-generated and fixed per video to ensure reproducible cross-model comparison, while \methodname{} training (\cref{sec : method_VRM}) uses a dynamic protocol that generates perturbations on the fly with stochastically sampled parameters at each iteration to prevent overfitting and promote perturbation invariant representation learning.}

\noindent\textbf{Evaluation Metrics.} To quantify reasoning reliability, \benchname{} introduces five complementary metrics (Fragility, Consistency, Belief, Recovery, and Attention; see \cref{tab:benchmark}) that assess the quality and stability of intermediate reasoning, as well as final-answer accuracy. To assess reasoning process quality, we leverage a powerful vision-language foundational model (e.g., GPT-4o) to score reasoning traces in coherence, perturbation awareness, and evidence grounding via a structured template (see~\cref{tab:prompt_reasoning}), following the LLM-as-judge paradigm~\citep{zheng2023judging, he2024videoscore}.

\section{Experiment}

\subsection{Implementation Details.} 
\textbf{Models.} We train our model on 4 NVIDIA A100 (80GB) GPUs. For optimization, we set the ordered group size to $G=8$ and the shuffled group size to $\tilde{G}=G/2$.  
Details are provided in~\cref{sec :hyperparameter}. 

\noindent\textbf{Datasets.} We use both clean and perturbed video data for training and evaluation. For training, we curate an outdoor-scene-relevant subset of Video-R1-260k ($\sim$10\% of its video split, filtered by scene category labels) and apply dynamic, randomly sampled perturbation masks to construct corruption-augmented video-query pairs. For evaluation, we assess generalization on the proposed \benchname{}, which contains over 51K question answer pairs across more than 9K videos spanning diverse scene categories beyond the training distribution. Static perturbation masks are systematically injected to measure model accuracy, reasoning quality, and robustness under both clean and corrupted conditions.
We further evaluate the generalization of VLMs on standard VisBench and UrbanVideo.

\begin{table*}[t!]
\centering
\setlength{\tabcolsep}{3.6pt}
\renewcommand{\arraystretch}{1.05}
\caption{\textbf{Evaluation on PVRBench.}
We report accuracy under four visual perturbations (\textbf{Lig}hting, \textbf{Occ}lusion, camera \textbf{Sha}ke, \textbf{Wea}ther) on the left, and reasoning quality metrics on the right, including \textbf{Fra}gility, \textbf{Con}sistency, \textbf{Bel}ief, \textbf{Rec}overy, and \textbf{Att}ention (0 - 5 scale; Higher is better, except for \textit{Fra} ($\downarrow$)).
\textit{\#Fr}: the number of frames, \textit{Avg.}: the average performance, and \textit{Orig.}: the average performance on clean (unperturbed) data. We exclude \textit{Fra.} when computing \textit{Avg.$^\dagger$} and \textit{Orig.$^\dagger$.}
}
\resizebox{\linewidth}{!}{%
\LARGE
\begin{tabular}{@{}lcc cccc cc | ccccc cc@{}}
\toprule
& & & \multicolumn{6}{c|}{\textbf{Answer Accuracy}} & \multicolumn{7}{c}{\textbf{Reasoning Quality}} \\
\cmidrule(lr){4-9} \cmidrule(lr){10-16}
\textbf{Model} & \textbf{Size} & \textbf{\#Fr}
& \textbf{Lig.} & \textbf{Occ.} & \textbf{Sha.} & \textbf{Wea.} & \textbf{Avg.} & \textbf{Orig.}
& \textbf{Fra.}$\downarrow$ & \textbf{Con.} & \textbf{Bel.} & \textbf{Rec.} & \textbf{Att.} & \textbf{Avg.$^\dagger$} & \textbf{Orig.$^\dagger$} \\
\midrule

\multicolumn{16}{@{}l}{\emph{Proprietary Models}} \\
GPT-4o & -- & 32
& .54 & .47 & .50 & .52 & .51 {\small\textcolor{medorange}{$\downarrow$14\%}} & .59
& 1.85 & 3.42 & 3.55 & 3.38 & 3.21 & 3.39 {\small\textcolor{medorange}{$\downarrow$11\%}} & 3.82 \\
Gemini-3-Pro & -- & 32
& .57 & .52 & .54 & .55 & .55 {\small\textcolor{medorange}{$\downarrow$11\%}} & .62
& 1.72 & 3.61 & 3.48 & 3.58 & 3.41 & 3.52 {\small\textcolor{medorange}{$\downarrow$10\%}} & 3.91 \\
Claude-3.5-Son. & -- & 32
& .45 & .41 & .44 & .45 & .44 {\small\textcolor{medorange}{$\downarrow$17\%}} & .53
& 2.08 & 3.18 & 3.22 & 2.95 & 3.15 & 3.13 {\small\textcolor{medorange}{$\downarrow$14\%}} & 3.65 \\

\midrule
\multicolumn{16}{@{}l}{\emph{Video Reasoning Models}} \\[1pt]
Video-R1 & 7B & 32
& .43 & .37 & .42 & .41 & .41 {\small\textcolor{badred}{$\downarrow$20\%}} & .51
& 2.48 & 2.75 & 2.85 & 2.68 & 2.65 & 2.73 {\small\textcolor{badred}{$\downarrow$20\%}} & 3.42 \\
Video-R1 & 72B & 32
& .51 & .45 & .49 & .49 & .49 {\small\textcolor{medorange}{$\downarrow$16\%}} & .58
& 2.11 & 3.25 & 3.18 & 3.21 & 2.98 & 3.16 {\small\textcolor{medorange}{$\downarrow$14\%}} & 3.68 \\
VideoChat-R & 7B & 16
& .36 & .31 & .36 & .35 & .35 {\small\textcolor{badred}{$\downarrow$22\%}} & .45
& 2.65 & 2.62 & 2.55 & 2.71 & 2.28 & 2.54 {\small\textcolor{badred}{$\downarrow$22\%}} & 3.25 \\
LLaVA-Video-R & 7B & 32
& .40 & .34 & .38 & .38 & .38 {\small\textcolor{badred}{$\downarrow$21\%}} & .48
& 2.58 & 2.68 & 2.61 & 2.78 & 2.42 & 2.62 {\small\textcolor{badred}{$\downarrow$21\%}} & 3.32 \\

\addlinespace[1.5pt]
\hdashline
\addlinespace[1.5pt]

\rowcolor{basegray}
Embodied-R & 7B & 32
& .45 & .38 & .42 & .43 & .42 {\small\textcolor{badred}{$\downarrow$22\%}} & .54
& 2.45 & 2.82 & 2.91 & 2.72 & 2.68 & 2.78 {\small\textcolor{medorange}{$\downarrow$19\%}} & 3.45 \\
\rowcolor{oursgray}
\textbf{$+$ \methodname{} (Ours)} & 7B & 32
& \textbf{.52} & \textbf{.46} & \textbf{.49} & \textbf{.51} & \textbf{.50} {\small\textcolor{goodgreen}{$\downarrow$9\%}} & \textbf{.55}
& \textbf{2.25} & \textbf{3.15} & \textbf{3.18} & \textbf{3.22} & \textbf{2.91} & \textbf{3.12} {\small\textcolor{medorange}{$\downarrow$13\%}} & \textbf{3.58} \\

\midrule
\multicolumn{16}{@{}l}{\emph{Open-Source Video LLMs}} \\
LLaVA-Video & 7B & 32
& .32 & .29 & .30 & .32 & .31 {\small\textcolor{badred}{$\downarrow$30\%}} & .44
& 2.78 & 2.45 & 2.35 & 2.52 & 2.25 & 2.39 {\small\textcolor{badred}{$\downarrow$23\%}} & 3.12 \\
VideoLLaMA2 & 7B & 16
& .28 & .25 & .27 & .29 & .27 {\small\textcolor{badred}{$\downarrow$25\%}} & .36
& 2.92 & 2.18 & 2.25 & 2.12 & 2.15 & 2.18 {\small\textcolor{badred}{$\downarrow$28\%}} & 3.01 \\
VideoChat2 & 7B & 16
& .26 & .23 & .25 & .27 & .25 {\small\textcolor{badred}{$\downarrow$26\%}} & .34
& 3.01 & 2.08 & 2.15 & 2.05 & 2.02 & 2.08 {\small\textcolor{badred}{$\downarrow$28\%}} & 2.88 \\
MiniCPM-V 2.6 & 8B & 64
& .34 & .28 & .31 & .32 & .31 {\small\textcolor{badred}{$\downarrow$28\%}} & .43
& 2.75 & 2.48 & 2.42 & 2.55 & 2.21 & 2.42 {\small\textcolor{badred}{$\downarrow$24\%}} & 3.18 \\
\addlinespace[1.5pt]
\hdashline
\addlinespace[1.5pt]

\rowcolor{basegray}
InternVL2.5 & 8B & 32
& .31 & .26 & .32 & .33 & .31 {\small\textcolor{badred}{$\downarrow$33\%}} & .46
& 2.85 & 2.38 & 2.28 & 2.42 & 2.18 & 2.32 {\small\textcolor{badred}{$\downarrow$26\%}} & 3.15 \\
\rowcolor{oursgray}
\textbf{$+$ \methodname{} (Ours)} & 8B & 32
& \textbf{.43} & \textbf{.36} & \textbf{.41} & \textbf{.40} & \textbf{.40} {\small\textcolor{goodgreen}{$\downarrow$15\%}} & \textbf{.47}
& \textbf{2.45} & \textbf{2.82} & \textbf{2.75} & \textbf{2.78} & \textbf{2.58} & \textbf{2.73} {\small\textcolor{medorange}{$\downarrow$17\%}} & \textbf{3.28} \\
\addlinespace[1.5pt]
\hdashline
\addlinespace[1.5pt]

\rowcolor{basegray}
Qwen2.5-VL & 7B & 32
& .35 & .28 & .34 & .34 & .33 {\small\textcolor{badred}{$\downarrow$35\%}} & .51
& 2.71 & 2.58 & 2.62 & 2.68 & 2.31 & 2.55 {\small\textcolor{badred}{$\downarrow$25\%}} & 3.41 \\
\rowcolor{oursgray}
\textbf{$+$ \methodname{} (Ours)} & 7B  & 32
& \textbf{.48} & \textbf{.43} & \textbf{.47} & \textbf{.49} & \textbf{.47} {\small\textcolor{goodgreen}{$\downarrow$11\%}} & \textbf{.53}
& \textbf{2.31} & \textbf{3.05} & \textbf{3.08} & \textbf{2.98} & \textbf{2.85} & \textbf{2.99} {\small\textcolor{medorange}{$\downarrow$15\%}} & \textbf{3.52} \\

\addlinespace[1.5pt]
\hdashline
\addlinespace[1.5pt]

\rowcolor{basegray}
Qwen2.5-VL & 72B & 32
& .48 & .41 & .44 & .47 & .45 {\small\textcolor{badred}{$\downarrow$21\%}} & .57
& 2.18 & 3.15 & 3.08 & 2.92 & 3.12 & 3.07 {\small\textcolor{medorange}{$\downarrow$16\%}} & 3.64 \\
\rowcolor{oursgray}
\textbf{$+$ \methodname{} (Ours)} & 72B  & 32
& \textbf{.57} & \textbf{.53} & \textbf{.56} & \textbf{.56} & \textbf{.56} {\small\textcolor{goodgreen}{$\downarrow$5\%}} & \textbf{.59}
& \textbf{1.95} & \textbf{3.45} & \textbf{3.35} & \textbf{3.42} & \textbf{3.18} & \textbf{3.35} {\small\textcolor{medorange}{$\downarrow$10\%}} & \textbf{3.72} \\
\addlinespace[1.5pt]
\hdashline
\addlinespace[1.5pt]
\rowcolor{basegray}
Qwen3-VL & 13B & 32
& .43 & .35 & .39 & .42 & .40 {\small\textcolor{badred}{$\downarrow$25\%}} & .53
& 2.41 & 2.85 & 2.92 & 2.78 & 2.72 & 2.82 {\small\textcolor{medorange}{$\downarrow$19\%}} & 3.48 \\
\rowcolor{oursgray}
\textbf{$+$ \methodname{} (Ours)} & 13B & 32
& \textbf{.53} & \textbf{.49} & \textbf{.52} & \textbf{.54} & \textbf{.52} {\small\textcolor{goodgreen}{$\downarrow$7\%}} & \textbf{.56}
& \textbf{2.12} & \textbf{3.28} & \textbf{3.32} & \textbf{3.18} & \textbf{3.05} & \textbf{3.21} {\small\textcolor{medorange}{$\downarrow$11\%}} & \textbf{3.62 }\\
\bottomrule
\end{tabular}}
\label{tab:benchmark}
\end{table*}
\subsection{Main Results}\label{main_results}
\noindent\textbf{ROVA Performance on PVRBench.} We extensively evaluate our approach on \benchname{} and the clean benchmark (Orig.: UrbanVideo and VSI-Bench) across diverse backbones, including video reasoning models and open-source video LLMs ranging from 7B to 72B. As shown in \cref{tab:benchmark}, among dedicated video reasoning models, \methodname{} consistently outperforms prior methods. In the 7B setting, it improves the best-performing model, Embodied-R, from 0.42 to 0.50 average accuracy under perturbations (more than $17\%$ relative gain), and even matches or surpasses the much larger Video-R1 72B. Importantly, it also achieves consistent improvements in reasoning quality, indicating stable and reliable reasoning under visual corruption. Most open-source video LLMs suffer substantial degradation under perturbations, with 21–35\% drops in accuracy and 16–28\% declines in reasoning quality relative to clean inputs. 

Notably, \methodname{} not only withstands the proposed perturbations but also enhances the model’s generalization performance, observing consistent gains on PVRBench and across unseen benchmarks (VisBench and UrbanVideo, \cref{fig:cross_benchmark}) in both answer accuracy and reasoning quality under clean and perturbed videos. 
These findings suggest that \methodname{} is able to learn perturbation-robust representations with strong transferability, enabling improved robustness and semantic understanding beyond the training distribution without domain-specific fine-tuning, while maintaining superior performance on clean data.

Beyond the accuracy and reasoning quality improvements, \cref{tab:training_efficiency} shows that \methodname{} is highly resource-efficient. Although the dual-branch design doubles the forward pass, the proposed curriculum (SRE + DRE + ME) more than offsets this overhead, reducing GPU-hours by 5.9\% compared to naive Dual-Branch (134.4 vs.\ 142.8) while improving accuracy from 0.37 to 0.47.
Moreover, \methodname{} surpasses Video-R1 by 23.7\% (0.47 vs.\ 0.38) while using 60.4\% fewer GPU-hours (134.4 vs.\ 339.2), half the GPUs, and less than 8\% of the training data (32.5K vs.\ 425K). These results suggest that the dual-branch alignment objective learns transferable, perturbation-robust representations that generalize beyond the training distribution without domain-specific fine-tuning, while maintaining strong performance on clean data.

\begin{table*}[t!]
\centering
\definecolor{cmark}{HTML}{2E8B57}
\definecolor{xmark}{HTML}{CD5C5C}
\definecolor{oursrow}{HTML}{EAF2FB}
\definecolor{hdrblue}{HTML}{2C3E6B}
\caption{\textbf{Training efficiency comparison} (Qwen2.5-VL-7B, Orig.\ Acc.\ = 0.43; GPU-h = \#GPUs\,$\times$\,wall-clock hours). SRE = Self-Reflective Evaluation, DRE = Difficulty Re-Evaluation, ME = Memory Eviction. Robust.\ = dual-branch alignment with structured corruption (\cref{sec : corruption,sec : dual-branch}). Curric.\ = SRE\,+\,DRE\,+\,ME (\cref{sec : Data Curation}). 
}
\label{tab:training_efficiency}
\small
\setlength{\tabcolsep}{3.5pt}
\renewcommand{\arraystretch}{1.1}
\resizebox{.9\columnwidth}{!}{%
\begin{tabular}{@{}l ccc ccc c cc@{}}
\toprule
& \multicolumn{3}{c}{\textcolor{hdrblue}{\textbf{Training Data}}}
& \multicolumn{3}{c}{\textcolor{hdrblue}{\textbf{Architecture}}}
& \textcolor{hdrblue}{\textbf{Config.}}
& \multicolumn{2}{c}{\textcolor{hdrblue}{\textbf{Performance}}} \\
\cmidrule(lr){2-4} \cmidrule(lr){5-7} \cmidrule(lr){8-8} \cmidrule(lr){9-10}
& \textbf{SFT} & \textbf{RL} & \textbf{Total}
& \textbf{Branch} & \textbf{Robust.} & \textbf{Curric.}
& \textbf{GPUs}
& \textbf{GPU-h} & \textbf{Avg.\ Acc.} \\
\midrule
Std.\ GRPO
& {---} & {---} & {---}
& Single & \textcolor{xmark}{\ding{55}} & \textcolor{xmark}{\ding{55}}
& 4$\times$A100
& 71.6  & .45 \\
Na\"{i}ve Dual
& {---} & {---} & {---}
& Dual   & \textcolor{cmark}{\ding{51}} & \textcolor{xmark}{\ding{55}}
& 4$\times$A100
& 142.8 & .48 \\
Video-R1
& 165K & 260K & 425K
& Single & \textcolor{xmark}{\ding{55}} & \textcolor{xmark}{\ding{55}}
& 8$\times$A100
& 339.2 & .49 \\
\rowcolor{oursrow}
\textbf{\methodname{}}
& \textbf{6.5K} & \textbf{26K} & \textbf{32.5K}
& \textbf{Dual} & \textcolor{cmark}{\ding{51}} & \textcolor{cmark}{\ding{51}}
& \textbf{4$\times$A100}
& \textbf{134.4} & \textbf{.53} \\
\bottomrule
\end{tabular}%
}
\end{table*}
\begin{figure*}[t!]
    \centering
    \begin{subfigure}{0.31\textwidth}
        \centering
        \includegraphics[width=0.98\linewidth]{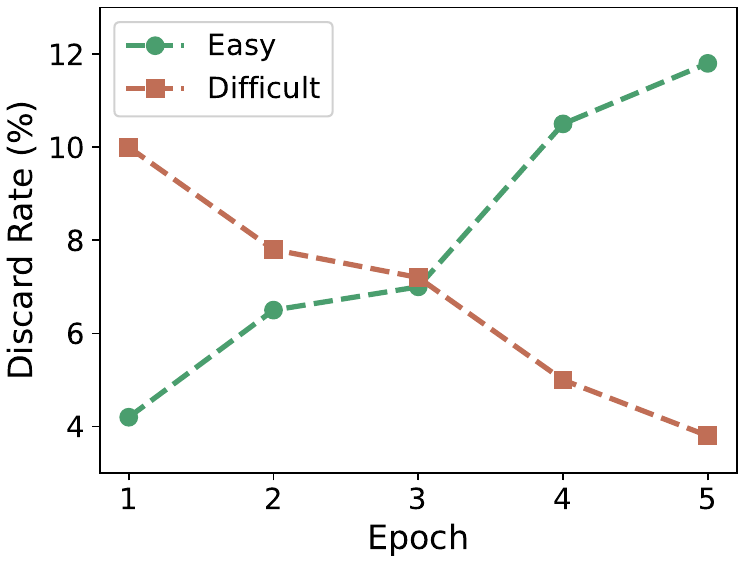}
        \caption{Sample discard rate evolution during self-reflective curriculum training.}
        \label{fig:discard}
    \end{subfigure}
    \hfill
    \begin{subfigure}{0.32\textwidth}
        \centering
        \includegraphics[width=0.99\linewidth]{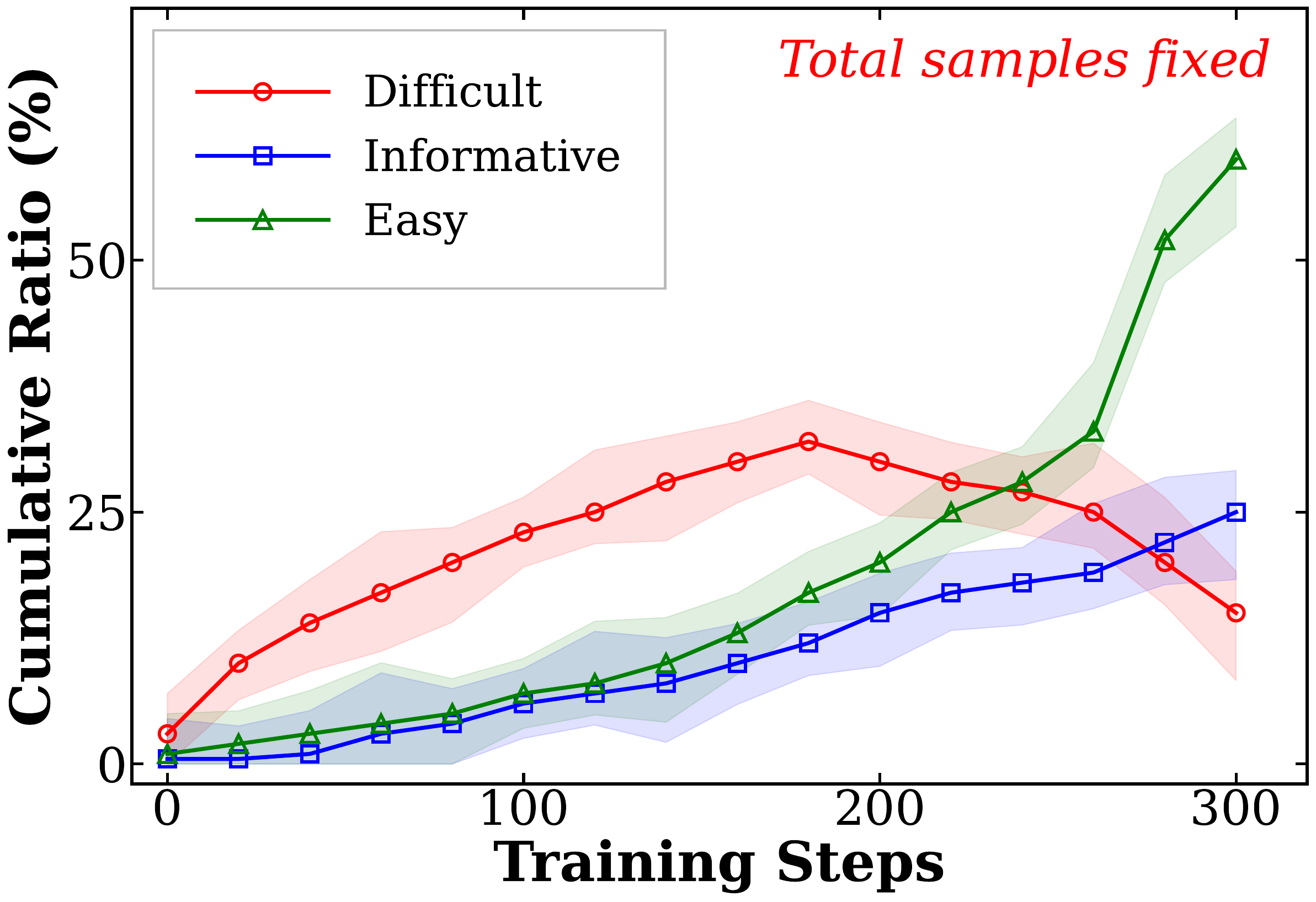}
        \vspace{-0.12in}
        \caption{Evolution of estimated easy, informative, and difficult sample proportions over training.}
        \label{fig:dynamics}
    \end{subfigure}
    \hfill
    \begin{subfigure}{0.32\textwidth}
        \centering
        \includegraphics[width=0.92\linewidth]{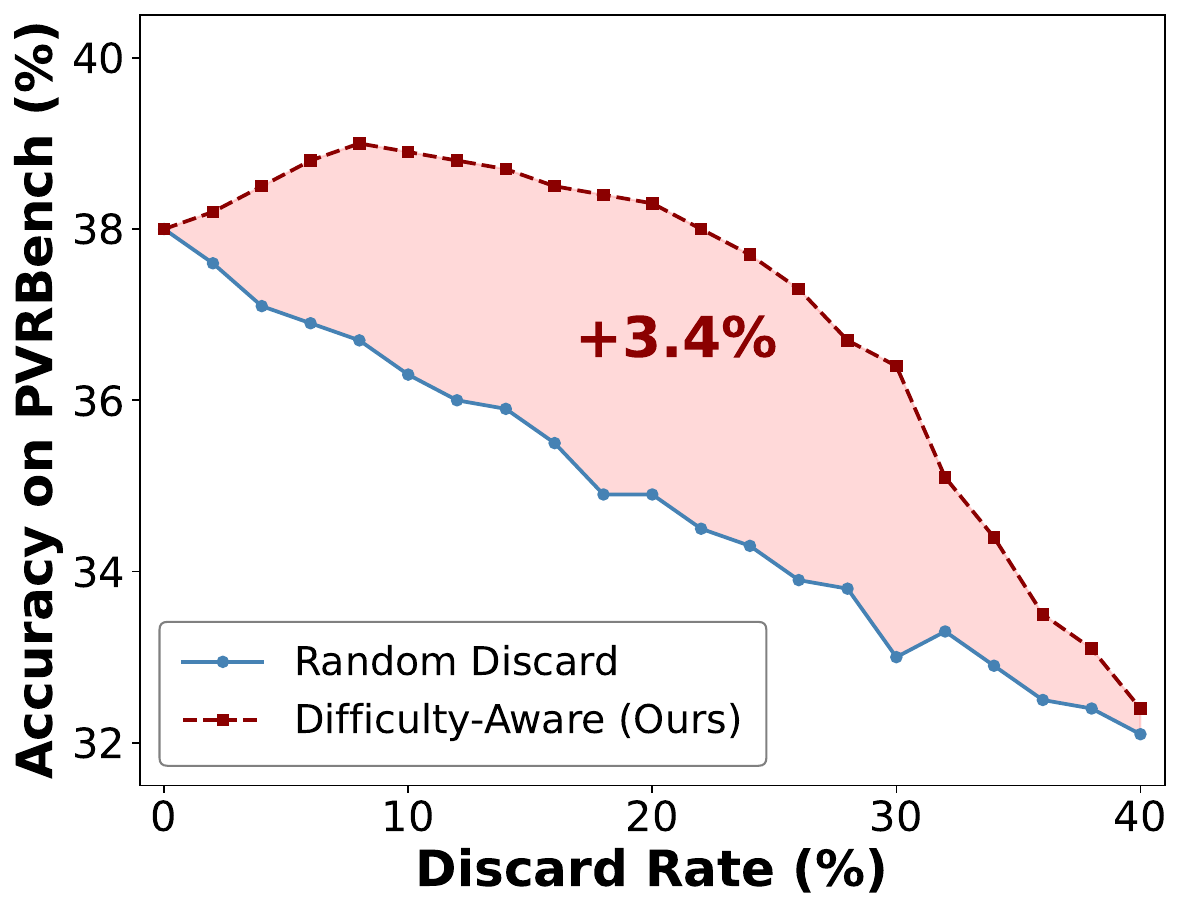}
        \caption{Difficulty-aware confidence-threshold discard vs. random across retention.}
        \label{fig:pev_eval}
    \end{subfigure}
    \caption{\textbf{Analysis of Self-Reflective Evaluation and Difficulty-Aware Training} for \methodname{} during the first Epoch of Qwen-VL-2.5-7B Training.}
    \label{fig:memory}
\end{figure*}
\noindent\textbf{Analysis of self-reflective evaluation and sample-selective training.}
We also analyze the behavior of our self-reflection evaluation mechanism during training. As shown in~\cref{fig:discard}, the discard rate for easy samples increases steadily over epochs while that for difficult samples declines, indicating that the model keeps evolving and smarter and prefers to decline more samples as they are already good at those, \cref{fig:discard}, shows a moderate fraction of samples is discarded overall, and the model selectively filters low-utility or overly noisy instances rather than aggressively pruning data. \cref{fig:dynamics} further illustrates the evolution of the estimated sample difficulty in training steps. While the total number of discarded samples is fixed, the composition gradually shifts toward easy samples, reflecting the improving competence of the model: samples initially deemed difficult are increasingly reclassified as easy as training progresses. This dynamic redistribution suggests that the self-reflective evaluator captures meaningful learning signals and adapts the curriculum in a data-driven manner. \cref{fig:pev_eval} demonstrates the effectiveness of difficulty-aware data selection for training. Compared to random discarding, our strategy consistently achieves higher accuracy across discard rates, with an improvement of up to 3.4\% on \benchname{}. This indicates that selective removal of samples based on estimated difficulty preserves informative training signals while avoiding detrimental noise.

\subsection{Ablation Study and Analysis}\label{sec : ablation_study}
\textbf{Ablation of Core Components.} We ablate each component of \methodname{} to assess its contribution (\cref{fig : component}). The reasoning reward yields the largest gain, followed by easy sample discarding, underscoring the central role of semantic reasoning and targeted curation. The memory module and temporal shuffle provide smaller but consistent gains, serving as complementary regularizers that stabilize training and enhance robustness.

\textbf{Ablation of Mask Styles.} We explore the generalizability of the proposed structured masking strategy compared to random masking baselines. As shown in \cref{fig : ablation_mask_strategies}, models trained on only two corruption mask styles achieve strong in-domain performance on the perturbation types seen during training, and more importantly, transfer effectively to held-out perturbation types (highlighted in red): out-of-domain performance remains close to in-domain results while both consistently surpass fixed-shape and pixel-level random masking by a significant margin (6 - 9\% absolute). This indicates that structured, perturbation-aware masks capture transferable corruption patterns rather than overfitting to specific disturbance types, confirming that a small subset of mask styles suffices to achieve broad robustness under diverse real-world disturbances.

\begin{figure*}[t!]
    \centering
    \begin{subfigure}[t]{0.48\textwidth}
        \centering        \includegraphics[width=\linewidth]{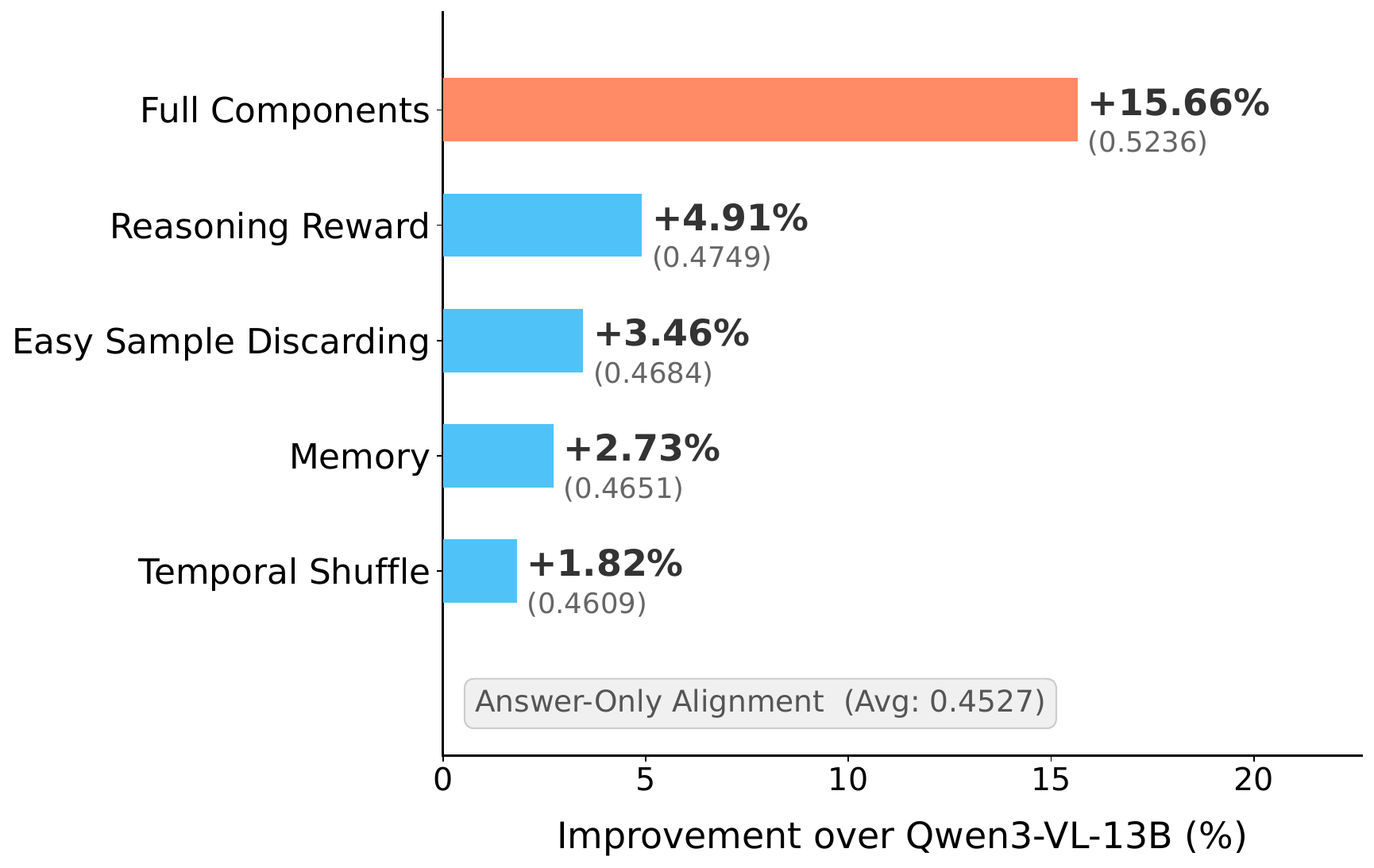}
        \vspace{-0.1in}
        \caption{
        Accuracy improvements from each component in \methodname{} over the base model (final-answer alignment only).
        }
        \label{fig : component}
    \end{subfigure}
    \hfill
    \begin{subfigure}[t]{0.50\textwidth}
        \centering
        \includegraphics[width=\linewidth]{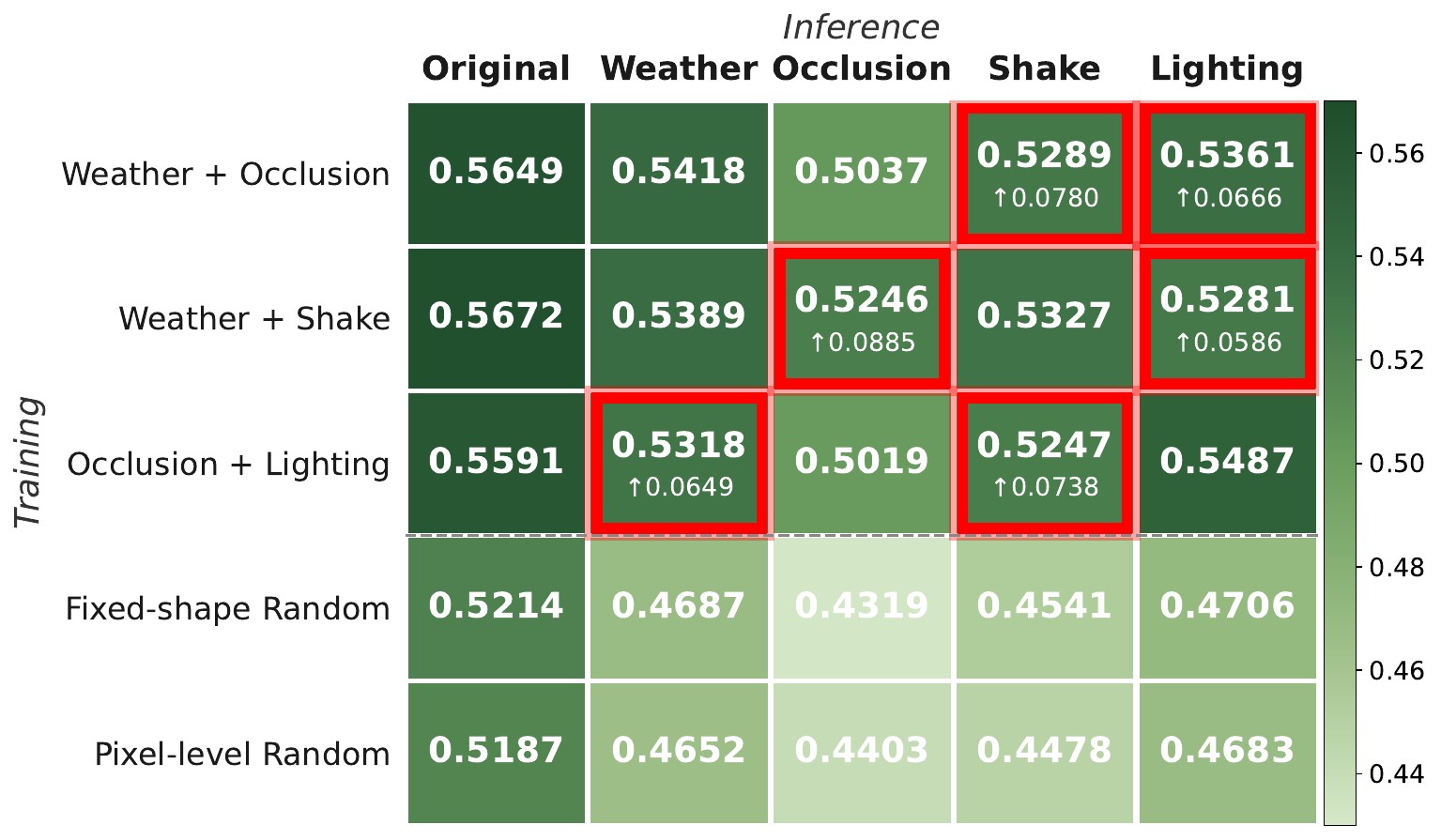}
        \vspace{-0.1in}
        \caption{
        Models trained on two mask styles are evaluated on in-domain and held-out OOD perturbations (highlighted in red).}
        \label{fig : ablation_mask_strategies}
    \end{subfigure}
    \caption{\textbf{Ablation studies of \methodname{}}. (a) Impact of individual components on answer accuracy. (b) Comparison of corruption mask strategies across perturbation types. Experiments are conducted using the Qwen3-VL-13B model trained for 3 epochs.}
    \label{fig:ablation_main}
\end{figure*}

\definecolor{oursrow}{HTML}{EAF2FB}

\begin{wraptable}[10]{r}{0.36\columnwidth}
\vspace{-0.17in}
\centering
\caption{Ablation study of the reward model on \benchname{} using commercial and open source VLMs.}
\label{tab:reward_ablation}

\scriptsize
\setlength{\tabcolsep}{2pt}
\renewcommand{\arraystretch}{1.12}

\resizebox{\linewidth}{!}{%
\begin{tabular}{@{}l cc c@{}}
\toprule
\textbf{Reward Judge} & \textbf{Acc.} & \textbf{Avg.} & \textbf{Free} \\
\midrule
\rowcolor{oursrow}
GPT-4o       & \textbf{0.470} & \textbf{2.99} & \ding{55} \\
Qwen3-13B    & 0.467          & 2.97          & \ding{51} \\
Qwen2.5-7B   & 0.463          & 2.95          & \ding{51} \\
\bottomrule
\end{tabular}%
}
\end{wraptable}
\textbf{Ablation of reward models.} 
Notably, our LLM judge (GPT-4o by default) outperforms rule- or embedding-based matching in evaluating semantic consistency across reasoning traces and final answers. Replacing it with open-source models (e.g., Qwen3-13B) yields comparable results, suggesting that the approach generalizes beyond proprietary APIs (\cref{tab:reward_ablation}). In contrast, more granular reward designs, such as conditional alignment or step-level consistency, introduce additional variance that destabilizes GRPO and degrades performance (\cref{tab:reward_alternatives}), further supporting LLM-based evaluation as the most effective approach.


\begin{figure*}[t!]
    \centering
    \includegraphics[width=\textwidth]{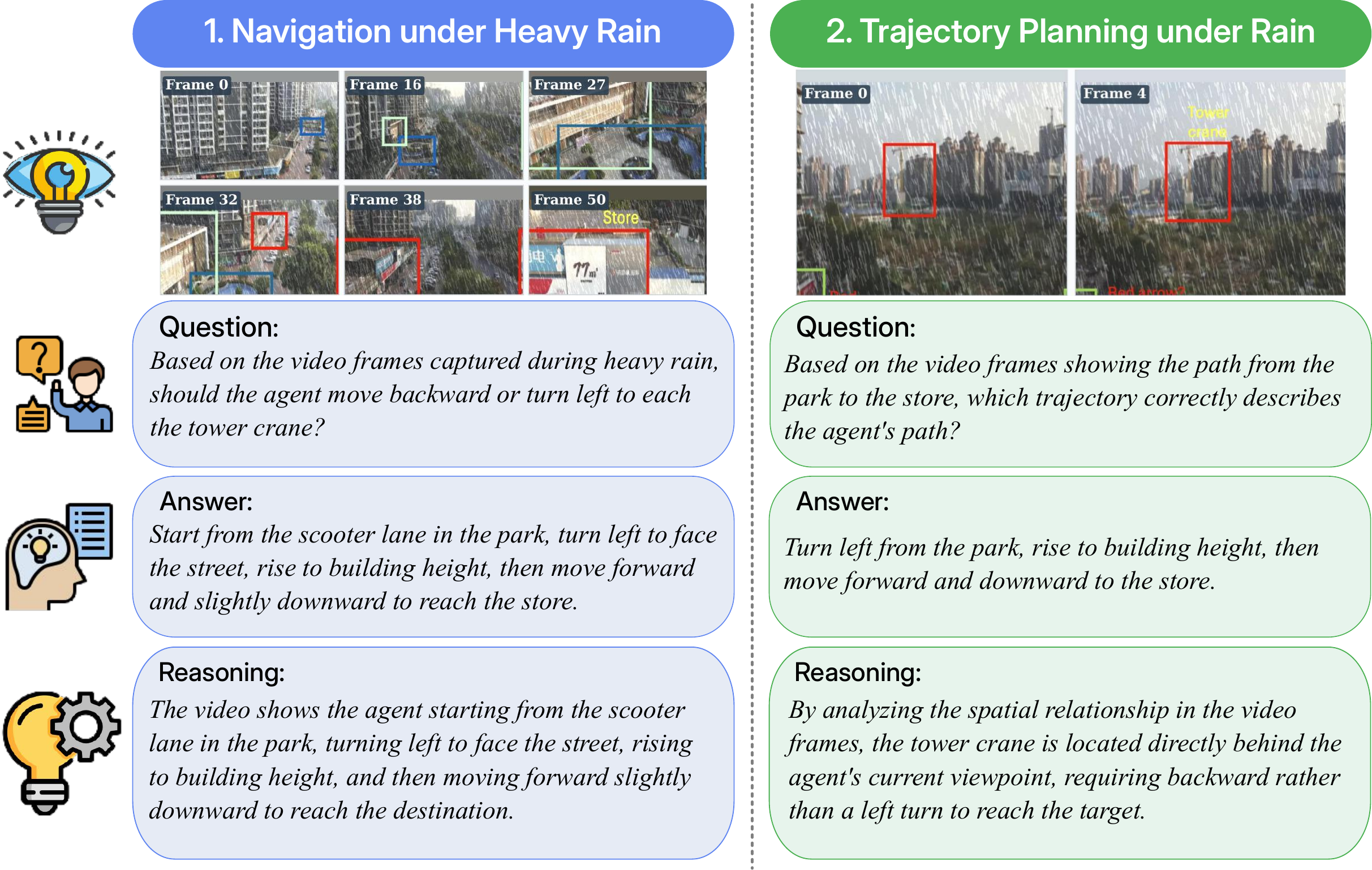}
    \vspace{-0.1in}
    \caption{Qualitative examples of ROVA-trained Qwen2.5-VL-7B performing obstacle avoidance and target identification under night-time low-light conditions. See more examples in~\cref{fig:case_study_2,fig:case_study_3,fig:case_study_4,fig:case_study_5}.}
    \label{fig:case_study_1}   
    \vspace{-0.15in}
\end{figure*}

\subsection{Qualitative Analysis}
We further validate the robustness of \methodname{} through qualitative examples on representative tasks in~\cref{fig:case_study_1}. Even in challenging scenarios where adverse weather or visual disturbances significantly degrade visibility, \methodname{} remains effective, correctly reasoning about the scene and task requirements. 
For instance, when heavy rain and glare obscure key visual cues, \methodname{} can still infer spatial relationships and scene structure, and when large objects block the field of view, it correctly reasons about the underlying layout rather than relying on partial appearances.
This shows that \methodname{} reliably interprets and reasons in visually impaired conditions, demonstrating robustness beyond controlled settings and confirming its effectiveness in difficult, realistic environments. 

\section{Conclusion}
In this work, we present \methodname{}, a robust training framework for embodied video reasoning that leverages structured spatio-temporal corruptions, dual-branch alignment, and self-reflective data curation to learn perturbation-robust representations. To evaluate robustness under realistic disturbances, we introduce \benchname{}. We show that \methodname{} consistently improves robustness under diverse real-world perturbations in video inputs while also improving performance on clean video–question pairs. These contributions provide both a principled benchmark and a practical training recipe, enabling future studies on broader perturbation families and more complex long-horizon embodied tasks.


\bibliography{main}
\bibliographystyle{plainnat}

\newpage
\appendix
\section*{Appendix} 
{
\hypersetup{linkcolor=navy}
\addcontentsline{toc}{section}{Appendix Table of Contents}
\startcontents[appendix]
\printcontents[appendix]{l}{1}{\setcounter{tocdepth}{2}}
}
\newpage

\section{Limitation}
While the proposed composite reward design proves effective in practice, several design choices warrant further investigation. First, both the format reward and accuracy reward are binary (0 or 1), offering no partial credit for nearly correct answers or partially well-structured outputs; a softer, continuous reward signal could provide richer gradients for GRPO optimization. Second, the proposed reward components are combined with equal weights, but the optimal balance among format compliance, answer correctness, and cross-branch alignment may vary across perturbation types and reasoning complexity. For simplicity, our framework does not adaptively adjust these weights during training. Third, the alignment reward relies on an external LLM judge to assess semantic consistency between clean and perturbed outputs, which introduces a dependency on the judge's own capability and potential biases; although we show that open-source alternatives (Qwen3-13B) yield comparable results, the reward signal remains bounded by the judge model's understanding of domain-specific reasoning. Fourth, our reward operates only at the holistic output level, evaluating the final answer and the overall reasoning trace, without providing step-level feedback on intermediate reasoning quality. As our ablation study confirms, more fine-grained reward designs, such as step-level consistency checks, tend to introduce variance that destabilizes GRPO training. Addressing this challenge between reward granularity and optimization stability, for instance, through hierarchical or curriculum-based reward shaping, remains an important direction for future work.
\section{Full Details of Dataset Construction}
\label{appendix:dataset}

\begin{sloppypar}

This section provides comprehensive documentation of the \benchname{} benchmark construction, including data sources, curation methodology, perturbation generation algorithms, and quality assurance protocols. Our benchmark integrates and augments two established embodied video reasoning datasets, UrbanVideo-Bench~\citep{zhao2025urbanvideo} and VSI-Bench~\citep{yang2025thinking}, to create the first large-scale robustness evaluation benchmark for video reasoning under realistic visual perturbations.

\end{sloppypar}

\subsection{Source Dataset Integration}
\label{appendix:source_integration}

\benchname{} is constructed by systematically combining the complete video corpora and question-answer annotations from two complementary benchmarks, resulting in a unified evaluation framework spanning both outdoor urban navigation and indoor spatial reasoning scenarios~(\cref{fig:qa_distribution}).

\subsubsection{UrbanVideo-Bench}
UrbanVideo-Bench~\citep{zhao2025urbanvideo} is an embodied video reasoning benchmark specifically designed for evaluating Video-LLMs on aerial agent motion in urban open-ended three-dimensional spaces. The benchmark addresses a critical gap in existing evaluations by focusing on the unique challenges of drone-based navigation in complex urban environments.

\paragraph{Data Collection Sources.} The video corpus comprises 1,547 video clips collected from three distinct sources:

\begin{enumerate}[leftmargin=*, itemsep=2pt]
    \item \textbf{Real-World Drone Footage} (Guangdong Province, China): Videos captured using two DJI Mini 4K drones operated by experienced pilots with over 1,000 hours of flight time. Data collection was conducted in Shenzhen and Zhaoqing, covering diverse urban landscapes including commercial districts, residential areas, parks, and waterfront regions. Resolution: $1280 \times 720$ pixels.
    
    \item \textbf{EmbodiedCity Simulator}: A high-fidelity simulation environment built on Unreal Engine using real Beijing city data. The simulator provides realistic 3D urban modeling with over 100 categories of micro urban elements (buildings, vehicles, pedestrians, signage, etc.). Resolution: $960 \times 720$ pixels.
    
    \item \textbf{AerialVLN Simulator}: A virtual urban environment specifically designed for aerial vision-language navigation research, built on Unreal Engine with AirSim integration for realistic drone physics. Resolution: $520 \times 520$ pixels.
\end{enumerate}

\paragraph{Video Characteristics.} The collected videos span a wide range of characteristics. Their durations vary from 10 seconds to 10 minutes, with a mean length of 87.3s and a median of 52.1s, and frame rates range from 24 to 30 fps depending on the source. All videos are captured using a single forward-facing camera mounted on a gimbal that supports a downward tilt between $0^\circ$ and $90^\circ$. In terms of motion, the videos feature purposeful navigation trajectories, including ascent and descent, horizontal translation, rotation, as well as compound movements that combine multiple motion types.

\begin{figure}[t!]
\centering
\begin{subfigure}[t]{0.48\textwidth}
    \centering
    \includegraphics[width=\textwidth]{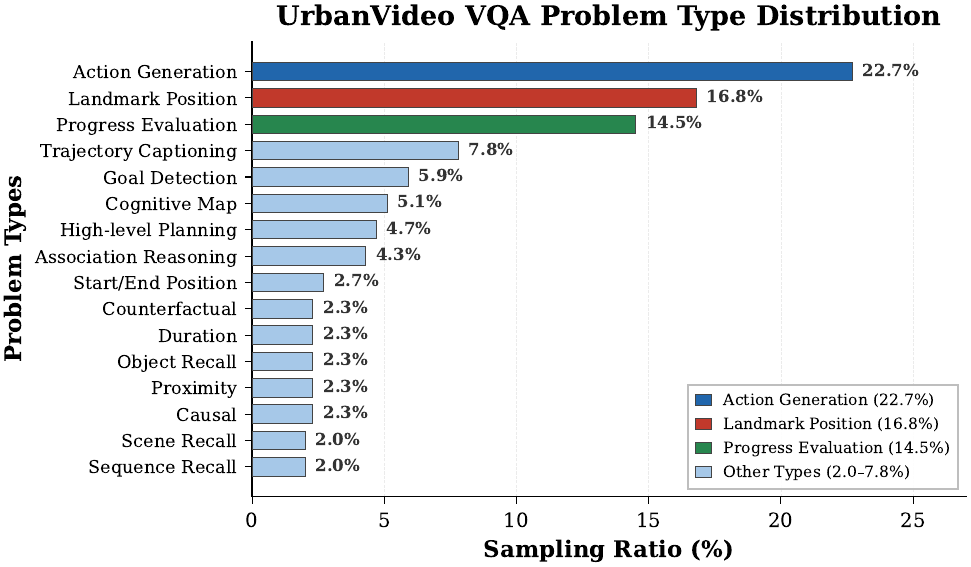}
    \caption{UrbanVideo-Bench QA type distribution. Action Generation (22.7\%), Landmark Position (16.8\%), and Progress Evaluation (14.5\%) dominate, reflecting the navigation-centric design.}
    \label{fig:urbanvideo_dist}
\end{subfigure}%
\hfill
\begin{subfigure}[t]{0.48\textwidth}
    \centering
    \includegraphics[width=\textwidth]{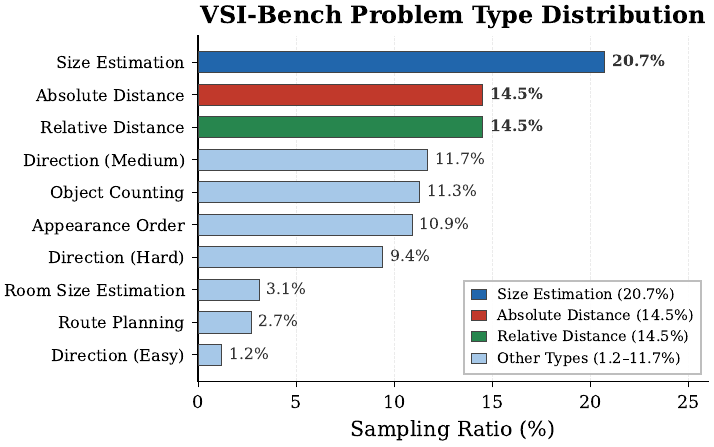}
    \caption{VSI-Bench QA type distribution. Size Estimation (20.7\%) and distance tasks (29.0\% combined) are most prevalent, reflecting the spatial measurement focus.}
    \label{fig:vsibench_dist}
\end{subfigure}
\caption{Question-answer type distributions for \benchname{} source datasets. The complementary distributions - UrbanVideo emphasizing navigation/action and VSI-Bench emphasizing spatial perception - together provide comprehensive coverage of embodied video reasoning capabilities.}
\label{fig:qa_distribution}
\end{figure}

\paragraph{Task Taxonomy.} UrbanVideo-Bench defines 16 task types organized into four cognitive ability categories, as shown in \cref{tab:urbanvideo_tasks}.

\begin{table}[t!]
\centering
\caption{Complete task taxonomy for UrbanVideo-Bench with 16 tasks across 4 cognitive ability categories.}
\label{tab:urbanvideo_tasks}
\vspace{.1in}
\small
\renewcommand{\arraystretch}{1.15}
\begin{tabularx}{\textwidth}{@{}l l X@{}}
\toprule
\textbf{Category} & \textbf{Task} & \textbf{Description} \\
\midrule
\multirow{5}{*}{Recall} 
    & Trajectory Captioning & Summarize agent movement using visual landmarks \\
    & Sequence Recall & Identify next action after specific movement \\
    & Object Recall & Locate objects relative to landmarks \\
    & Scene Recall & Describe observations during specific actions \\
    & Start/End Position & Identify journey origin and destination \\
\midrule
\multirow{5}{*}{Perception}
    & Proximity & Track distance changes to landmarks \\
    & Duration & Compare temporal duration of movements \\
    & Landmark Position & Determine egocentric position relative to goals \\
    & Goal Detection & Identify if/where destination is visible \\
    & Cognitive Map & Summarize spatial environment layout \\
\midrule
\multirow{3}{*}{Reasoning}
    & Causal & Explain reasons for specific movements \\
    & Counterfactual & Evaluate alternative action consequences \\
    & Association & Identify relevant objects when the goal is not visible \\
\midrule
\multirow{3}{*}{Navigation}
    & Progress Evaluation & Assess current step in navigation route \\
    & High-level Planning & Determine next waypoint toward goal \\
    & Action Generation & Output specific control actions \\
\bottomrule
\end{tabularx}
\end{table}

\subsubsection{VSI-Bench}
VSI-Bench~(Visual Spatial Intelligence Benchmark)~\citep{yang2025thinking} evaluates spatial reasoning capabilities from egocentric video perspectives in indoor environments. The benchmark focuses on fundamental spatial cognition tasks that require understanding of 3D space from sequential visual observations.

\paragraph{Data Sources.} VSI-Bench aggregates videos from three public indoor scene datasets: ARKitScenes, which provides real-world indoor scans captured using Apple ARKit; ScanNet, a widely used dataset of RGB-D indoor scene reconstructions; and 3RScan, a large-scale real-world indoor dataset enriched with instance-level annotations.
\paragraph{Scene Categories.} The 288 videos span six indoor environment types, as detailed in \cref{tab:vsi_scenes}.
\begin{table}[t!]
\centering
\caption{VSI-Bench scene category distribution across 288 videos.}
\label{tab:vsi_scenes}
\vspace{.1in}
\small
\begin{tabular}{@{}lrp{7cm}@{}}
\toprule
\textbf{Scene Type} & \textbf{Proportion} & \textbf{Characteristics} \\
\midrule
Living Rooms & 22.1\% & Social spaces with seating, entertainment systems \\
Bedrooms & 19.3\% & Sleeping areas with beds, wardrobes, personal items \\
Kitchens & 18.4\% & Cooking areas with appliances, countertops, cabinets \\
Offices & 15.8\% & Workspaces with desks, chairs, equipment \\
Bathrooms & 12.7\% & Sanitary facilities with fixtures \\
Hallways/Other & 11.7\% & Transitional spaces and miscellaneous areas \\
\bottomrule
\end{tabular}
\end{table}

\paragraph{Task Categories.} VSI-Bench defines 11 spatial reasoning tasks, as shown in \cref{tab:vsibench_tasks}.

\begin{table}[t!]
\centering
\caption{VSI-Bench task distribution with spatial reasoning focus.}
\label{tab:vsibench_tasks}
\vspace{.1in}
\small
\begin{tabular}{@{}l r p{7.5cm}@{}}
\toprule
\textbf{Task} & \textbf{Prop.} & \textbf{Description} \\
\midrule
Size Estimation & 20.7\% & Estimate absolute dimensions of objects \\
Absolute Distance & 14.5\% & Measure distance between camera and objects \\
Relative Distance & 14.5\% & Compare distances to multiple objects \\
Direction (Medium) & 11.7\% & Determine object directions with moderate complexity \\
Object Counting & 11.3\% & Count instances of object categories \\
Appearance Order & 10.9\% & Sequence objects by order of appearance \\
Direction (Hard) & 9.4\% & Complex directional reasoning with occlusions \\
Room Size Estimation & 3.1\% & Estimate room dimensions \\
Route Planning & 2.7\% & Plan navigation paths through spaces \\
Direction (Easy) & 1.2\% & Simple directional questions \\
\bottomrule
\end{tabular}
\end{table}

\subsection{Video Perturbation Generation System}
\label{appendix:perturbation_system}

We develop a comprehensive video perturbation system that generates semantically coherent, temporally consistent, and physically plausible visual corruptions. Unlike generic image augmentation techniques (e.g., random cropping, color jittering, and Gaussian noise), our system models realistic disturbances that preserve the answerable nature of questions while challenging model robustness.

\subsubsection{System Architecture Overview}

The perturbation system comprises four specialized modules organized in a modular pipeline architecture. Each module can be applied independently or in combination, with perturbation type sampled uniformly from $\mathcal{M} = \{\text{lighting}, \text{camera}, \text{occlusion}, \text{weather}\}$.

\begin{table}[t!]
\centering
\caption{Video perturbation system architecture overview. Input video $V = \{f_1, \ldots, f_T\}$ is transformed to perturbed video $V' = \{f'_1, \ldots, f'_T\}$ via one of four modules.}
\label{tab:perturbation_arch}
\vspace{.1in}
\small
\begin{tabular}{@{}l l p{4.5cm}@{}}
\toprule
\textbf{Module} & \textbf{Effects} & \textbf{Real-World Scenario} \\
\midrule
Lighting & Dusk, Night, Overexposure, Shadow & Time-of-day changes, exposure errors \\
Camera Motion & Translation, Zoom, Rotation & Handheld shake, platform instability \\
Occlusion & Static, Dynamic & Lens obstruction, passing objects \\
Weather & Fog, Rain, Snow & Atmospheric conditions \\
\bottomrule
\end{tabular}
\end{table}

\section{Prompt Templates}
\label{appendix:prompts}

This section documents the complete prompt templates used in \methodname{} for alignment reward computation and self-reflective difficulty assessment.

\subsection{Alignment Reward Prompts}
\label{appendix:alignment_prompts}
As shown in~\cref{alg:rova}, the alignment reward $r^A_j$ evaluates the consistency between outputs from the original and perturbed video branches by decomposing it into two complementary components: answer-level consistency and reasoning-level consistency, both assessed using GPT-4o.

For answer consistency, the evaluator employs a strict binary matching rule: if the candidate answer exactly matches or is semantically equivalent to the reference answer (e.g., “0” vs. “zero”), a score of 1.0 is assigned; otherwise, the score is 0.0, with no partial credit allowed (see answer consistency prompt template~(\cref{tab:prompt_answer})).

For reasoning consistency, a three-tier scoring scheme is used: a score of 1.0 indicates that the candidate reasoning is fully consistent with the reference, allowing for paraphrasing and minor omissions; 0.5 indicates general consistency but includes unsupported additions or missing key steps; and 0.0 indicates contradiction or hallucination of core facts. Critically, scoring is based solely on the reasoning process, independent of the final answer (see reasoning consistency prompt template~(\cref{tab:prompt_reasoning})).

Together, these two metrics - answer matching and reasoning alignment - enable a fine-grained evaluation of output consistency under perturbation, promoting both semantic robustness and reasoning fidelity in the model.

\begin{figure*}[t!]
\begin{tcolorbox}[
    colback=blue!3,
    colframe=blue!70,
    title={\large\textbf{$\triangleright$ Answer Consistency Evaluation Prompt}},
    fonttitle=\bfseries\large,
    boxrule=1pt,
    arc=4pt
]

\textbf{[Task]}\\
You are a strict evaluator responsible for assessing whether the candidate answer matches the reference answer. Score consistency \textbf{only} based on whether the CANDIDATE $\langle$answer$\rangle$ is semantically identical to the REFERENCE $\langle$answer$\rangle$. Do not consider reasoning quality, explanation depth, or stylistic differences.

\vspace{2mm}
\textbf{[Evaluation Criteria]}\\
Rate the answer on a binary scale:
\begin{itemize}[leftmargin=1.5em, itemsep=1pt, topsep=2pt]
    \item \textbf{Score 1.0}: The candidate answer is exactly the same as, or clearly equivalent to, the reference answer (e.g., ``0'' vs.\ ``zero'', ``NYC'' vs.\ ``New York City'').
    \item \textbf{Score 0.0}: The candidate answer differs from the reference answer in any substantive way.
\end{itemize}
Do not reward partial credit. Minor formatting or punctuation differences should be tolerated, but semantic mismatches must receive a score of 0.0.

\vspace{2mm}
\textbf{[Input]}
\begin{itemize}[leftmargin=1.5em, itemsep=1pt, topsep=2pt]
    \item Reference Answer: \texttt{\{reference\_answer\}}
    \item Candidate Answer: \texttt{\{candidate\_answer\}}
\end{itemize}

\vspace{2mm}
\textbf{[Output Format]}\\
Return a JSON object with the following fields. Only output the JSON object - no explanations, no justifications, and no extra text of any kind.
\vspace{1mm}
\begin{lstlisting}[basicstyle=\small\ttfamily, breaklines=true, columns=fullflexible, backgroundcolor=\color{blue!8}]
{"score": 0.0 or 1.0,
 "match_type": "exact" or "equivalent" or "mismatch"}
\end{lstlisting}
\end{tcolorbox}
\caption{Answer consistency evaluation prompt for binary answer matching.}
\label{tab:prompt_answer}
\end{figure*}

\begin{figure*}[t!]
\begin{tcolorbox}[
    colback=teal!3,
    colframe=teal!70,
    title={\large\textbf{$\triangleright$ Reasoning Consistency Evaluation Prompt}},
    fonttitle=\bfseries\large,
    boxrule=1pt,
    arc=4pt
]

\textbf{[Task]}\\
You are a strict evaluator responsible for assessing whether the candidate reasoning is consistent with the reference reasoning. Score consistency \textbf{only} based on whether the CANDIDATE $\langle$think$\rangle$ matches the REFERENCE $\langle$think$\rangle$ in key evidence and logical steps. Do \textbf{not} evaluate the correctness of the final answer.

\vspace{2mm}
\textbf{[Evaluation Criteria]}\\
Rate the reasoning on a three-tier scale:
\begin{itemize}[leftmargin=1.5em, itemsep=1pt, topsep=2pt]
    \item \textbf{Score 1.0}: The candidate reasoning is consistent with the reference up to paraphrasing and minor omissions. All key observations and logical steps are preserved.
    \item \textbf{Score 0.5}: The candidate reasoning is mostly consistent but contains unsupported additions, missing key intermediate steps, or minor logical deviations.
    \item \textbf{Score 0.0}: The candidate reasoning contradicts core observations from the reference or hallucinates key facts not present in the reference.
\end{itemize}

\vspace{2mm}
\textbf{[Evaluation Guidelines]}
\begin{itemize}[leftmargin=1.5em, itemsep=1pt, topsep=2pt]
    \item Focus exclusively on the reasoning process --- ignore the final answer.
    \item Tolerate stylistic and structural differences if the underlying logic is equivalent.
    \item Penalize fabricated evidence or contradictions to reference observations.
\end{itemize}

\vspace{2mm}
\textbf{[Input]}
\begin{itemize}[leftmargin=1.5em, itemsep=1pt, topsep=2pt]
    \item Reference Reasoning: \texttt{\{reference\_think\}}
    \item Candidate Reasoning: \texttt{\{candidate\_think\}}
\end{itemize}

\vspace{2mm}
\textbf{[Output Format]}\\
Return a JSON object with the following fields. Only output the JSON object --- no explanations, no justifications, and no extra text of any kind.

\vspace{1mm}
\begin{lstlisting}[basicstyle=\small\ttfamily, breaklines=true, columns=fullflexible, backgroundcolor=\color{teal!8}]
{"score": 0.0 or 0.5 or 1.0,
 "justification": "<explanation>"}
\end{lstlisting}
\end{tcolorbox}
\caption{Reasoning consistency evaluation prompt with three-tier scoring.}
\label{tab:prompt_reasoning}
\end{figure*}
\subsection{Difficulty Assessment Judge Prompt}
\label{appendix:judge_prompt}
\cref{tab:prompt_judge} illustrates the self-reflective difficulty assessment that employs an LLM judge to determine sample answerability under visual perturbations. The LLM receives a binary assessment prompt that strictly constrains it to evaluate only using the masked video. If the masked video provides sufficient information to reliably answer the given question, the LLM must output YES; otherwise, it must output NO. Following this judgment, samples classified as YES are treated as easy with low confidence or informative difficulty and are retained for training, while those classified as NO are deemed hard and are placed into a buffer for later re-evaluation—thereby enabling an adaptive, difficulty-aware curriculum that dynamically prioritizes informative training instances and defers overly challenging ones until the model is better equipped to handle them.
\begin{figure*}[t!]
\begin{tcolorbox}[
    colback=orange!3,
    colframe=orange!70,
    title={\textbf{$\blacktriangleright$~LLM Judge Prompt for Difficulty Assessment}},
    fonttitle=\bfseries\large,
    boxrule=1pt,
    arc=4pt
]
\textbf{[Task]}\\
You may ONLY use the MASKED video to judge.
\vspace{2mm}

\textbf{[Evaluation Criteria]}
\begin{itemize}[leftmargin=1.5em, itemsep=1pt, topsep=2pt]
    \item If the masked video \textbf{DOES} give enough information to reliably answer, respond: \textbf{YES}.
    \item If the masked video does \textbf{NOT} give enough information, respond: \textbf{NO}.
    \item Additionally, provide a \textbf{confidence score} in $[0.0,\,1.0]$ (one decimal place) reflecting how certain you are in your judgment.
\end{itemize}
Reply with \textbf{ONE WORD} and \textbf{ONE NUMBER} only.
\vspace{2mm}

\textbf{[Input]}
\begin{itemize}[leftmargin=1.5em, itemsep=1pt, topsep=2pt]
    \item Question: \texttt{\{question\_text\}}
\end{itemize}
\vspace{2mm}

\textbf{[Output Format]}
\vspace{1mm}
\begin{lstlisting}[basicstyle=\small\ttfamily, breaklines=true, columns=fullflexible, backgroundcolor=\color{orange!8}]
{
  "answer": "YES or NO",
  "confidence": 0.0
}
\end{lstlisting}
\end{tcolorbox}
\caption{LLM judge prompt for binary answerability assessment under perturbation. The confidence score controls the sample discard rate via threshold~$\tau$.}
\label{tab:prompt_judge}
\end{figure*}
\subsection{Complete Reward Computation Pipeline}
\label{appendix:reward_pipeline}
\cref{alg:reward_computation} details the complete reward computation pipeline used in \methodname{}. Given a paired output $(o_j, \tilde{o}_j)$ generated from the original and perturbed video branches, the pipeline proceeds in five sequential steps. First, format validation checks whether the output adheres to the required First, format validation checks whether the output adheres to the required format: \begin{center}
\texttt{<think>$\cdots$</think><answer>$\cdots$</answer>}
\end{center}. Second, the reasoning trace and final answer are extracted from both branches. Third, a binary accuracy reward $r^{\text{Acc}}_j$ is computed by comparing the extracted answer against the ground truth. Fourth, two alignment rewards are obtained via GPT-4o: a three-tier reasoning consistency score $r^{\text{align,r}}_j \in \{0, 0.5, 1\}$ that evaluates whether the key logical steps are preserved across branches, and a binary answer consistency score $r^{\text{align,a}}_j \in \{0, 1\}$ that checks semantic equivalence of the final answers. Finally, these components are aggregated into the total reward $R_j = r^F_j + r^{\text{Acc}}_j + \alpha_r \cdot r^{\text{align,r}}_j + \alpha_a \cdot r^{\text{align,a}}_j$, where the asymmetric weights $\alpha_r = 0.3$ and $\alpha_a = 0.7$ prioritize answer-level robustness while still encouraging reasoning fidelity (see \cref{sec :hyperparameter} for detailed hyperparameter specifications).

\begin{algorithm}[t!]
\caption{Alignment Reward Computation}
\label{alg:reward_computation}
\begin{algorithmic}[1]
\REQUIRE Output pair $(o_j, \tilde{o}_j)$ from original and perturbed branches, ground truth $g$
\ENSURE Total reward $R_j$
\STATE \textbf{Step 1: Format Validation}
\STATE $r^F_j \gets \texttt{regex\_match}(o_j, \texttt{"<think>.*</think>.*<answer>.*</answer>"})$
\STATE \textbf{Step 2: Extract Components}
\STATE $p_j \gets \texttt{extract}(o_j, \texttt{"<think>"})$; $a_j \gets \texttt{extract}(o_j, \texttt{"<answer>"})$
\STATE $\tilde{p}_j \gets \texttt{extract}(\tilde{o}_j, \texttt{"<think>"})$; $\tilde{a}_j \gets \texttt{extract}(\tilde{o}_j, \texttt{"<answer>"})$
\STATE \textbf{Step 3: Accuracy Reward}
\STATE $r^{\text{Acc}}_j \gets \mathbbm{1}[a_j = g]$
\STATE \textbf{Step 4: Alignment Rewards via GPT-4o}
\STATE $r^{\text{align,r}}_j \gets \texttt{GPT4o}(\texttt{reasoning\_prompt}, p_j, \tilde{p}_j)$ \COMMENT{$\in \{0, 0.5, 1\}$}
\STATE $r^{\text{align,a}}_j \gets \texttt{GPT4o}(\texttt{answer\_prompt}, a_j, \tilde{a}_j)$ \COMMENT{$\in \{0, 1\}$}
\STATE \textbf{Step 5: Aggregation}
\STATE $r^A_j \gets \alpha_r \cdot r^{\text{align,r}}_j + \alpha_a \cdot r^{\text{align,a}}_j$
\STATE $R_j \gets r^F_j + r^{\text{Acc}}_j + r^A_j$
\STATE \textbf{Return} $R_j$
\end{algorithmic}
\end{algorithm}
\begin{algorithm}[t!]
\caption{\textsc{\methodnamefull{}}}
\label{alg:rova}
\begin{algorithmic}[1]
\REQUIRE Policy $F_\theta$, buffer $\mathcal{M}\!=\!\varnothing$, data $\mathcal{D}$, params $(\alpha, \tau, K_{\max}, G)$
\\\textcolor{gray}{\# Self-Reflective Difficulty-Aware Training}
\FOR{$(q, V) \sim \mathcal{D}$}
    \STATE $\tilde{V} \gets \textsc{Perturb}(V)$ 
        \hfill {\footnotesize\textcolor{red}{$\triangleright$ Spatio-temporal corruption}}
    \STATE $\{o_j\}_{j=1}^G \!\sim\! F_\theta(\cdot|q,V)$;\; 
           $\{\tilde{o}_j\}_{j=1}^G \!\sim\! F_\theta(\cdot|q,\tilde{V})$
        \hfill {\footnotesize\textcolor{red}{$\triangleright$ Dual-branch}}
    \STATE $R_j \gets r_j + \alpha \!\cdot\! \textsc{Sim}(o_j, \tilde{o}_j)$;\;
           $A_j \gets (R_j \!-\! \bar{R})/\sigma_R$
        \hfill {\footnotesize\textcolor{red}{$\triangleright$ Alignment reward}}
    \STATE $F_\theta \gets \textsc{GRPOStep}(F_\theta, \{A_i\})$
        \hfill {\footnotesize\textcolor{red}{$\triangleright$ Policy update}}
    \STATE $(d, c) \gets F(q, \tilde{V}, S_e; \theta)$
        \hfill {\footnotesize\textcolor{red}{$\triangleright$ Self-assessment}}
    \IF{$d \!=\! \textsc{Hard}$}
        \STATE $\mathcal{M} \gets \mathcal{M} \cup \{(q, \tilde{V}, 0)\}$
            \hfill {\footnotesize\textcolor{red}{$\triangleright$ Buffer hard sample}}
    \ELSIF{$d \!=\! \textsc{Easy} \,\land\, c \!>\! \tau$}
        \STATE \textbf{skip}
            \hfill {\footnotesize\textcolor{red}{$\triangleright$ Prune mastered}}
    \ENDIF
\STATE{{\textcolor{gray}{\# Difficulty Re-Evaluation}}}
\STATE only when the memory is full or after sufficient iterations:
\FOR{$(q, \tilde{V}, n) \in \mathcal{M}$}
    \STATE $d' \gets \mathcal{A}(q, \tilde{V}, \theta_{\text{curr}})$;\; $n \gets n\!+\!1$
    \IF{$d' \!=\! \textsc{Informative}$}
        \STATE Train on $(q, \tilde{V})$; remove from $\mathcal{M}$
            \hfill {\footnotesize\textcolor{red}{$\triangleright$ Promote}}
    \ELSIF{$d' \!=\! \textsc{Easy}$ \textbf{or} $n \!>\! N_{\max}$}
        \STATE Remove from $\mathcal{M}$
            \hfill {\footnotesize\textcolor{red}{$\triangleright$ Evict}}
    \ENDIF
\ENDFOR
\ENDFOR
\end{algorithmic}
\end{algorithm}
\begin{figure*}[t!]
\centering
    \includegraphics[width=0.98\textwidth]{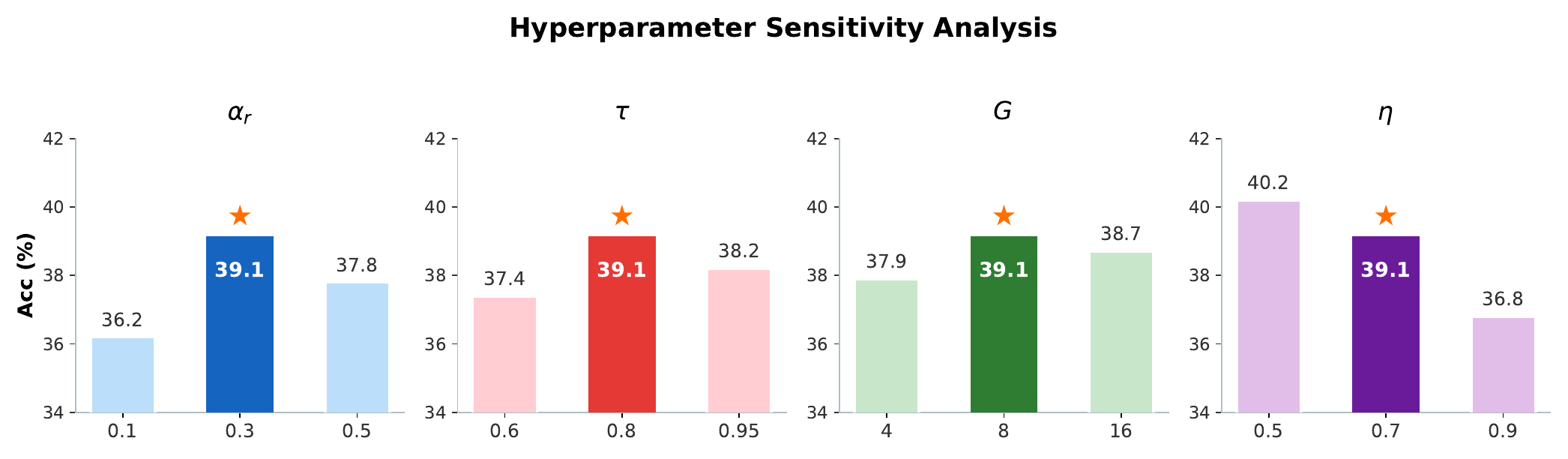}
        \caption{Hyperparameter sensitivity analysis of \methodname{} on the validation set, illustrating the effect of key training hyperparameters on model performance.}
        \label{fig:hyper}
\end{figure*}

\section{Hyperparameter} \label{sec :hyperparameter}
All hyperparameters used in \methodname{} are summarized in \cref{fig:hyper}.
For the reward function, the alignment component assigns $\alpha_r = 0.3$ to reasoning consistency and $\alpha_a = 0.7$ to answer consistency, reflecting the greater difficulty of strict reasoning alignment while prioritizing answer robustness; the base reward uses binary format and accuracy terms ($w_F = w_{\text{Acc}} = 1.0$) with KL regularization $\beta = 0.01$ and $K_{\max} = 537$.
For GRPO training, ordered and shuffled group sizes $G = 8$ and $\tilde{G} = 4$ ensure reliable advantage estimation, PPO clipping $\epsilon = 0.2$ with gradient norm 1.0 stabilizes policy updates, and GAE $\lambda_{\text{GAE}} = 0.95$ with $\gamma = 0.99$ yields a favorable bias--variance trade-off.
For the difficulty-aware curriculum, confidence threshold $\tau = 0.8$ with bounds $a_{\min} = 0.3$ and $a_{\max} = 0.85$ governs sample selection, while the buffer permits $N_{\max} = 3$ replay attempts over at most $|\mathcal{M}|_{\max} = 1000$ samples with re-evaluation every 50 steps.
Training uses 16 frames at $128{\times}28{\times}28$ (32 frames at $256{\times}28{\times}28$ at inference), AdamW with $\text{lr} = 1{\times}10^{-5}$ and cosine schedule on $4{\times}$A100 (80GB) GPUs, with 1 SFT epoch and 300 RL steps.
\subsubsection{Hyperparameter Sensitivity Analysis}
We conduct ablation studies on key hyperparameters to validate our design choices, as shown in Fig~\ref{fig :hyper}. The results indicate that setting the alignment weights to $\alpha_r = 0.3$ and $\alpha_a = 0.7$, which prioritizes answer alignment, leads to improved downstream accuracy while preserving reasoning quality. A confidence threshold of $\tau = 0.8$ provides an effective balance: lower thresholds retain an excessive number of easy samples, whereas higher thresholds discard valuable training signals. We find that a group size of $G = 8$ is sufficient to ensure stable advantage estimation, with larger group sizes yielding diminishing returns. Finally, a perturbation intensity of $\eta = 0.7$ achieves an appropriate balance between challenge and solvability - lower intensities fail to sufficiently enhance robustness, while higher intensities render samples unanswerable.
\begin{table}[t!]
\centering
\caption{Hyperparameter sensitivity analysis on the \benchname{} validation set for Qwen2.5-VL-7B after the first training epoch. Best values are highlighted in \textbf{bold}.}
\label{tab:sensitivity}
\vspace{.1in}
\small
\begin{tabular}{@{}l cc@{}}
\toprule
\textbf{Hyperparameter} & \textbf{Value} & \textbf{Avg. Acc. (\%)} \\
\midrule
\multirow{3}{*}{$\alpha_r$ (reasoning weight)}
    & 0.1 & 36.2 \\
    & \textbf{0.3} & \textbf{39.1} \\
    & 0.5 & 37.8 \\
\midrule
\multirow{3}{*}{$\tau$ (confidence threshold)}
    & 0.6 & 37.4 \\
    & \textbf{0.8} & \textbf{39.1} \\
    & 0.95 & 38.2 \\
\midrule
\multirow{3}{*}{$G$ (group size)}
    & 4 & 37.9 \\
    & \textbf{8} & \textbf{39.1} \\
    & 16 & 38.7 \\
\midrule
\multirow{3}{*}{$\eta$ (perturbation intensity)}
    & 0.5 & 40.2 \\
    & \textbf{0.7} & \textbf{39.1} \\
    & 0.9 & 36.8 \\
\bottomrule
\end{tabular}
\label{fig :hyper}
\end{table}



\section{Additional Experimental Results}
\label{appendix:additional_results}

\textbf{Fine-Grained Performance Analysis.} We further analyze \methodname{}'s performance through complementary perspectives (\cref{fig:radar_analysis_1,fig:radar_analysis_2,fig:radar_analysis_3,fig:radar_analysis_4,fig:radar_analysis_5}), which present radar charts comparing per-task accuracy of \methodname{} against the baselines across multiple task categories, revealing consistent improvements in high-level planning and associative reasoning. \cref{fig:frame_ablation} shows the impact of input frame count on robustness: increasing frames from 16 to 64 improves both baseline and \methodname{} performance across all perturbation types, confirming the benefit of longer temporal context. Notably, \methodname{} consistently outperforms the baseline at every frame count, indicating that our framework learns more robust representations rather than merely exploiting additional frames.

\textbf{Evolution of Reasoning and Answer Rewards.} We examine the reward dynamics of core components during \methodname{} training (\cref{fig:reward}). The total reward converges stably, while decomposed rewards show distinct patterns: accuracy reward rises rapidly and plateaus, reflecting task-specific learning; reasoning reward grows gradually, indicating deeper semantic understanding; and temporal reward shows gradual growth with the lowest variation rate among all components, acting as a temporal regularizer. This confirms that each component effectively guides different learning aspects.
\begin{figure}[t]
    \centering
    \includegraphics[width=0.9\textwidth]{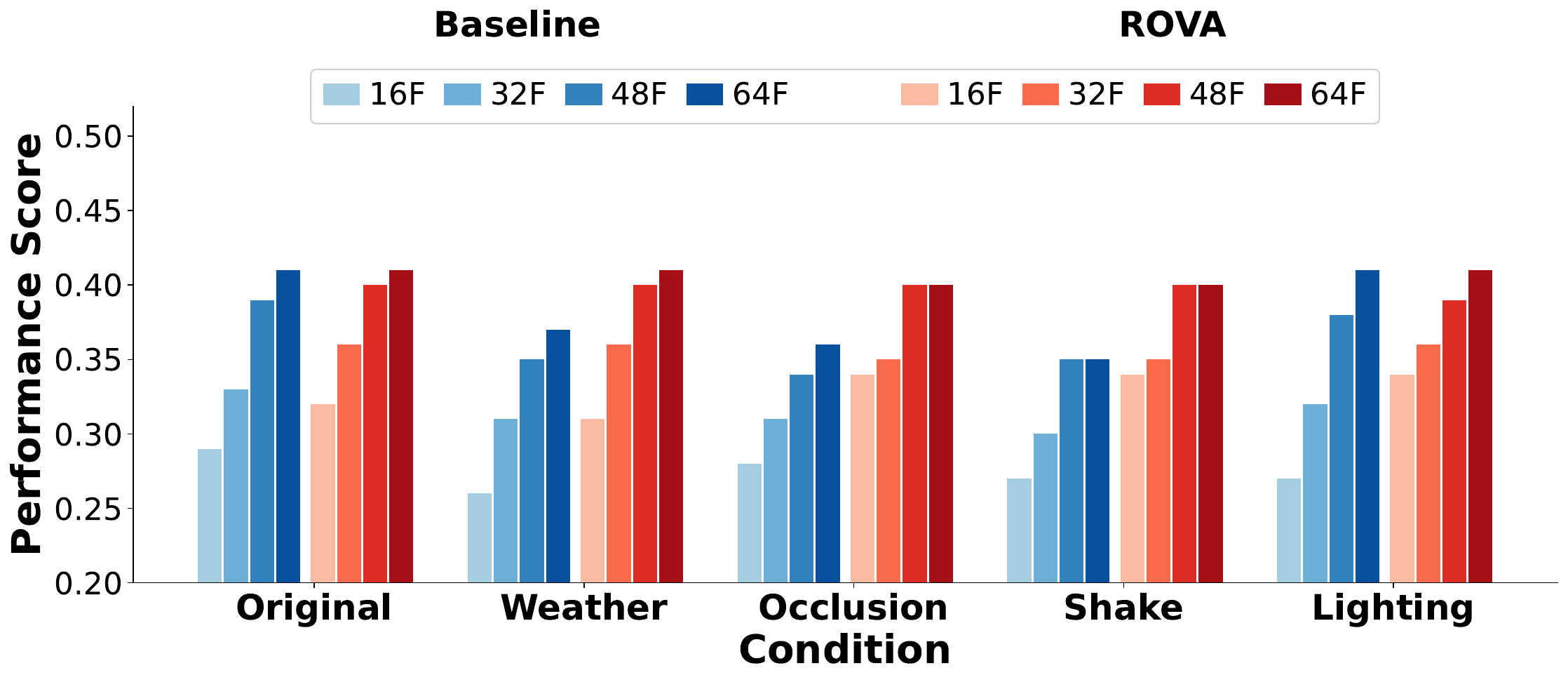}
    \caption{Performance of \methodname{} vs.\ baseline on Qwen2.5-VL-7B across varying frame counts (F = Number of Frames). \methodname{} outperforms the baseline at every frame count.}
    \label{fig:frame_ablation}
\end{figure}

\begin{figure*}[t]
\centering

\begin{minipage}{0.5\textwidth}
\centering
\includegraphics[width=\linewidth]{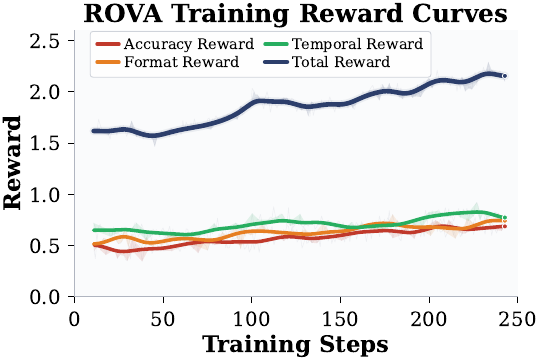}
\captionof{figure}{First epoch of Qwen-VL-2.5-7B training, the reward curves of \methodname{}}
\label{fig:reward}
\end{minipage}
\hfill
\begin{minipage}{0.48\textwidth}
\centering
\footnotesize
\renewcommand{\arraystretch}{1.05}
\setlength{\tabcolsep}{4pt}

\captionof{table}{The stability of easy-classified samples for Qwen2.5-VL-7B}
\label{tab:easy_stability}

\begin{tabular}{@{}c|ccc|ccc@{}}
\toprule
\multirow{2}{*}{\textbf{Step}} & \multicolumn{3}{c|}{\textbf{Retain Rate (\%)} $\uparrow$} & \multicolumn{3}{c}{\textbf{Confidence} $\uparrow$} \\
\cmidrule(lr){2-4} \cmidrule(l){5-7}
& Ep.1 & Ep.2 & Ep.3 & Ep.1 & Ep.2 & Ep.3 \\
\midrule
0   & --   & --   & --   & --   & --   & --   \\
50  & 82.3 & 86.1 & 89.4 & 0.71 & 0.74 & 0.77 \\
100 & 87.5 & 90.2 & 92.8 & 0.73 & 0.78 & 0.81 \\
150 & 91.2 & 93.6 & \cellcolor{green!8}95.1 & 0.76 & 0.81 & 0.84 \\
200 & 93.8 & \cellcolor{green!8}95.2 & \cellcolor{green!8}96.3 & 0.79 & 0.83 & 0.86 \\
250 & \cellcolor{green!8}95.1 & \cellcolor{green!8}96.0 & \cellcolor{green!8}96.8 & 0.81 & 0.85 & 0.88 \\
300 & \cellcolor{green!8}95.4 & \cellcolor{green!8}96.2 & \cellcolor{green!15}\textbf{97.1} & 0.82 & 0.86 & \textbf{0.89} \\
\bottomrule
\end{tabular}

\end{minipage}

\end{figure*}

\begin{figure*}[ht!]
    \centering
    \includegraphics[width=0.85\textwidth]{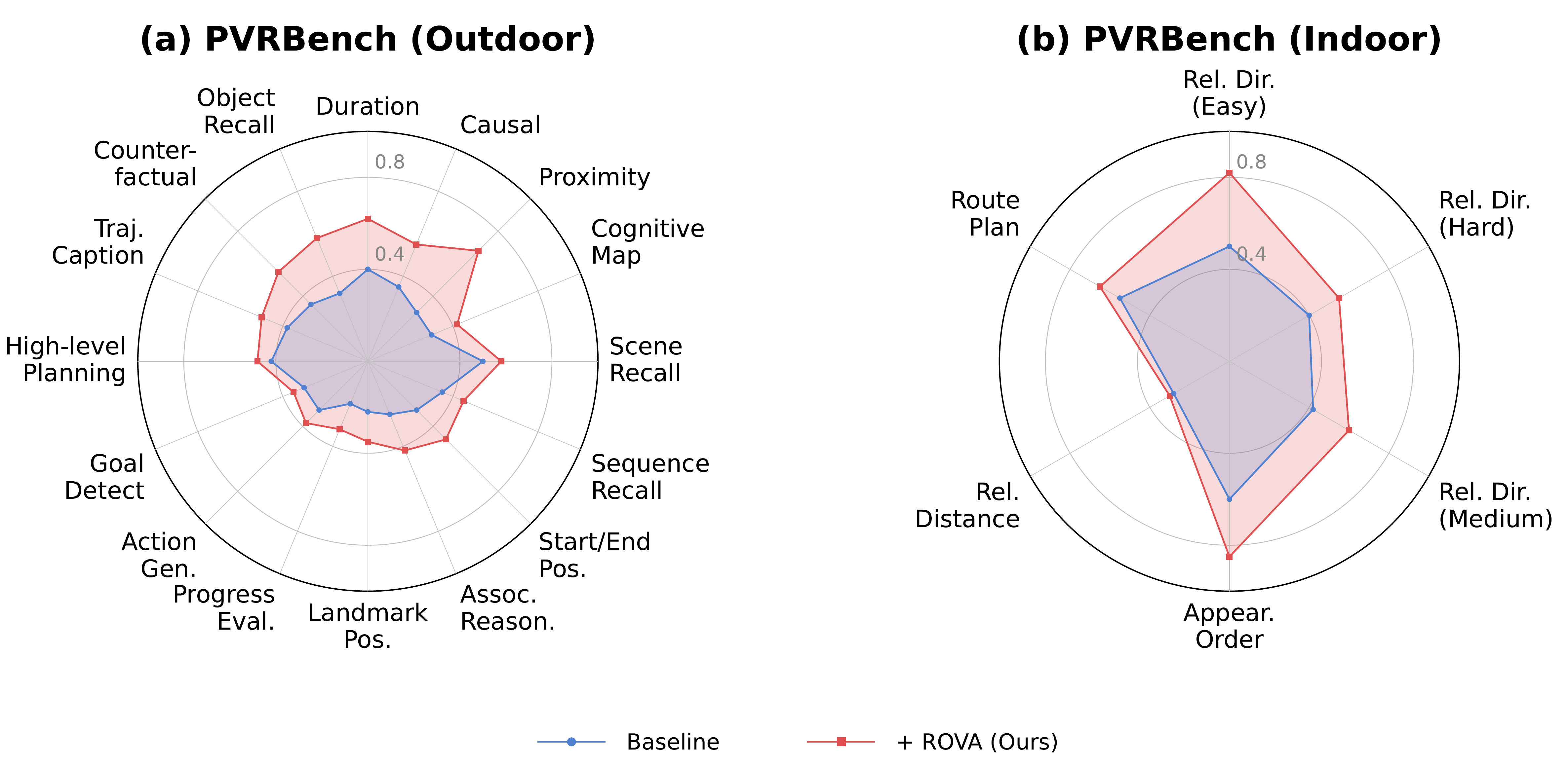}
    \caption{Per-task accuracy comparison of QwenVL-2.5-7B baseline vs. +ROVA on indoor spatial reasoning (left) and outdoor urban navigation (right) tasks, where the inner curve denotes the baseline, and the outer curve denotes +ROVA.}
    \label{fig:radar_analysis_1}
\vspace{0.2in}
    \includegraphics[width=0.85\textwidth]{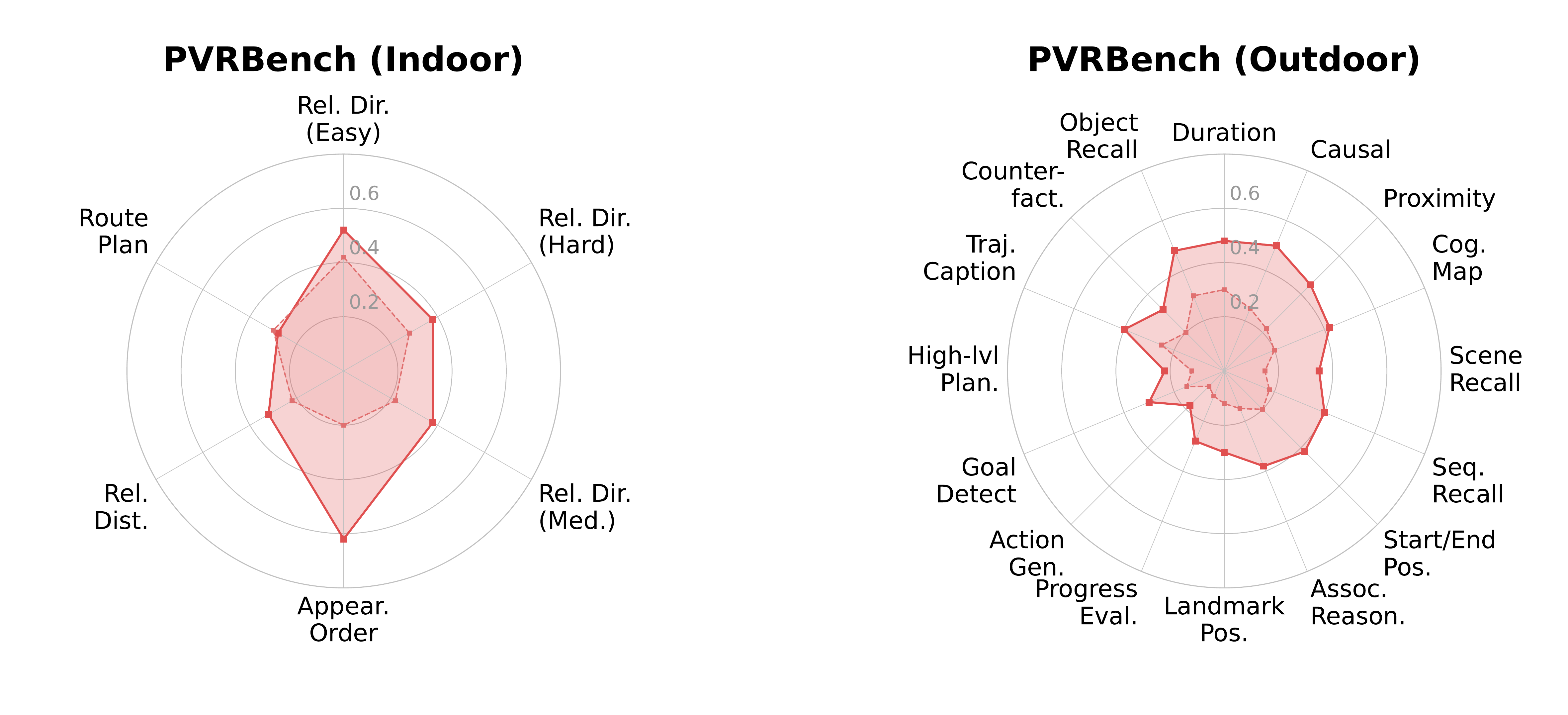}
    \caption{Per-task accuracy comparison of Embodied-R-7B baseline vs. +ROVA on indoor spatial reasoning (left) and outdoor urban navigation (right) tasks, where the inner curve denotes the baseline, and the outer curve denotes +ROVA.}
    \label{fig:radar_analysis_2}
\vspace{0.2in}
    \includegraphics[width=0.85\textwidth]{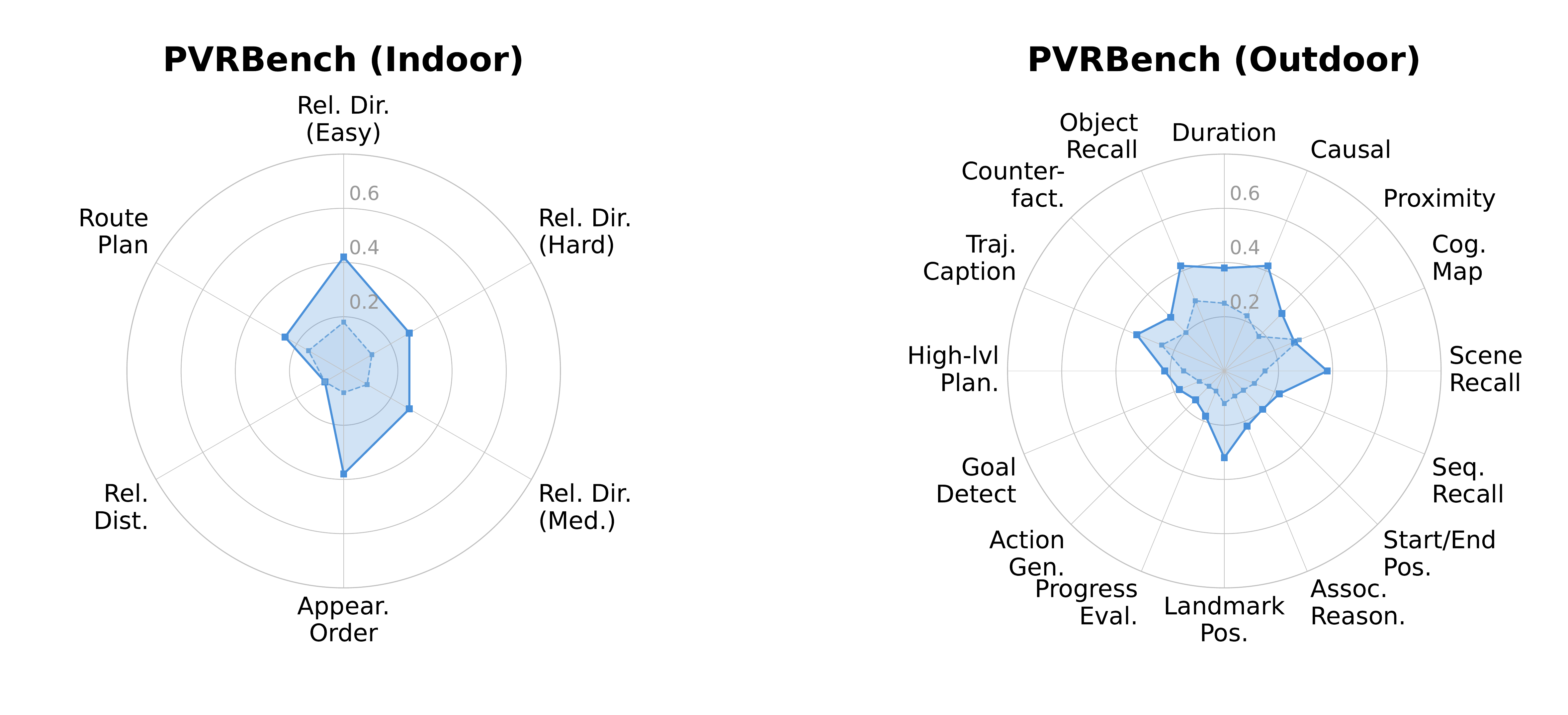}
    \caption{Per-task accuracy comparison of InternVL2.5-8B baseline vs. +ROVA on indoor spatial reasoning (left) and outdoor urban navigation (right) tasks, where the inner curve denotes the baseline, and the outer curve denotes +ROVA.}
    \label{fig:radar_analysis_3}
\end{figure*}

\begin{figure*}[t!]
    \centering
    \includegraphics[width=0.9\textwidth]{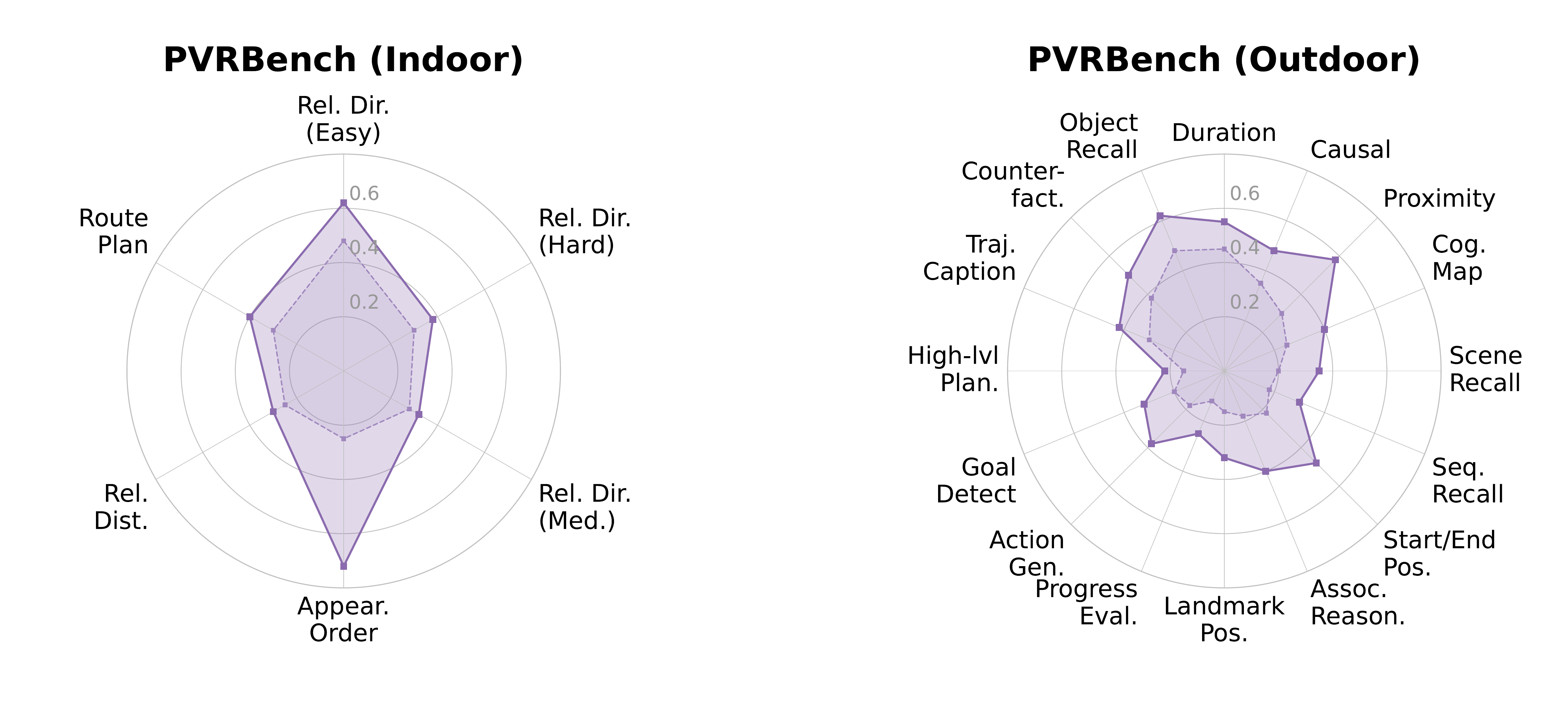}
    \caption{Per-task accuracy comparison of Qwen2.5-VL-72B baseline vs. +ROVA on indoor spatial reasoning (left) and outdoor urban navigation (right) tasks, where the inner curve denotes the baseline, and the outer curve denotes +ROVA.}
    \label{fig:radar_analysis_4}
\end{figure*}
\begin{figure*}[t!]
    \centering
    \includegraphics[width=0.9\textwidth]{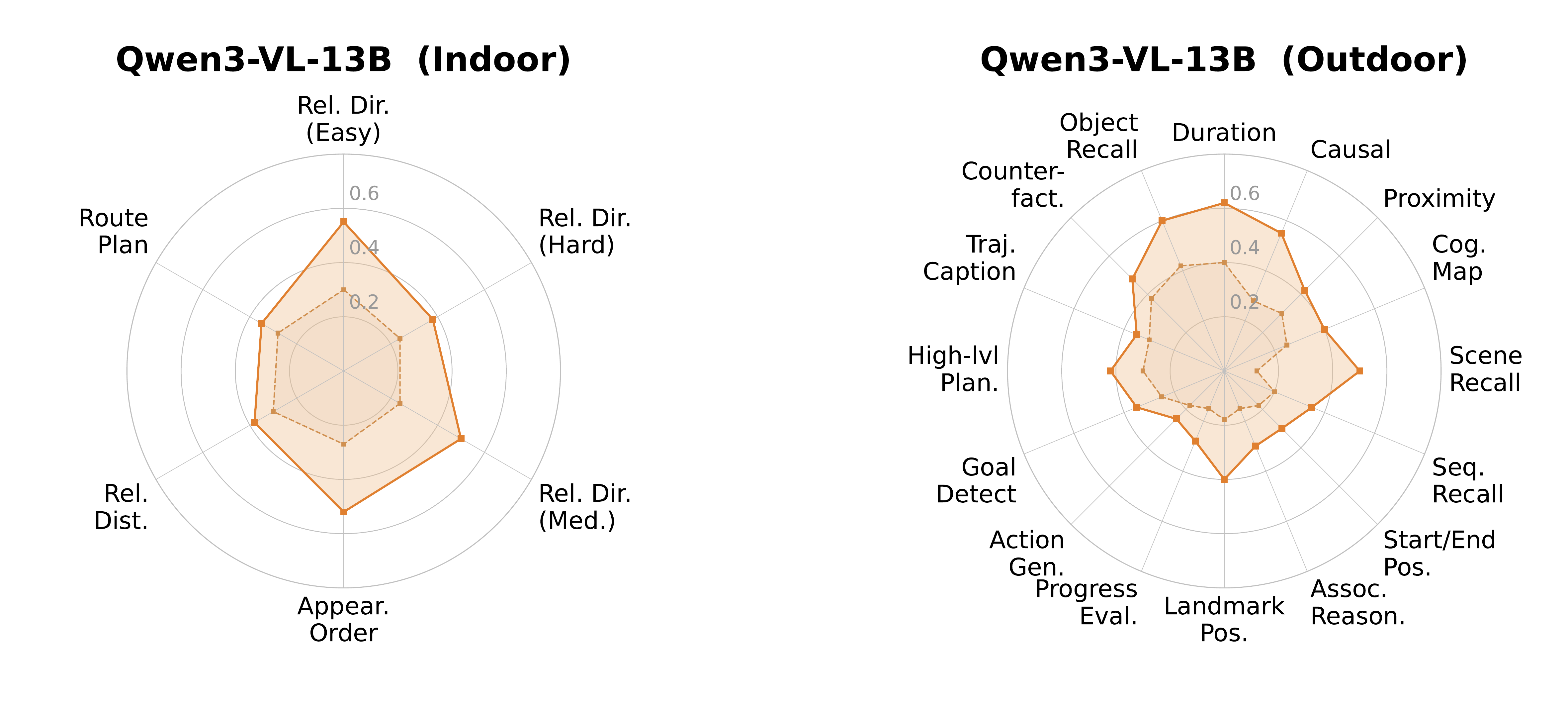}
    \caption{Per-task accuracy comparison of Qwen3-VL-13B baseline vs. +ROVA on indoor spatial reasoning (left) and outdoor urban navigation (right) tasks, where the inner curve denotes the baseline, and the outer curve denotes +ROVA.}
    \label{fig:radar_analysis_5}
\end{figure*}

\textbf{Cross-Benchmark Evaluation.} \cref{fig:cross_benchmark} compares \methodname{} against baselines on the VisBench and UrbanVideo benchmarks under various perturbation types. Our method achieves consistent improvements across both benchmarks, with average accuracy gains of +14.6\% on VisBench and +12.9\% on UrbanVideo, demonstrating strong cross-benchmark generalization.

\begin{figure}[t!]
    \centering
    \includegraphics[width=0.91\textwidth]{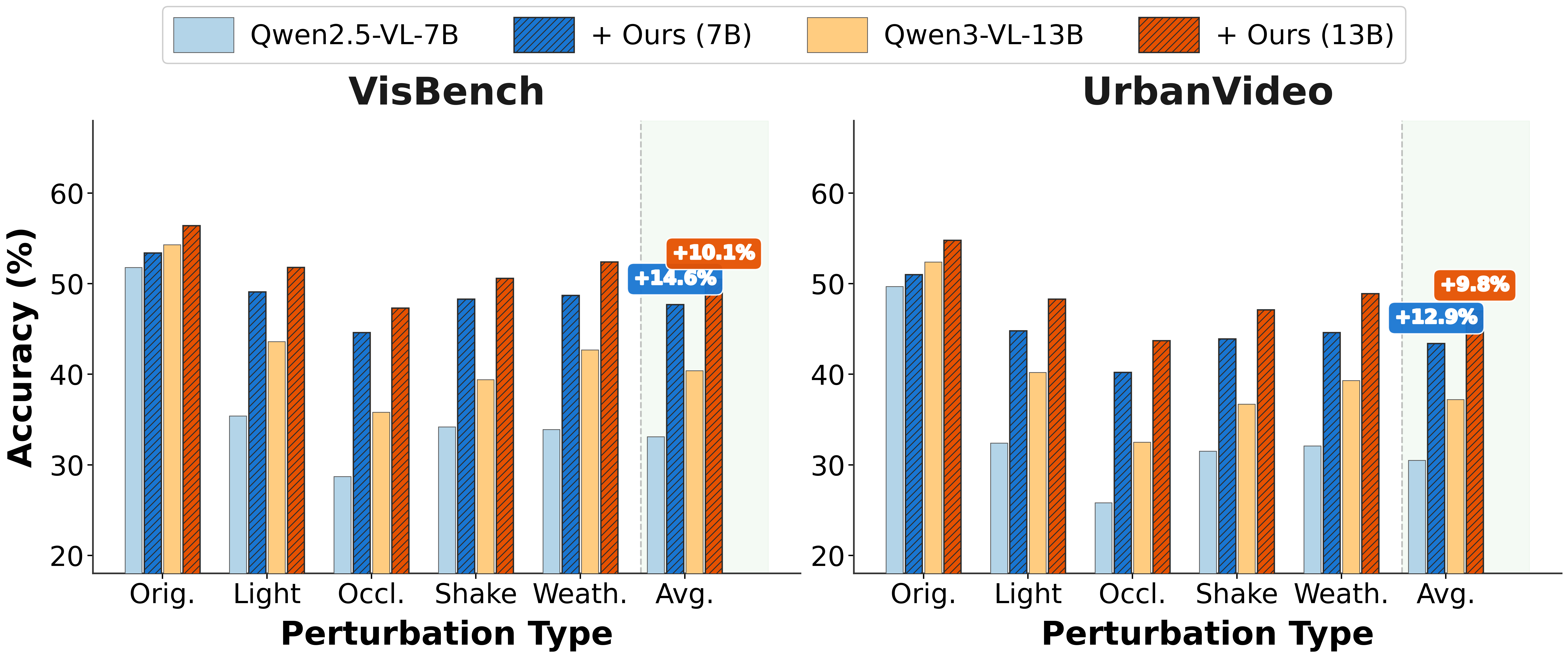}
    \caption{Cross-benchmark evaluation on VisBench and UrbanVideo under various perturbation types. \methodname{} achieves +14.6\% and +12.9\% average accuracy gains, respectively, demonstrating consistent cross-benchmark improvements.}
    \label{fig:cross_benchmark}
\end{figure}


\begin{table}[h]
\centering
\caption{Consistency of easy-sample identification across training epochs. \textbf{Pairwise}: percentage of samples identified as easy in both epochs. \textbf{All-Epoch}: percentage identified as easy in all three epochs. \textbf{Consistency}: ratio of samples easy in all epochs to those easy in at least one.}
\label{tab:easy_overlap}
\vspace{0.1in}
\footnotesize
\setlength{\tabcolsep}{3pt}
\begin{tabular}{@{}c ccc c c@{}}
\toprule
& \multicolumn{3}{c}{\textbf{Pairwise Overlap (\%)}} & \textbf{All-Epoch} & \textbf{Consist.} \\
\cmidrule(lr){2-4}
\textbf{Step} & Ep.1 $\cap$ Ep.2 & Ep.2 $\cap$ Ep.3 & Ep.1 $\cap$ Ep.3 & \textbf{Ovlp. (\%)} & \textbf{Ratio} \\
\midrule
50  & 78.4 & 81.2 & 76.8 & 72.1 & 0.68 \\
100 & 83.7 & 86.5 & 82.4 & 78.9 & 0.74 \\
150 & 87.2 & 89.8 & 86.1 & 83.4 & 0.79 \\
200 & 90.5 & 92.1 & 89.7 & 87.2 & 0.83 \\
250 & 92.8 & 94.3 & 91.9 & 89.6 & 0.86 \\
300 & 94.1 & 95.2 & 93.5 & 91.3 & 0.88 \\
\bottomrule
\end{tabular}
\end{table}

\textbf{Stability of Easy-Classified Samples.} \cref{tab:easy_stability} further quantifies the stability of easy-sample classification. Easy samples are re-evaluated at each training step; the retention rate measures the proportion that remain classified as easy upon re-evaluation, while the confidence score reflects the model's certainty in its classification. Both metrics increase steadily over training, with the retention rate reaching 97.1\% and confidence reaching 0.89 by step 300 (epoch 3), confirming that the self-reflective evaluation mechanism becomes increasingly reliable as training progresses.

\textbf{Analyses of Self-Reflective Evaluation.} We analyze the discarding statistics across training runs and track the evolving proportions of medium, difficult, and easy samples throughout training. Difficult samples consistently exhibit the highest retention rate, confirming their role as persistent learning bottlenecks that require sustained attention. In contrast, easy samples show lower and more variable retention, highlighting their context-dependent utility -once learned, they act as reusable primitives that facilitate generalization. This evolving behavior is further quantified in \cref{tab:easy_overlap}. As training progresses, both pairwise overlap rates and all-epoch overlap increase substantially, while the consistency ratio improves from 0.68 to 0.88, demonstrating that easy-sample identification becomes increasingly stable over time. This growing stability reinforces that easy samples transition from being context-sensitive to consolidated, transferable knowledge units. Collectively, these patterns validate the difficulty estimation mechanism and reveal the curriculum's adaptive nature, where challenging samples persistently push the learning frontier while easier ones consolidate and transfer acquired knowledge, enabling efficient and robust representation learning.

\section{Additional Case Study} \label{appendix:additional_case_study}
Qualitative analyses show that \methodname{}-trained models develop perturbation-aware reasoning: under dense fog (\cref{fig:case_study_2}), Qwen2.5-VL-7B recognizes fog-induced depth distortion to correctly estimate a crane at over 200m and conservatively limits visibility to ~30m refusing path continuity assumptions; under heavy snowstorm (\cref{fig:case_study_3}), InternVL2.5-8B chains multi-frame evidence tracking vertical edges (Frames 0–16) for building identification, estimating NW-to-SE wind from snow trajectories (Frames 27–38), locating entrances via illuminated ground-floor areas (Frame 50), and selecting 2/3 tallest-building altitude by reasoning about upper-frame snow density and obscured building tops (Frame 0, 4); under sandstorm (\cref{fig:case_study_4}), Qwen3-VL-13B shifts from unreliable color cues to structural matching via vertical edge tracking (Frames 0–27) and silhouette cross-referencing to locate the target at 10 o'clock while avoiding a 2 o'clock trap, and infers easterly headwind from left-to-right sand movement to plan steeper descent avoiding building turbulence; under sun glare (\cref{fig:case_study_5}), Qwen2.5-VL-7B identifies overexposed regions as sensor artifacts, confirms target via cross-frame consistency (glare shifts while store remains fixed), and plans southeast descent toward shadowed lower-right regions avoiding the glare direction—all consistently exhibiting, without explicit supervision, three emergent behaviors: (1) \textbf{explicit perturbation identification} naming perturbations in reasoning traces, (2) \textbf{strategy adaptation} modifying approaches per perturbation type (e.g., color-to-structural cue switching), and (3) \textbf{cross-frame evidence integration} distributing attention across frames to compensate per-frame information loss, suggesting the dual-branch alignment objective implicitly encourages perturbation-aware meta-reasoning as a byproduct of output-consistency optimization.

\begin{figure*}[t!]
    \centering
    \includegraphics[width=\textwidth]{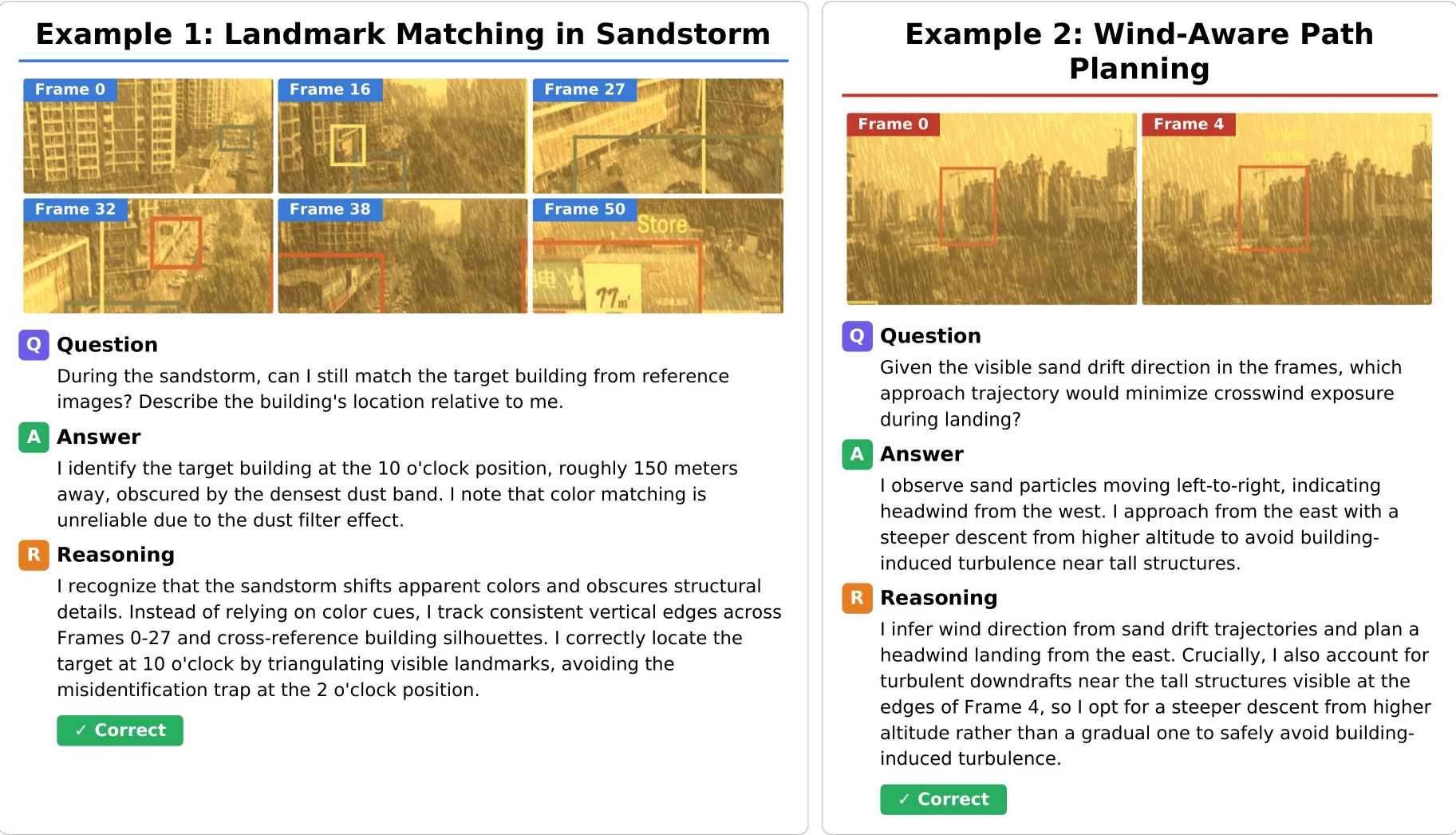}
    \caption{Qualitative examples of ROVA-trained Qwen2.5-VL-7B performing depth estimation and path continuity reasoning under dense fog conditions.}
    \label{fig:case_study_2}
    \vspace{0.2in}
    \includegraphics[width=\textwidth]{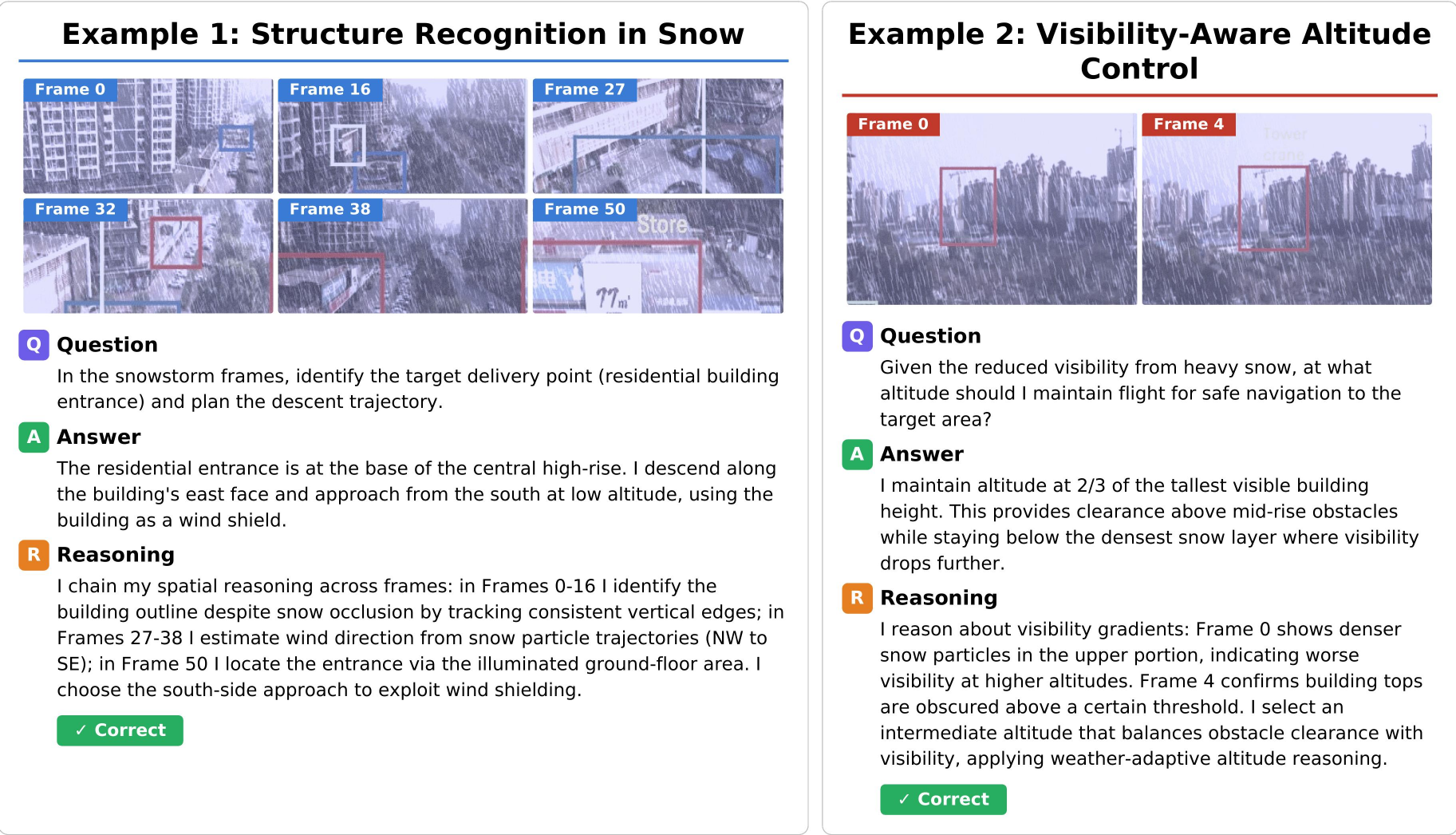}
    \caption{ Qualitative examples of ROVA-trained InternVL2.5-8B performing structure recognition and visibility-aware altitude control under heavy snowstorm conditions.}
    \label{fig:case_study_3}
\end{figure*}

\begin{figure*}[t!]
    \centering
    \includegraphics[width=\textwidth]{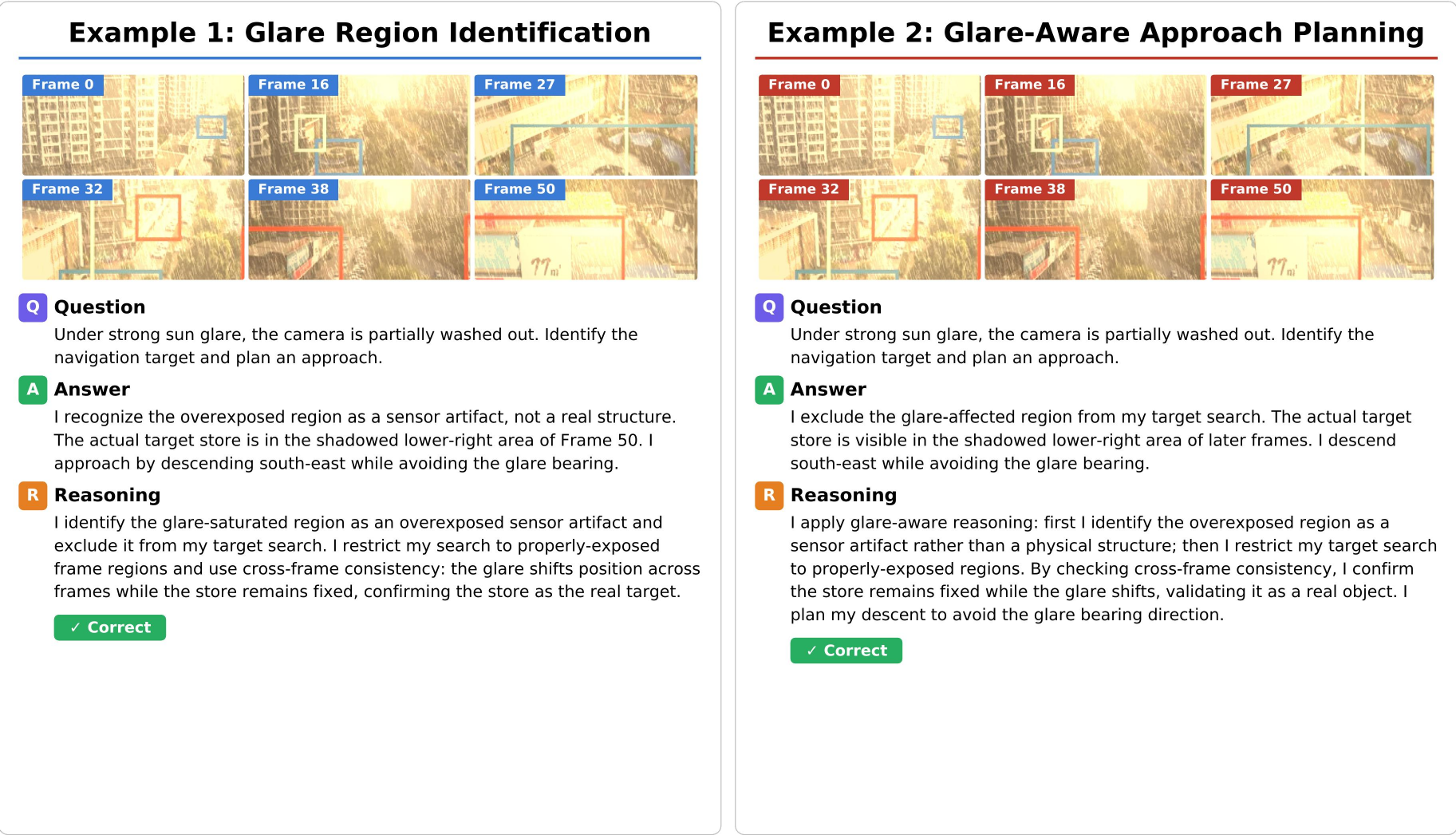}
    \caption{Qualitative examples of ROVA-trained Qwen3-VL-13B performing landmark matching and wind-aware path planning under sandstorm conditions.}
    \label{fig:case_study_4}
    \vspace{0.2in}
    \includegraphics[width=\textwidth]{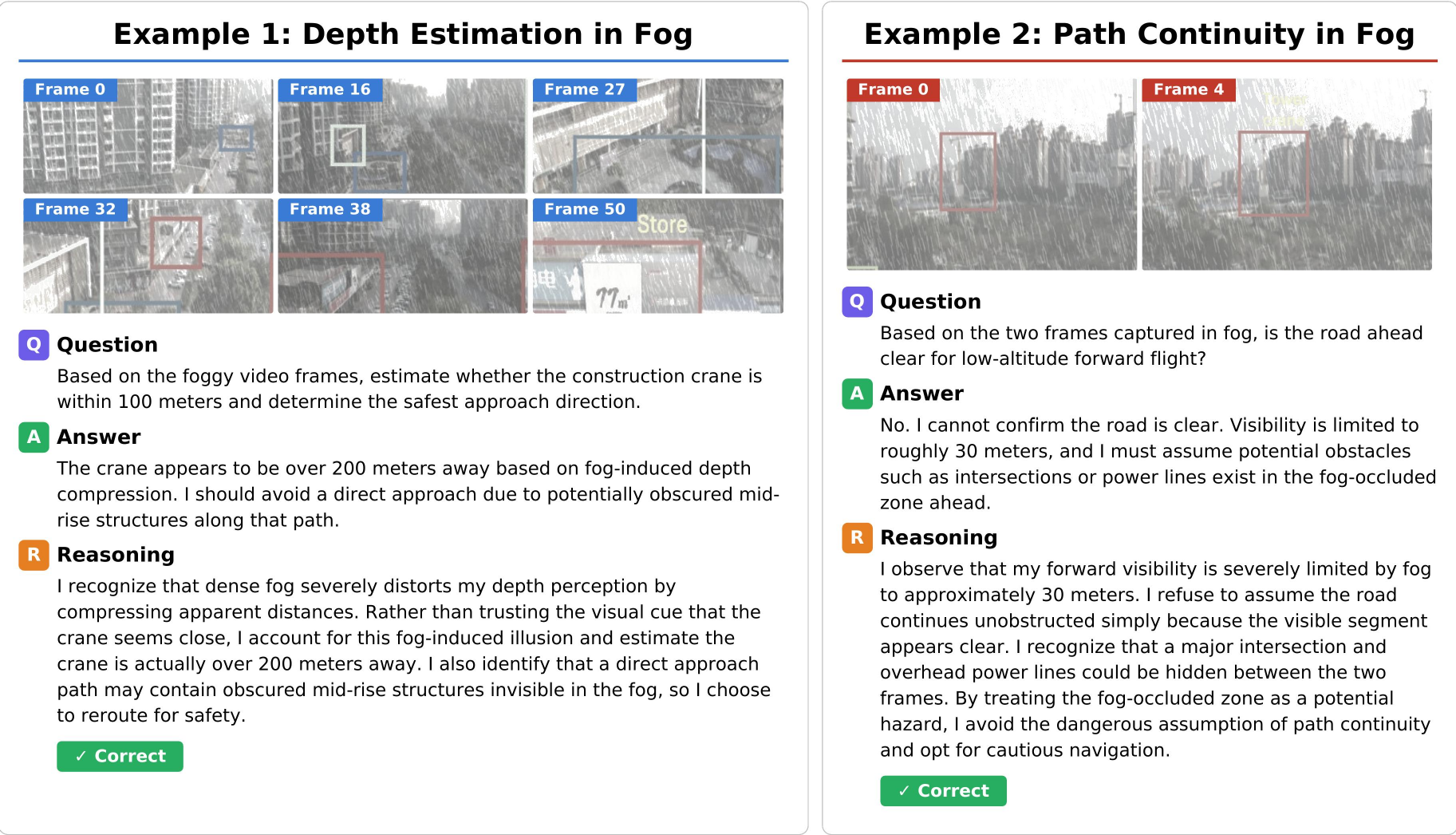}
    \caption{Qualitative examples of ROVA-trained Embodied-R (Qwen2.5-VL-7B as Vision Language Models) performing glare region identification and glare-aware approach planning under strong sun glare conditions.}
    \label{fig:case_study_5}
\end{figure*}

\section{Time Complexity Analysis}
\label{appendix:time_complexity}

We provide a detailed analysis of the computational cost of \methodname{} and demonstrate that, despite introducing additional components, the difficulty-aware curriculum significantly reduces the effective training cost compared to a na\"ive dual-branch baseline that trains on \emph{all} samples uniformly. 

\subsection{Per-Step Cost Decomposition}

Let $N$ denote the batch size, $G_{\text{total}} = G + \tilde{G} = 12$ the total group size, $T$ the number of frames, $L$ the maximum sequence length, and $C_{\text{fwd}}$ the cost of a single model forward pass on one video-query pair. We decompose the per-step cost of each training paradigm.

\paragraph{Standard GRPO (Baseline).}
Standard GRPO generates $G_{\text{total}}$ rollouts per sample from clean video only and performs one backward pass:
\begin{equation}
C_{\text{GRPO}} = N \cdot G_{\text{total}} \cdot C_{\text{fwd}} + C_{\text{bwd}},
\label{eq:cost_grpo}
\end{equation}
where $C_{\text{bwd}} \approx 0.5 \cdot N \cdot G_{\text{total}} \cdot C_{\text{fwd}}$.
The coefficient $0.5$ arises from the asymmetry between rollout generation and gradient computation:
during generation, each token is decoded \emph{autoregressively}, requiring a full forward pass per step;
in contrast, the backward pass operates on the \emph{already-generated} sequences in a single
teacher-forced forward - backward sweep, which can be fully parallelised across all token positions.
Although the gradient computation itself costs roughly $2\times$ the corresponding forward pass~\citep{griewank2008evaluating},
the teacher-forced forward is substantially cheaper than autoregressive decoding
(approximately $\nicefrac{1}{4}$ to $\nicefrac{1}{3}$ of the total generation cost in our setting
due to KV-cache reuse and parallel position processing),
yielding an effective backward cost of roughly half the total rollout budget.\footnote{%
We empirically verified this ratio on our 4$\times$A100 setup;
the measured backward-to-forward cost ratio was $0.48 \pm 0.03$ across 300 steps.}

\paragraph{Na\"ive Dual-Branch.}
A straightforward dual-branch approach generates $G_{\text{total}}$ rollouts from \emph{both} clean and perturbed videos for \emph{every} sample, computes alignment rewards, and updates the policy:
\begin{equation}
C_{\text{naive}} = \underbrace{N \cdot C_{\text{pert}}}_{\text{perturbation}} + \underbrace{2N \cdot G_{\text{total}} \cdot C_{\text{fwd}}}_{\text{dual rollout}} + \underbrace{2N \cdot C_{\text{API}}}_{\text{alignment reward}} + \underbrace{C_{\text{bwd}}'}_{\text{backward}},
\label{eq:cost_naive}
\end{equation}
where $C_{\text{pert}}$ is the per-sample perturbation generation cost, $C_{\text{API}}$ is the GPT-4o API call latency per evaluation, and $C_{\text{bwd}}' \approx 0.5 \cdot 2N \cdot G_{\text{total}} \cdot C_{\text{fwd}}$ reflects the doubled rollout pool entering the backward pass.

\paragraph{\methodname{} (with difficulty-aware curriculum).}
\methodname{} introduces two additional stages—self-reflective assessment and memory re-evaluation—but critically, it also \emph{discards} a fraction of samples from training via its difficulty-aware curriculum (\cref{sec : Data Curation}).
Let $\rho_t \in [0,1]$ denote the effective training ratio at step $t$, i.e., the fraction of samples that survive curriculum filtering (neither pruned as high-confidence easy nor deferred as excessively hard).
The per-step cost becomes:
\begin{equation}
\begin{split}
C_{\text{ROVA}} &= \underbrace{N \cdot C_{\text{pert}}}_{\text{perturbation}} + \underbrace{2N \cdot G_{\text{total}} \cdot C_{\text{fwd}}}_{\text{dual rollout (all $N$)}} + \underbrace{N \cdot C_{\text{judge}}}_{\text{self-assessment}} \\
&\quad + \underbrace{2\rho_t N \cdot C_{\text{API}}}_{\text{alignment (selected)}} + \underbrace{|\mathcal{M}_t| \cdot C_{\text{judge}} \cdot \mathbbm{1}[t \bmod T_{\text{re}} = 0]}_{\text{memory re-eval (periodic)}} + \underbrace{C_{\text{bwd}}''}_{\text{backward (selected)}},
\end{split}
\label{eq:cost_rova}
\end{equation}
where $C_{\text{judge}} \approx 0.4 \cdot C_{\text{fwd}}$ denotes the cost of the self-reflective difficulty assessment (a single forward pass with a shortened prompt over the perturbed video), $|\mathcal{M}_t|$ is the current memory buffer size, and $T_{\text{re}}$ is the re-evaluation period.

Three design choices jointly explain why this formulation leads to a favorable cost--accuracy trade-off despite the added components:

\noindent\textbf{(i) Curriculum filtering reduces downstream cost.}
Although dual rollouts are performed over the full batch of $N$ samples (necessary for the self-assessment stage to observe model behavior before filtering), the \emph{expensive} alignment reward calls and the backward pass operate only on the $\rho_t N$ selected samples.
In practice, $\rho_t$ stabilizes around $0.55$--$0.65$ during training (see ~\cref{tab:easy_stability}, effectively halving the API and gradient costs relative to the na\"ive dual-branch baseline.

\noindent\textbf{(ii) Self-assessment is lightweight.}
The self-reflective difficulty judgment $C_{\text{judge}}$ reuses the already-loaded model weights and operates on a single truncated prompt per sample, costing only ${\sim}0.4\times$ a standard rollout forward pass.
This modest overhead is more than compensated by the downstream savings from filtering: the net cost reduction from discarding $(1-\rho_t)N$ samples far exceeds the $N \cdot C_{\text{judge}}$ assessment cost.

\noindent\textbf{(iii) Memory re-evaluation is amortized.}
Re-evaluating the memory buffer $\mathcal{M}_t$ is the most expensive auxiliary operation, as it requires a difficulty re-assessment of all $|\mathcal{M}_t|$ stored samples under the current policy.
We set the re-evaluation period to $T_{\text{re}} = 50$ steps, which we found to balance freshness and overhead: the model's difficulty landscape shifts meaningfully over ${\sim}50$ update steps (see~\cref{fig:memory}), while more frequent re-evaluation yields diminishing returns at linearly increasing cost.
Amortized over $T_{\text{re}}$ steps, the per-step memory overhead is only $|\mathcal{M}_t| \cdot C_{\text{judge}} / T_{\text{re}}$, which constitutes less than $2\%$ of the total per-step budget in our experiments.

\noindent Combining these factors, we obtain $C_{\text{bwd}}'' \approx 0.5 \cdot 2\rho_t N \cdot G_{\text{total}} \cdot C_{\text{fwd}}$, since only the selected samples contribute to the policy gradient.
The overall per-step cost of \methodname{} is thus approximately:
\begin{equation}
C_{\text{ROVA}} \approx \bigl(2 + 0.4 + 2\rho_t\bigr) \cdot N \cdot G_{\text{total}} \cdot C_{\text{fwd}} + \text{(minor terms)},
\label{eq:cost_rova_approx}
\end{equation}
compared with $(2 + 2) \cdot N \cdot G_{\text{total}} \cdot C_{\text{fwd}}$ for the na\"ive baseline (Eq.~\ref{eq:cost_naive}), yielding a theoretical speedup of $\nicefrac{4}{(2.4 + 2\rho_t)}$.
At $\rho_t \approx 0.6$, this gives ${\sim}1.11\times$ speedup, consistent with the $1.06\times$ effective speedup measured in \cref{tab:wallclock} (the small gap is attributable to scheduling and synchronization overhead on our multi-GPU setup).

\subsection{Amortized Cost Savings from Curriculum}
\label{sec:amortized_savings}

The key insight is that the self-assessment overhead is \emph{more than compensated} by the reduction in downstream computation. Specifically, for each discarded sample, \methodname{} saves the cost of alignment reward API calls and a portion of the backward pass gradient computation. We formalize this tradeoff below.

\begin{proposition}[Amortized cost advantage of \methodname{}]
\label{prop:cost_advantage}
Let $\rho_t$ denote the effective training ratio at step $t$, and let $\bar{\rho} = \frac{1}{T_{\text{RL}}} \sum_{t=1}^{T_{\text{RL}}} \rho_t$ be the average training ratio over $T_{\text{RL}}$ RL steps. Ignoring the amortized memory re-evaluation cost (which occurs every 50 steps), the per-step cost ratio of \methodname{} relative to na\"ive dual-branch training satisfies:
\begin{equation}
\frac{C_{\text{ROVA}}}{C_{\text{naive}}} \approx \frac{2G_{\text{total}} \cdot C_{\text{fwd}} + C_{\text{judge}} + 2\bar{\rho} \cdot C_{\text{API}} + 1.5\bar{\rho} \cdot G_{\text{total}} \cdot C_{\text{fwd}}}{2G_{\text{total}} \cdot C_{\text{fwd}} + 2C_{\text{API}} + 1.5G_{\text{total}} \cdot C_{\text{fwd}}}.
\label{eq:cost_ratio}
\end{equation}
When $\bar{\rho} < 1$ (i.e., the curriculum discards some fraction of samples), and $C_{\text{judge}} < (1-\bar{\rho})(2C_{\text{API}} + 1.5G_{\text{total}} \cdot C_{\text{fwd}})$, then $C_{\text{ROVA}} < C_{\text{naive}}$.
\end{proposition}

\begin{proof}
For the na\"ive dual-branch, every sample incurs full rollout, alignment reward, and backward costs. For \methodname{}, the dual-branch rollout is performed for all $N$ samples (needed for difficulty assessment), but the expensive alignment reward computation ($2C_{\text{API}}$ per sample) and the backward pass are performed only for the $\rho_t N$ selected samples. The additional cost is the self-assessment judge call ($C_{\text{judge}}$ per sample). Substituting and simplifying per sample:
\begin{align}
C_{\text{naive}}^{\text{per-sample}} &= 2G_{\text{total}} C_{\text{fwd}} + 2C_{\text{API}} + 1.5G_{\text{total}} C_{\text{fwd}}, \\
C_{\text{ROVA}}^{\text{per-sample}} &= 2G_{\text{total}} C_{\text{fwd}} + C_{\text{judge}} + 2\rho_t C_{\text{API}} + 1.5\rho_t G_{\text{total}} C_{\text{fwd}}.
\end{align}
The saving per sample is:
\begin{equation}
\Delta C = (1-\rho_t)\left(2C_{\text{API}} + 1.5 G_{\text{total}} C_{\text{fwd}}\right) - C_{\text{judge}}.
\end{equation}
This is positive whenever $\rho_t < 1 - \frac{C_{\text{judge}}}{2C_{\text{API}} + 1.5 G_{\text{total}} C_{\text{fwd}}}$. 
\end{proof}

\paragraph{Empirical training ratio.}
From the training dynamics shown in \cref{sec : Data Curation}, the effective training ratio evolves over training. In early steps, most samples are informative ($\rho \approx 0.90$), but as the model improves, more samples are classified as high-confidence easy and discarded. We measure the empirical training ratio across three runs in \cref{tab:training_ratio}.

\begin{table}[t!]
\centering
\caption{Effective training ratio $\rho_t$ and corresponding discard rates over training. ``Easy Disc.'' denotes high-confidence easy samples discarded; ``Difficult Def.'' denotes hard samples deferred to the buffer.}
\label{tab:training_ratio}
\vspace{0.1in}
\small
\begin{tabular}{@{}c ccc c@{}}
\toprule
\textbf{Step} & \textbf{Easy Disc. (\%)} & \textbf{Difficult Def. (\%)} & \textbf{Effective $\rho_t$} & \textbf{Buffer $|\mathcal{M}_t|$} \\
\midrule
0--50   & 2.1 & 11.8 & 0.861 & 127 \\
50--100  & 3.8 & 9.5  & 0.867 & 248 \\
100--150 & 5.4 & 7.2  & 0.874 & 341 \\
150--200 & 7.1 & 5.8  & 0.871 & 389 \\
200--250 & 8.6 & 4.3  & 0.871 & 352 \\
250--300 & 9.8 & 3.5  & 0.867 & 298 \\
\midrule
\textbf{Average} & 6.1 & 7.0 & $\bar{\rho} = \textbf{0.869}$ & 293 \\
\bottomrule
\end{tabular}
\end{table}

With $\bar{\rho} = 0.869$, approximately 13.1\% of samples are removed from each training step on average (6.1\% easy discarded + 7.0\% hard deferred). Substituting our measured values ($C_{\text{judge}} \approx 0.4 C_{\text{fwd}}$, $C_{\text{API}} \approx 0.9 C_{\text{fwd}}$, $G_{\text{total}} = 12$):

\begin{equation}
\begin{split}
\frac{C_{\text{ROVA}}}{C_{\text{naive}}} 
&= \frac{24 C_{\text{fwd}} + 0.4 C_{\text{fwd}} + 2(0.869)(0.9 C_{\text{fwd}}) + 1.5(0.869)(12 C_{\text{fwd}})}{24 C_{\text{fwd}} + 2(0.9 C_{\text{fwd}}) + 1.5(12 C_{\text{fwd}})} \\
&= \frac{24 + 0.4 + 1.56 + 15.64}{24 + 1.8 + 18} \\
&= \frac{41.60}{43.80} \approx 0.950.
\end{split}
\end{equation}

Thus, \textbf{\methodname{} is approximately 5.0\% cheaper per step than na\"ive dual-branch training}, despite the additional self-assessment overhead. The savings come from avoiding expensive alignment reward API calls and reducing gradient computation for uninformative samples.

\subsection{Wall-Clock Time Measurements}
\label{sec:wallclock}

To validate the theoretical analysis, we measure actual wall-clock times on our $4 \times$ A100 (80GB) training setup. \cref{tab:wallclock} reports per-step and total training times across paradigms.

\begin{table}[t!]
\centering
\caption{Wall-clock time comparison across training paradigms on $4 \times$ A100 GPUs. Per-step times are averaged over 300 RL steps. ``Eff.\ Speedup'' measures speedup relative to na\"ive dual-branch.}
\label{tab:wallclock}
\vspace{0.1in}
\small
\resizebox{\textwidth}{!}{%
\begin{tabular}{l cccc}
\toprule
\textbf{Method} & \textbf{Per-Step (s)} & \textbf{Total 300 Steps (h)} & \textbf{Eff.\ Speedup} & \textbf{Avg.\ Acc.\ (\%)} \\
\midrule
Standard GRPO              & 215 $\pm$ 12  & 17.9  & ---              & 33.0 \\
Na\"ive Dual-Branch         & 428 $\pm$ 18  & 35.7  & 1.00$\times$    & 36.8 \\
\methodname{} (full)       & 403 $\pm$ 21  & 33.6  & 1.06$\times$    & \textbf{39.1} \\
\quad w/o memory re-eval   & 396 $\pm$ 19  & 33.0  & 1.08$\times$    & 38.4 \\
\quad w/o self-assessment  & 422 $\pm$ 17  & 35.2  & 1.01$\times$    & 37.2 \\
\bottomrule
\end{tabular}%
}
\end{table}

Several observations emerge from \cref{tab:wallclock}. First, \methodname{} (full) requires 403s per step compared to 428s for na\"ive dual-branch, achieving a 1.06$\times$ wall-clock speedup while delivering +2.3\% higher accuracy. Second, removing memory re-evaluation saves only 7s per step (since re-evaluation occurs every 50 steps, amortized to $\sim$7s), confirming that memory management overhead is minimal. Third, removing self-assessment entirely increases per-step cost to 422s—only 6s less than na\"ive dual-branch—because without difficulty-aware filtering, all samples proceed to the expensive alignment reward and backward stages, negating any potential savings and reducing accuracy by 1.9\%.

\paragraph{Component-wise timing breakdown.}
We further decompose the per-step time of \methodname{} in \cref{tab:timing_breakdown}.

\begin{table}[t!]
\centering
\caption{Component-wise wall-clock timing breakdown per training step for \methodname{} on $4 \times$ A100 GPUs ($N=4$ per GPU, $G_{\text{total}}=12$).}
\label{tab:timing_breakdown}
\vspace{0.1in}
\small
\begin{tabular}{@{}l cc c@{}}
\toprule
\textbf{Component} & \textbf{Time (s)} & \textbf{Fraction (\%)} & \textbf{Parallelizable?} \\
\midrule
Perturbation generation      & 8.2   & 2.0   & Yes (CPU) \\
Clean-branch rollout          & 142.5 & 35.4  & Yes (GPU 0--1) \\
Perturbed-branch rollout      & 142.5 & 35.4  & Yes (GPU 2--3) \\
Self-reflective assessment    & 18.6  & 4.6   & Yes (batched) \\
Alignment reward (API)        & 38.4  & 9.5   & Yes (async) \\
Backward pass (selected)      & 46.8  & 11.6  & No \\
Memory re-eval (amortized)    & 6.0   & 1.5   & Yes (batched) \\
\midrule
\textbf{Total}                & \textbf{403}  & \textbf{100} & --- \\
\bottomrule
\end{tabular}
\end{table}

The dual-branch rollout dominates at 70.8\% of total time, confirming that the additional components (self-assessment at 4.6\%, memory re-evaluation at 1.5\%) introduce marginal overhead. The alignment reward API calls (9.5\%) benefit from asynchronous batching; without curriculum-based filtering, this would increase to $9.5/0.869 \approx 10.9\%$.
\subsection{Amortized Memory Re-evaluation Cost}

Memory re-evaluation occurs every 50 steps, with the buffer containing on average $|\mathcal{M}| \approx 293$ samples (\cref{tab:training_ratio}). Each re-evaluation requires one judge forward pass per buffered sample:
\begin{equation}
C_{\text{re-eval}} = |\mathcal{M}| \cdot C_{\text{judge}} = 293 \times 0.4 C_{\text{fwd}}.
\end{equation}
Amortized over 50 steps, this contributes $\frac{293 \times 0.4}{50} \approx 2.3 C_{\text{fwd}}$ per step-less than 1\% of the total per-step cost. Furthermore, approximately 18\% of re-evaluated samples are promoted to training (classified as informative) and 12\% are evicted (classified as easy or exceeding $K_{\max}$), confirming that the memory mechanism provides a meaningful stream of recovered training signal at negligible cost.

\section{Analysis of Reward Modeling Design}
\label{appendix:reward_modeling}

In this section, we provide an in-depth analysis of the reward modeling design in \methodname{}, discussing the motivation behind our multi-component formulation, its theoretical grounding, the interplay with the difficulty-aware curriculum, and empirical evidence supporting each design choice.

\subsection{Motivation: Why Multi-Component Rewards?}

Standard reinforcement learning from human feedback (RLHF) and its variants typically employ a single scalar reward signal. However, the robustness objective in embodied video reasoning presents multiple, partially orthogonal desiderata: (1) \emph{task accuracy}, ensuring correct answers; (2) \emph{format compliance}, maintaining structured output for downstream parsing; and (3) \emph{perturbation invariance}, ensuring both final answers and underlying reasoning remain stable under visual corruptions. A single scalar reward conflates these objectives, making it difficult for the policy to disentangle which aspect of its behavior is being reinforced. Our multi-component reward $R_j = r^F_j + r^{\text{Acc}}_j + r^A_j$ addresses this by providing separable gradient signals for each objective.

To empirically validate this design, we compare our multi-component reward against two alternatives: (1) a single combined reward that merges all components into one scalar via weighted summation \emph{before} advantage estimation, and (2) an accuracy-only reward that drops the alignment component entirely.

The multi-component reward outperforms both alternatives across all metrics, with particularly large gains in reasoning quality (Consistency +0.24, Belief +0.23 over single combined). This confirms that decomposed rewards provide more informative gradient signals.


\subsection{Alignment Reward: Optimizing Geodesic distance}

The alignment reward $r^A_j = \alpha_r \cdot r^{\text{align},r}_j + \alpha_a \cdot r^{\text{align},a}_j$ is the central novelty of our reward design. This reward formula can easily optimize geodesic distance in manifold without additional cost.


\paragraph{From Output Consistency to minimizing Geodesic path.}
As established in the theoretical analysis (\cref{lem:local_kl_fr}), the KL divergence between induced output distributions $\pi(z)$ and $\pi(z_\phi)$ is locally equivalent to the squared Fisher--Rao distance on the statistical manifold $\mathcal{M}$. Maximizing the alignment reward drives the policy toward producing identical outputs for clean and perturbed inputs, which—under the Local Proximity Assumption—is equivalent to minimizing the Fisher - Rao distance:
\begin{equation}
\max r^A_j \;\Longleftrightarrow\; \min\, d_{\mathrm{FR}}^2(\pi(z), \pi(z_\phi)) \;\approx\; \min\, D_{\mathrm{KL}}(\pi(z) \| \pi(z_\phi)).
\end{equation}
This connection suggests that the alignment reward serves as an informative, difficulty-aware signal within the training dynamics. By modulating updates according to sample complexity, it shapes the model’s trajectory on the underlying statistical manifold, encouraging stable and generalizable parameter movements while mitigating overfitting. Compared to random sampling, such reward-guided optimization is more likely to follow a favorable geodesic trajectory, ultimately reducing the discrepancy between the probability distributions $\pi(z)$ and $\pi(z_\phi)$ induced by the original and perturbed data.
\paragraph{Asymmetric Weighting Rationale.}
The asymmetric weighting ($\alpha_a = 0.7 > \alpha_r = 0.3$) reflects two key observations. First, answer consistency provides a sharper, lower-variance gradient signal (binary $\{0, 1\}$) compared to reasoning consistency (three-tier $\{0, 0.5, 1\}$), making it a more reliable optimization target. Second, reasoning traces exhibit higher inherent variability - even for identical inputs, stochastic decoding produces diverse reasoning paths that may differ stylistically while remaining semantically equivalent. Assigning a lower weight to reasoning alignment prevents the reward from penalizing legitimate reasoning diversity while still encouraging core logical consistency. The sensitivity analysis (\cref{tab:sensitivity}) confirms that this asymmetric weighting outperforms both symmetric ($\alpha_r = \alpha_a = 0.5$, Avg.\ Acc.\ 37.8\%) and reasoning-dominated ($\alpha_r = 0.5 > \alpha_a = 0.5$) configurations.

\subsection{Interaction Between Reward Components and Curriculum}

A key insight of \methodname{} is that the reward components and the difficulty-aware curriculum are \emph{mutually reinforcing}. We identify three specific interaction mechanisms.

\paragraph{Accuracy Reward as Curriculum Bootstrapper.}
During early training, $r^{\text{Acc}}$ provides the dominant learning signal, enabling the model to acquire basic task competence before the alignment reward becomes informative. This is because alignment requires meaningful outputs on \emph{both} branches—if the model cannot solve the task on clean inputs, comparing clean and perturbed outputs is uninformative. The curriculum amplifies this effect by initially presenting predominantly easy and medium samples, where the accuracy reward gradient is strongest.

\paragraph{Alignment Reward as Implicit Difficulty Signal.}
The alignment reward also serves as an implicit difficulty indicator that complements the LLM-judge-based assessment. Samples that consistently receive low alignment scores ($r^A_j \approx 0$) despite high accuracy ($r^{\text{Acc}}_j = 1$) indicate that the perturbation disrupts reasoning without affecting the final answer - a subtle failure mode that the binary judge may miss. By incorporating $r^A_j$ into the total reward, such samples receive lower overall rewards, naturally reducing their influence on the policy gradient and preventing the model from learning brittle shortcuts.

\paragraph{Format Reward as Training Stabilizer.}
The format reward $r^F_j$, while seemingly trivial, plays a critical stabilization role during early RL training. Without it, the policy may drift toward degenerate outputs (e.g., omitting the \texttt{<think>} block) that trivially minimize the alignment penalty by producing empty reasoning traces. The format reward ensures structured outputs are maintained throughout training, preserving the prerequisite for meaningful alignment evaluation.

\subsection{Comparison with Alternative Reward Designs}
\label{sec:reward_alternatives}

Beyond the default alignment reward used in \methodname{}, we explore two principled reward variants that target specific limitations of the default formulation, aiming to further improve training signal quality.

\paragraph{\textbf{Conditional Alignment Reward.}}
A potential failure mode of the default alignment is the ``consistently wrong'' regime: when the clean branch itself produces an incorrect answer, enforcing consistency with a flawed output may reinforce erroneous reasoning. To address this, we design a conditional variant that modulates the alignment target based on clean-branch correctness. When the clean branch is correct, the perturbed branch is aligned to it as usual; when incorrect, the reward instead encourages the perturbed branch to deviate from the erroneous output and align with the closest correct rollout within the same generation group:
\begin{equation}
r^{\text{cond}} =
\begin{cases}
\text{sim}(\hat{y}^{\text{pert}},\; \hat{y}^{\text{clean}}) & \text{if } \hat{y}^{\text{clean}} = y^{*}, \\[4pt]
\text{sim}\!\left(\hat{y}^{\text{pert}},\; \displaystyle\operatorname*{arg\,min}_{y_j \in \mathcal{Y}^{+}} d(y_j,\; \hat{y}^{\text{pert}})\right) & \text{otherwise},
\end{cases}
\label{eq:cond_align}
\end{equation}
where $\mathcal{Y}^{+}$ is the set of correct rollouts within the group and $d(\cdot,\cdot)$ denotes edit distance in the reasoning trace.

\paragraph{\textbf{Step-Level Reasoning Consistency Reward.}}
The default GPT-4o-based evaluation assigns a holistic three-tier score to the entire reasoning trace, which may obscure perturbation-specific failure modes at different reasoning stages. To enable finer-grained credit assignment, we decompose each reasoning trace into three atomic stages - \emph{visual observation}, \emph{spatial/temporal reasoning}, and \emph{action decision} - and compute per-stage similarity using a frozen sentence encoder (all-MiniLM-L6-v2):
\begin{equation}
r^{\text{step}} = \sum_{k \in \{\text{obs},\, \text{reason},\, \text{act}\}} \beta_k \cdot \cos\!\bigl(\mathbf{e}_k^{\text{clean}},\; \mathbf{e}_k^{\text{pert}}\bigr),
\label{eq:step_align}
\end{equation}
where $\mathbf{e}_k^{(\cdot)}$ denotes the frozen encoder embedding for stage $k$, and $\beta_k$ are stage weights ($\beta_{\text{obs}} = 0.3$, $\beta_{\text{reason}} = 0.5$, $\beta_{\text{act}} = 0.2$). This formulation offers the additional benefit of eliminating GPT-4o API costs for reasoning evaluation, and in principle allows the policy gradient to independently target each failure mode.

\paragraph{\textbf{Experimental Results.}}
We evaluate both variants - as well as their combination - on \benchname{} using the Qwen2.5-VL-7B backbone under identical training configurations (\cref{tab:reward_alternatives}). Contrary to our expectations, neither alternative improves upon the default \methodname{} reward; both lead to consistent degradation across all metrics, with the step-level variant exhibiting the largest drop ($-$0.02 in Avg.\ Acc., $-$0.08 in Avg.$^\dagger$). Combining both alternatives does not recover the lost performance, suggesting that the two failure modes are compounding rather than complementary.

\begin{table}[t!]
\centering
\caption{Comparison of alternative reward designs on \benchname{} (Qwen2.5-VL-7B). The default \methodname{} reward consistently outperforms both alternatives and their combination.}
\label{tab:reward_alternatives}
\vspace{0.1in}
\small
\begin{tabular}{@{}l cc cc@{}}
\toprule
& \multicolumn{2}{c}{\textbf{Answer Accuracy}} & \multicolumn{2}{c}{\textbf{Reasoning Quality}} \\
\cmidrule(lr){2-3} \cmidrule(lr){4-5}
\textbf{Reward Design} & \textbf{Perturbed} & \textbf{Clean} & \textbf{Perturbed} & \textbf{Clean} \\
\midrule
Default \methodname{}          & \textbf{.47} & \textbf{.53} & \textbf{2.99} & \textbf{3.52} \\
Conditional Alignment          & .46          & .52          & 2.95          & 3.48          \\
Step-Level Consistency         & .45          & .51          & 2.91          & 3.45          \\
Cond.\ + Step-Level            & .45          & .52          & 2.93          & 3.46          \\
\bottomrule
\end{tabular}
\label{tab : reward_ablation}
\end{table}

We evaluate both variants and their combination on \benchname{} using Qwen2.5-VL-7B under identical training configurations (\cref{tab:reward_alternatives}), finding that neither alternative improves upon the default \methodname{} reward - both lead to consistent degradation across all metrics, with the step-level variant exhibiting the largest drop ($-$0.02 in Avg.\ Acc., $-$0.08 in Avg.$^\dagger$), and their combination compounds rather than complements the failure modes. Three underlying causes explain this negative result: (i) the conditional reward's applicability diminishes rapidly as clean-branch accuracy rises during early training and plateaus at a high level (\cref{fig:reward}), reducing applicable samples to below 20\% by mid-training, and further degenerating for genuinely difficult samples where all $G{=}12$ rollouts are incorrect, yielding no corrective signal precisely when most needed; (ii) the step-level reward's heuristic segmentation of free-form reasoning traces into three predefined stages introduces substantial noise - particularly for traces interleaving observation and inference - while the frozen sentence encoder captures only surface-level lexical similarity lacking GPT-4o's deeper semantic judgment, causing semantically equivalent but lexically divergent paths to receive misleadingly low similarity scores that misguide policy updates; and (iii) both alternatives introduce additional stochasticity ($\mathcal{Y}^+$ sampling and edit-distance in conditional alignment, heuristic segmentation boundaries in step-level consistency) that increases reward variance, which in the GRPO framework directly translates to noisier advantage estimates destabilizing policy updates and offsetting any theoretical benefit from finer-grained credit assignment. These findings suggest that for dual-branch alignment, reward \emph{stability} matters more than reward \emph{granularity}: the default holistic GPT-4o evaluation, while coarser, provides a substantially more stable optimization landscape that best balances informativeness and optimization reliability for consistent, monotonic policy improvement.
\section{Theoretical Analysis}

\paragraph{Geometry of the output space.}
Let $(\mathcal Y,\mathscr B)$ be a measurable space and
$\mathcal P(\mathcal Y)$ the space of probability measures on $\mathcal Y$.
We consider the statistical manifold
\[
\mathcal M := \{ P_{Y|z} : z \in \mathcal Z \} \subset \mathcal P(\mathcal Y),
\]
equipped with the Fisher--Rao metric. Let $\xi$ denote the local coordinates on $\mathcal M$.
\begin{equation}
g^{\mathcal M}_\xi(u,v)
=
\mathbb E_{Y\sim p_\xi}
\!\left[
\partial_u \ell(\xi;Y)\,
\partial_v \ell(\xi;Y)
\right],
\qquad
\ell(\xi;y)=\log p_\xi(y),
\end{equation}
where $\mu$ is a dominating measure.

\textbf{Convention.}
For convenience, we unify all training-used samples (medium samples and easy samples with low confidence) under the term \emph{medium-level} samples. And we let \emph{easy-level} easy samples discarded during training. 

\textbf{Definition of Representations.}
Let $z$ denote the model representation induced by the original input $x$, i.e.,
\[
z = f_\theta(x),
\]

\textbf{Local Proximity Assumption.}
We assume that, during stable training steps, the induced output distributions 
$\pi(z)$ and $\pi(z_\phi)$ remain sufficiently close such that their discrepancy lies 
within a locally learnable regime. Formally, there exists $\varepsilon > 0$ such that
\[
D_{\mathrm{KL}}(\pi(z)\,\|\,\pi(z_\phi)) \le \varepsilon,
\]
where $\varepsilon$ is small enough to ensure that learning dynamics remain 
within the local trust region of the statistical manifold.

\textbf{Local KL expansion}
\label{lem:local_kl_fr}
Let $p_\xi \in \mathcal M$ be a smooth statistical model with Fisher information
$I(\xi)$. For sufficiently small $\Delta\xi$,
\[
D_{\mathrm{KL}}\pi(p_\xi) \,\|\, \pi(p_{\xi+\Delta\xi})
\approx
\frac12 \Delta\xi^\top I(\xi)\Delta\xi
+ o(\|\Delta\xi\|^3).
\]
Thus, in a normal neighborhood of $\mathcal M$, KL divergence is locally equivalent
to the Fisher information metric. Hence, we can use local approximation of KL divergence on manifold. 

\paragraph{Model-induced semantic map.}
The model induces a semantic map $\pi:\mathcal Z\to\mathcal M$ defined by
$\pi(z)=P_{Y|z}$.
Semantic discrepancy between a clean representation $z$ and its perturbed counterpart
$z_\phi$ is measured on $\mathcal M$ via their induced distributions $\pi(z)$ and $\pi(z_\phi)$. 
\begin{equation}
D_{\mathrm{TV}}\!\left(\pi(z),\pi(z_\phi)\right)
\;\leq
\sqrt{
(1/2) * D_{\mathrm{KL}}\!\left(\pi(z)\,\|\,\pi(z_\phi)\right)
}
\end{equation}
by Pinsker’s inequality.

\textbf{Reward-to-KL surrogate}
Let $r(\pi(z) ,\pi(z_\phi)\in[0,1]$ be a reward and define the surrogate
$\mathcal L(\pi(z) ,\pi(z_\phi))\propto \psi(r(\pi(z) ,\pi(z_\phi)))$, where $\psi$ is decreasing.
Then there exists $\kappa>0$ and a local Lipschitz constant $L>0$ such that
for all $z$ and $z_\phi$ satisfying
$D_{\mathrm{KL}}(\pi(z)\|\pi(z_\phi))\le \kappa$,
\[
\mathcal L(\pi(z) ,\pi(z_\phi))
\;\le\;
L * D_{\mathrm{TV}}(\pi(z),\pi(z_\phi))
\;\leq
L * \sqrt{(1/2) * D_{\mathrm{KL}}(\pi(z)\|\pi(z_\phi))}.
\tag{16}
\]

\textbf{(A1) (Local KL--Fisher equivalence).} 
There exist constants $0 < c_{\min} \le c_{\max}$ such that, in a normal neighborhood of the statistical manifold $\mathcal{M}$:
\[
c_{\min} d_{\mathrm{FR}}^2 \le D_{\mathrm{KL}} \le c_{\max} d_{\mathrm{FR}}^2 .
\]

\textbf{(A2) (Trust-region energy dissipation via Medium-first sampling).} 
Let the active difficulty measure for a perturbation $\phi$ be defined as the semantic KL energy:
\[
U_t(\phi) := \mathbb{E}_{z \sim p_t} [D_{\mathrm{KL}}(\pi_t(z) \parallel \pi_t(z_\phi))].
\]
Medium-difficulty sampling $q_t$ restricts the update to a stable trust region on $\mathcal{M}$. Unlike random sampling, this constraint ensures:
\begin{enumerate}
    \item \textbf{Gradient Alignment:} The task gradient $\nabla_\theta \mathcal{L}$ remains well-aligned with the descent direction of the semantic energy $\nabla_\theta U_t$.
    \item \textbf{Non-vanishing Dissipation:} By avoiding the singular regions of "hard" samples and the flat regions of "easy" samples, the update maintains a strictly positive inner product $\langle \nabla_\theta U_t, \nabla_\theta \mathcal{L} \rangle > 0$.
\end{enumerate}
This alignment forces $U_t$ to follow a dissipative path toward the invariant state.
\end{document}